\documentclass[11pt,a4paper]{article}

\usepackage{graphicx, graphics, epsfig, color}
\usepackage{booktabs}
\usepackage{subcaption}
\usepackage[backend=biber,style=numeric-comp,isbn=false,doi=false,eprint=false]{biblatex}
\usepackage{amsmath,amssymb,amsthm}
\usepackage{tikz, pgfplots}
\usetikzlibrary{plotmarks,spy}
\pgfplotsset{compat=newest}
\usepackage[margin=1in]{geometry}
\usepackage{hyperref}
\hypersetup{colorlinks=false}
\usepackage{enumitem}

\bibliography{liao}

\DeclareMathOperator{\tr}{tr}

\DeclareMathOperator{\diag}{diag}
\DeclareMathOperator{\rank}{rank}

\DeclareMathOperator{\supp}{supp}
\DeclareMathOperator{\dist}{dist}

\newcommand{\T}{{\sf T}}

\newcommand{\RR}{{\mathbb{R}}}
\newcommand{\CC}{{\mathbb{C}}}
\newcommand{\EE}{{\mathbb{E}}}
\newcommand{\NN}{{\mathcal{N}}}

\newcommand{\A}{\mathbf{A}}
\newcommand{\B}{\mathbf{B}}
\newcommand{\C}{\mathbf{C}}
\newcommand{\D}{\mathbf{D}}

\newcommand{\G}{\mathbf{G}}

\newcommand{\M}{\mathbf{M}}
\newcommand{\V}{\mathbf{V}}
\newcommand{\U}{\mathbf{U}}

\newcommand{\X}{\mathbf{X}}
\newcommand{\Z}{\mathbf{Z}}
\newcommand{\Q}{\mathbf{Q}}

\renewcommand{\H}{\mathbf{H}}
\renewcommand{\P}{\mathbf{P}}

\newcommand{\uu}{\mathbf{u}}

\newcommand{\x}{\mathbf{x}}

\newcommand{\z}{\mathbf{z}}

\newcommand{\w}{\mathbf{w}}
\newcommand{\vv}{\mathbf{v}}
\newcommand{\zo}{\mathbf{0}}
\newcommand{\one}{\mathbf{1}}
\newcommand{\I}{\mathbf{I}}

\newcommand{\balpha}{\boldsymbol{\alpha}}
\newcommand{\bmu}{\boldsymbol{\mu}}

\newcommand{\bomega}{\boldsymbol{\omega}}
\newcommand{\btheta}{\boldsymbol{\theta}}

\newcommand{\bLambda}{\boldsymbol{\Lambda}}

\newcommand{\asto}{{ \xrightarrow{a.s.} }}

\definecolor{RED}{rgb}{0.7,0,0}
\definecolor{BLUE}{rgb}{0,0,0.69}
\definecolor{GREEN}{rgb}{0,0.6,0}
\definecolor{PURPLE}{rgb}{0.69,0,0.8}

\newtheorem{Assumption}{Assumption}
\newtheorem{Theorem}{Theorem}
\newtheorem{Corollary}{Corollary}

\newtheorem{Lemma}{Lemma}
\newtheorem{Remark}{Remark}

\makeatletter
\newcommand\setItemnumber[1]{\setcounter{enum\romannumeral\@enumdepth}{\numexpr#1-1\relax}}
\makeatother

\title{Hessian Eigenspectra of More Realistic Nonlinear Models}
\author{
  Zhenyu Liao\\
  ICSI and Department of Statistics\\
  University of California, Berkeley, USA\\
  \texttt{zhenyu.liao@berkeley.edu}
  \and
   Michael W. Mahoney\\
  ICSI and Department of Statistics\\
  University of California, Berkeley, USA\\
  \texttt{mmahoney@stat.berkeley.edu}
}
\date{\today}

\begin{document}
\maketitle

\begin{abstract}
Given an optimization problem, the Hessian matrix and its eigenspectrum can be used in many ways, ranging from designing more efficient second-order algorithms to performing model analysis and regression diagnostics. 
When nonlinear models and non-convex problems are considered, strong simplifying assumptions are often made to make Hessian spectral analysis more tractable.
This leads to the question of how relevant the conclusions of such analyses are for more realistic nonlinear models. 
In this paper, we exploit deterministic equivalent techniques from random matrix theory to make a \emph{precise} characterization of the Hessian eigenspectra for a broad family of nonlinear models, including models that generalize the classical generalized linear models, without relying on strong simplifying assumptions used previously. 
We show that, depending on the data properties, the nonlinear response model, and the loss function, the Hessian can have \emph{qualitatively} different spectral behaviors: of bounded or unbounded support, with single- or multi-bulk, and with 
isolated eigenvalues on the left- or right-hand side of the bulk. 
By focusing on such a simple but nontrivial nonlinear model, our analysis takes a step forward to unveil the theoretical origin of many visually striking features observed in more complex machine learning models.
\end{abstract}


\section{Introduction}
\label{sec:intro}

The Hessian matrix is ubiquitous in applied mathematics, statistics, and machine learning (ML). 
Given a (loss) function $L(\w)$ with some parameters $\w \in \RR^p$, the Hessian $\H(\w) \in \RR^{p \times p}$ is defined as the second derivative of the objective function with respect to the model parameter, i.e., $\H(\w) = \partial L(\w)/(\partial \w \cdot \partial \w^\T)$.
When a ML model is being trained, it is common to parameterize that model by $\w$, and then train that model by minimizing some (smooth) loss function $L(\w)$, with associated Hessian matrix $\H(\w)$, by backpropagating the error to improve $\w$~\cite{goodfellow2016deep,yao2020adahessian}.
Alternatively, once a ML model is trained, the Hessian matrix $\H(\w)$ (and the related Fisher information matrix \cite{van2000asymptotic,wainwright2019high}) can be examined to identify outliers, perform diagnostics, and/or engage in other sorts of uncertainty quantification and model validation~\cite{hastie2009elements,yao2019pyhessian,sankar2020deeper}.

For convex problems, the Hessian $\H(\w)$ provides detailed information on how to adjust the gradient to achieve improved convergence, e.g., as in Newton-like methods.
For non-convex problems, the properties of the local loss ``landscape'' around a given point $\w$ in the parameter space is of central significance~\cite{dauphin2014identifying,kawaguchi2016deep,chi2018nonconvex,lee2019first,yao2018hessian,yao2019trust,yao2019pyhessian}.
In this case, most obviously, the signs of the smallest and largest Hessian eigenvalue can be used to test whether a given $\w$ is a local maximum, local minimum, or a saddle point (i.e., the second-derivative test).
More subtly, the Hessian eigenvalue distribution characterizes the local curvature of the loss function and provides direct access to, for instance, the fraction of negative Hessian eigenvalues.
This quantity determines the number of (local) descent directions, a quantity that is directly connected to the rates of convergence of various optimization algorithms \cite{jain2017non,chitour2018geometric}, as well as the large-scale structure of the configuration space~\cite{bray2007statistics,auffinger2013random}. 

For theoretical analysis of neural network (NN) models, Hessian eigenspectra are often assumed to follow well-known random matrix distributions such the Mar{\u c}enko–Pastur law \cite{marchenko1967distribution} or the Wigner's semicircle law \cite{wigner1955characteristic}.
This enables one to use Random Matrix Theory (RMT), but it involves (for NNs, at least) making relatively strong simplifying assumptions \cite{pennington2017geometry,pennington2018spectrum,choromanska2015loss}. 
A somewhat more realistic setup involves using a so-called \emph{spiked model} (or a spiked covariance model) \cite{baik2005phase,benaych2011eigenvalues,lesieur2015phase}.
In this case, the matrix follows a \emph{signal-plus-noise} model and consists of \emph{full rank} random noise matrix and \emph{low rank} statistical information structure. %
The ``signal'' eigenvalues are generally larger than the noisy ``bulk'' eigenvalues; and the maximum eigenvalues, when isolated from the bulk, are referred to as the ``spikes.'' 
A substantial theory-practice gap exists, however. 
In both toy examples \cite{granziol2020beyond}, as well as state-of-the-art NN models \cite{yao2018hessian,yao2019trust,yao2019pyhessian,yao2020adahessian,shen2020q,dong2019hawq}, the strong simplifying assumptions are far from realistic or satisfactory.
(A similar theory-practice gap has been observed for other NN matrices to which RMT has been applied, perhaps most notably weight matrices~\cite{MM18_TR,MM20a_trends_TR}.)
A more precise understanding of the Hessian eigenspectra (and its dependence on input data structure, the underlying response model and model parameters, as well as the loss function) for more realistic ML models is~needed.


\subsection{Our approach}\label{subsec:our_approach}

In this article, we address these issues, in a setting that is simple enough to be analytically tractable but complex enough to shed light on much more realistic large-scale models.
We consider a family of generalized generalized linear models (G-GLMs) that extends the popular generalized linear model (GLM) \cite{dobson2018introduction,hastie2009elements}; and we show that, even for such simple models, the key simplifying assumptions used in previous theoretical analyses of Hessian eigenspectra, can be very inexact.
In particular, apart from a few special cases (including linear least squares and the logistic regression model in the case of homogeneous features), most Hessians of G-GLM are \emph{not} close to the Mar{\u c}enko–Pastur and/or the semicircle law.
Instead, the corresponding Hessian depends in a complicated way on the interplay between the input feature structure, the underlying response model, and the loss function.
These tradeoffs can be precisely characterized by the proposed analysis in the G-GLM model we consider.

\smallskip
The G-GLM describes a generalized linear relation between the input feature $\x_i \in \RR^p$ and the corresponding response $y_i$, in the sense that there exists some parameters $\w_* \in \RR^p$ such that for given $\w_*^\T \x_i$, the response $y_i$ is independently drawn from
\begin{equation}\label{eq:def-GLM}
  y_i \sim f (y \mid \w_*^\T \x_i) ,
\end{equation}
for some conditional density function $f(\cdot \mid \cdot)$. 
This setup extends GLMs such as the logistic~model
\begin{equation}\label{eq:def-logistic}
  \mathbb P(y = 1 \mid \w_*^\T \x) = (1 + e^{-\w_*^\T \x})^{-1}, \quad y \in \{ -1, 1 \}, 
\end{equation}
and it covers a large family of models in many applications in statistics and ML.
Other examples of G-GLM models include:
\begin{itemize}
	\item the (noisy) nonlinear factor model \cite{child1990essentials}, where $y \sim \NN( g(\w_*^\T \x), \sigma^2 )$ for some nonlinear linking function $g: \RR \to \RR$ and $\sigma > 0$;
	\item the (noiseless) phase retrieval model \cite{fienup1982phase}, with $y = (\w_*^\T \x)^2$, in which case we wish to reconstruct $\w_*$ from its magnitude measurements; and
	\item the single-layer NN model $y = \sigma(\w_*^\T \x)$, for some nonlinear activation function $\sigma(t)$ such as the $\tanh$-sigmoid $\sigma(t) = \tanh(t)$.
\end{itemize}

For a given training set $\{ (\x_i ,y_i) \}_{i=1}^n$ of size $n$, the standard approach to obtain/recover the parameter $\w_* \in \RR^p$ is to solve the following optimization problem
\begin{equation}\label{eq:optim}
  \textstyle \min_\w L(\w) = \min_\w \frac1n \sum_{i=1}^n \ell(y_i, \w^\T \x_i),
\end{equation}
for some loss function $\ell (y, h): \RR \times \RR \to \RR$.
This includes, e.g., the negative log-likelihood of the observation model within the maximum likelihood estimation framework \cite{hastie2009elements}, such as the logistic loss $\ell(y, h) = \ln(1+e^{- y h})$ in the case of the logistic model \eqref{eq:def-logistic}.
In many applications, however, the optimization problem in \eqref{eq:optim} may not be convex \cite{mason2000boosting,wu2007robust,chapelle2008tighter},%
\footnote{As we shall see, in some non-convex cases of G-GLMs, the top Hessian eigenvector can be shown to correlate positively with the sought-for model parameter $\w_*$ and therefore be used as the initialization of gradient descent methods \cite{candes2015phase,keshavan2010matrix,jain2013low}. This motivates our study of the possible isolated Hessian eigenvalue-eigenvector pairs.}
and it can be NP-hard in general (the noiseless phase retrieval model $y = (\w_*^\T \x)^2$ with the square loss $\ell(y, h) = (y - h^2)^2$ is such an example \cite{candes2015phase}).

\subsection{Our main contributions}

The main contribution of this work is the \emph{exact asymptotic characterization} of the Hessian eigenspectrum for the family of G-GLMs in \eqref{eq:def-GLM}, in the high dimensional regime where the feature dimension $p$ and the sample size $n$ are both large and comparable. 
Our main result is that we establish:
\begin{enumerate}
  \item 
  the limiting eigenvalue distribution of the Hessian matrix (Theorem~\ref{theo:lsd}); and
  \item 
  the behavior of (possible) isolated eigenvalues and the associated eigenvectors (Theorem~\ref{theo:spike}~and~\ref{theo:eigenvector});
\end{enumerate}
as a function of the dimension ratio $c = \lim p/n$, the feature (first-~and second-order) statistics, the loss function $\ell$ in \eqref{eq:optim}, and the underlying response model in \eqref{eq:def-GLM}. 
Our results are based on a technical result of independent interest:
\begin{enumerate}
\setItemnumber{3}
  \item
  a \emph{deterministic equivalent} 
  (Theorem~\ref{theo:DE}) of the random \emph{resolvent} $\Q(z) = (\H - z \I_p)^{-1}$, for $z \in \CC$ not an eigenvalue of $\H$, of the generalized sample covariance:
  \begin{equation}\label{eq:def-H}
    \H = \textstyle \frac1n \sum_{i=1}^n \ell''(y_i, \w^\T \x_i) \x_i \x_i^\T \equiv \frac1n \X \D \X^\T,
  \end{equation}
\end{enumerate}
for $\D \equiv \diag \{ \ell''(y_i, \w^\T \x_i) \}_{i=1}^n$, $\X = [\x_1, \ldots, \x_n] \in \RR^{p \times n}$ and $\ell''(y,h) \equiv \partial^2 \ell(y, h)/\partial h^2$.
The analysis is performed in the high dimensional limit as $n,p \to \infty$ with $p/n \to c \in (0,\infty)$, under the setting of input Gaussian feature $\x_i \sim \NN(\bmu, \C)$, with some statistical mean $\bmu \in \RR^p$ and positive definite covariance $\C \in \RR^{p \times p}$. 
We also demonstrate our results empirically by showing that:
\begin{enumerate}
\setItemnumber{4}
  \item 
  for a given response model \eqref{eq:def-GLM}, the Hessian eigenvalue distribution depends on the choice of loss function and the data/feature statistics in an intrinsic manner, e.g., bounded versus unbounded support (Figure~\ref{fig:log-exp-lsd}) and single- versus multi-bulk (Figure~\ref{fig:logistic-lsd}); and
  \item 
  there may exist two \emph{qualitatively} different spikes---the well-known one due to \emph{data signal} $\bmu$, and the other less-obvious one due to the \emph{underlying model itself}---which may appear on different sides of the main bulk, and their associated phase transition behaviors are different and are characterized (Figure~\ref{fig:mu-spikes} versus~\ref{fig:logistic-spike-w}).
\end{enumerate}



\medskip

To illustrate pictorially our contributions, we examine in Figure~\ref{fig:intro} the Hessian eigenspectra as a function of the response model \eqref{eq:def-GLM}, the loss function $\ell$, and the feature statistics $\bmu, \C$ (that characterize the statistical property of the data/feature). In Figure~\ref{fig:intro-a}~and~\ref{fig:intro-b}, we compare the Hessian eigenvalues for the logistic model in \eqref{eq:def-logistic} with the (maximum likelihood) logistic loss $\ell(y, h) = \ln(1+e^{- y h})$, for different choices of parameter $\w$. A nontrivial interplay between the response model parameter $\w_*$, the feature statistics $\bmu,\C$, as well as $\w$ is reflected by the range of the Hessian eigenvalue support and an additional right-hand spike in Figure~\ref{fig:intro-b} (as opposed to Figure~\ref{fig:intro-a}), as confirmed by our theory.

For the phase retrieval model $y = (\w_*^\T \x)^2$ with square loss $\ell(y, h) = (y - h^2 )^2/4$, the non-convex nature of the problem is reflected by a (relatively large) fraction of negative Hessian eigenvalues in Figure~\ref{fig:intro-c}. We also note that the top eigenvector (that corresponds to the largest eigenvalue) contains structural information of the underlying model, in the sense that it is positively correlated with $\w_*$, as predicted by the theory in Figure~\ref{fig:intro-d}, and that there exists an additional negative eigenvalue. This is connected to the Hessian-based initialization scheme discussed at the end of Section~\ref{subsec:our_approach} and in Section~\ref{subsubsec:related} below widely used in non-convex problems.

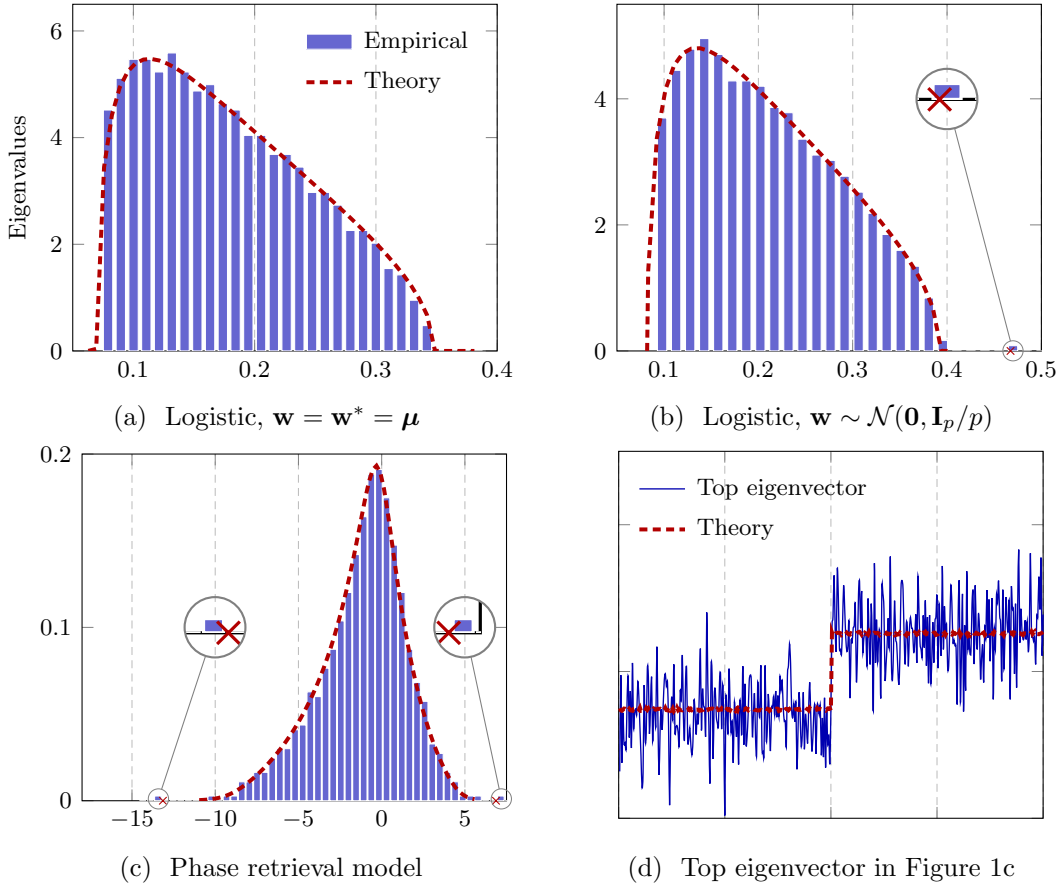
\begin{figure}[!htb]
\vskip 0.1in
\begin{center}
\begin{subfigure}[t]{0.45\columnwidth}
  \begin{tikzpicture}[font=\footnotesize]
    \renewcommand{\axisdefaulttryminticks}{4} 
    \pgfplotsset{every major grid/.style={densely dashed}}       
    \pgfplotsset{every axis legend/.style={cells={anchor=west},fill=none, at={(0.98,0.98)}, anchor=north east, font=\footnotesize }}
    \begin{axis}[
      width=1\linewidth,
      xmin=0.05,
      ymin=0,
      xmax=0.4,
      ymax=6.5,
      bar width=3.5pt,
      grid=major,
      ymajorgrids=false,
      scaled ticks=true,
      xlabel={},
      ylabel={Eigenvalues}
      ]
      \addplot+[ybar,mark=none,draw=white,fill=BLUE!60!white,area legend] coordinates{
      (0.079127,4.517544)(0.089641,5.111958)(0.100156,5.468606)(0.110670,5.468606)(0.121185,5.230841)(0.131700,5.587489)(0.142214,5.230841)(0.152729,4.874193)(0.163243,4.993075)(0.173758,4.636427)(0.184272,4.517544)(0.194787,4.042013)(0.205302,4.042013)(0.215816,3.685365)(0.226331,3.685365)(0.236845,3.447600)(0.247360,2.972069)(0.257874,2.972069)(0.268389,2.734303)(0.278903,2.258772)(0.289418,2.258772)(0.299933,2.021007)(0.310447,1.545476)(0.320962,1.426593)(0.331476,0.951062)(0.341991,0.475531)(0.352505,0.000000)(0.363020,0.000000)(0.373535,0.000000)(0.384049,0.000000)
      };
      \addplot[densely dashed,RED,line width=1.5pt] coordinates{
      (0.062834,0.001568)(0.069332,0.039432)(0.075829,3.445968)(0.082326,4.424502)(0.088824,4.945368)(0.095321,5.237332)(0.101818,5.393890)(0.108315,5.463639)(0.114813,5.475043)(0.121310,5.446012)(0.127807,5.388362)(0.134305,5.310176)(0.140802,5.217119)(0.147299,5.113239)(0.153796,5.001480)(0.160294,4.883999)(0.166791,4.762391)(0.173288,4.637835)(0.179786,4.511208)(0.186283,4.383153)(0.192780,4.254139)(0.199278,4.124491)(0.205775,3.994437)(0.212272,3.864105)(0.218769,3.733560)(0.225267,3.602803)(0.231764,3.471783)(0.238261,3.340397)(0.244759,3.208500)(0.251256,3.075890)(0.257753,2.942320)(0.264251,2.807481)(0.270748,2.670996)(0.277245,2.532400)(0.283742,2.391113)(0.290240,2.246404)(0.296737,2.097331)(0.303234,1.942638)(0.309732,1.780604)(0.316229,1.608734)(0.322726,1.423158)(0.329223,1.217284)(0.335721,0.977977)(0.342218,0.669766)(0.348715,0.001744)(0.355213,0.000437)(0.361710,0.000292)(0.368207,0.000225)(0.374705,0.000184)(0.381202,0.000157)
      };
      \legend{ {Empirical}, {Theory} }; 
    \end{axis}
  \end{tikzpicture}
  \caption{ { Logistic, $\w = \w^* = \bmu$ } } \label{fig:intro-a}
\end{subfigure}
\begin{subfigure}[t]{0.45\columnwidth}
\centering
  \begin{tikzpicture}[font=\footnotesize,spy using outlines]
    \renewcommand{\axisdefaulttryminticks}{4} 
    \pgfplotsset{every major grid/.append style={densely dashed}}       
    \pgfplotsset{every axis legend/.append style={cells={anchor=west},fill=none, at={(0.98,0.98)}, anchor=north east, font=\footnotesize }}
    \begin{axis}[
      width=1\linewidth,
      xmin=0.05,
      ymin=0,
      xmax=0.5,
      ymax=5.5,
      bar width=3.5pt,
      grid=major,
      ymajorgrids=false,
      scaled ticks=true,
      xlabel={},
      ylabel={}
      ]
      \addplot+[ybar,mark=none,draw=white,fill=BLUE!60!white,area legend] coordinates{
      (0.097894,3.696031)(0.112775,4.452038)(0.127656,4.788041)(0.142537,4.956042)(0.157418,4.704040)(0.172298,4.284036)(0.187179,4.284036)(0.202060,4.200036)(0.216941,3.864033)(0.231822,3.780032)(0.246703,3.360028)(0.261583,3.108026)(0.276464,3.024026)(0.291345,2.772023)(0.306226,2.520021)(0.321107,2.184019)(0.335988,1.848016)(0.350868,1.596014)(0.365749,1.344011)(0.380630,0.840007)(0.395511,0.168001)(0.410392,0.000000)(0.425272,0.000000)(0.440153,0.000000)(0.455034,0.000000)(0.469915,0.084001)(0.484796,0.000000)(0.499677,0.000000)(0.514557,0.000000)(0.529438,0.000000)
      };
      \addplot[densely dashed,RED,line width=1.5pt] coordinates{
      (0.082008,0.001127)(0.083150,1.272566)(0.092293,3.373887)(0.101436,4.145756)(0.110578,4.530093)(0.119721,4.719987)(0.128864,4.796797)(0.138006,4.802782)(0.147149,4.762390)(0.156292,4.690725)(0.165434,4.597565)(0.174577,4.489421)(0.183720,4.370750)(0.192862,4.244631)(0.202005,4.113227)(0.211148,3.978042)(0.220290,3.840127)(0.229433,3.700187)(0.238576,3.558677)(0.247718,3.415862)(0.256861,3.271834)(0.266003,3.126561)(0.275146,2.979890)(0.284289,2.831542)(0.293431,2.681127)(0.302574,2.528095)(0.311717,2.371727)(0.320859,2.211060)(0.330002,2.044802)(0.339145,1.871137)(0.348287,1.687435)(0.357430,1.489597)(0.366573,1.270590)(0.375715,1.016275)(0.384858,0.687824)(0.394001,0.001073)(0.403143,0.000308)
      };
      \addplot+[only marks,mark=x,RED] coordinates{ (0.4673,0) };
      \coordinate (spy0) at (axis cs:0.47,0.005); \coordinate (spypoint0) at (axis cs:0.4,4);
    \begin{scope}
      \spy[black!50!white,size=0.8cm,circle,connect spies,magnification=3] on (spy0) in node [fill=none] at (spypoint0);
      \end{scope}
    \end{axis}
  \end{tikzpicture}
  \caption{ { Logistic, $\w \sim \NN(\zo, \I_p/p)$ } } \label{fig:intro-b}
\end{subfigure}
\bigskip%
\begin{subfigure}[t]{0.45\columnwidth}
\centering
  \begin{tikzpicture}[font=\footnotesize,spy using outlines]
    \renewcommand{\axisdefaulttryminticks}{4} 
    \pgfplotsset{every major grid/.append style={densely dashed}}       
    \pgfplotsset{every axis legend/.style={cells={anchor=west},fill=none, at={(0,0.98)}, anchor=north west, font=\footnotesize }}
    \pgfplotsset{every major grid/.style={densely dashed}}
    \begin{axis}[
      width=1\linewidth,
      xmin=-18,
      ymin=0,
      xmax=7.5,
      ymax=0.2,
      ytick={0,0.1,0.2},
      bar width=2.5pt,
      grid=major,
      ymajorgrids=false,
      scaled ticks=true,
      xlabel={},
      ylabel={}
      ]
      \addplot+[ybar,mark=none,draw=white,fill=BLUE!60!white,area legend] coordinates{
      (-14.356836,0.000000)(-13.899090,0.000000)(-13.441344,0.002731)(-12.983598,0.000000)(-12.525852,0.000000)(-12.068106,0.000000)(-11.610360,0.000000)(-11.152614,0.000000)(-10.694868,0.000000)(-10.237122,0.002731)(-9.779376,0.000000)(-9.321630,0.002731)(-8.863884,0.002731)(-8.406138,0.010923)(-7.948392,0.010923)(-7.490646,0.016385)(-7.032900,0.016385)(-6.575154,0.024577)(-6.117408,0.030038)(-5.659662,0.030038)(-5.201916,0.040962)(-4.744170,0.043692)(-4.286424,0.062808)(-3.828678,0.060077)(-3.370932,0.076462)(-2.913186,0.087385)(-2.455441,0.103769)(-1.997695,0.120154)(-1.539949,0.142000)(-1.082203,0.163846)(-0.624457,0.188423)(-0.166711,0.191154)(0.291035,0.174769)(0.748781,0.147462)(1.206527,0.120154)(1.664273,0.087385)(2.122019,0.068269)(2.579765,0.057346)(3.037511,0.032769)(3.495257,0.027308)(3.953003,0.016385)(4.410749,0.010923)(4.868495,0.002731)(5.326241,0.002731)(5.783987,0.002731)(6.241733,0.000000)(6.699479,0.000000)(7.157225,0.002731)(7.614971,0.000000)(8.072717,0.000000)
      };
      \addplot[densely dashed,RED,line width=1.5pt] coordinates{
      (-10.960730,0.000359)(-10.734169,0.000480)(-10.507608,0.000644)(-10.281047,0.000867)(-10.054486,0.001170)(-9.827925,0.001577)(-9.601364,0.002115)(-9.374802,0.002810)(-9.148241,0.003679)(-8.921680,0.004727)(-8.695119,0.005952)(-8.468558,0.007340)(-8.241997,0.008877)(-8.015436,0.010549)(-7.788875,0.012350)(-7.562313,0.014273)(-7.335752,0.016313)(-7.109191,0.018472)(-6.882630,0.020751)(-6.656069,0.023157)(-6.429508,0.025692)(-6.202947,0.028366)(-5.976385,0.031188)(-5.749824,0.034171)(-5.523263,0.037323)(-5.296702,0.040662)(-5.070141,0.044203)(-4.843580,0.047965)(-4.617019,0.051968)(-4.390458,0.056236)(-4.163896,0.060796)(-3.937335,0.065676)(-3.710774,0.070911)(-3.484213,0.076537)(-3.257652,0.082597)(-3.031091,0.089138)(-2.804530,0.096209)(-2.577968,0.103860)(-2.351407,0.112144)(-2.124846,0.121106)(-1.898285,0.130767)(-1.671724,0.141109)(-1.445163,0.152027)(-1.218602,0.163253)(-0.992041,0.174241)(-0.765479,0.183998)(-0.538918,0.190969)(-0.312357,0.193255)(-0.085796,0.189467)(0.140765,0.179813)(0.367326,0.166101)(0.593887,0.150561)(0.820449,0.134860)(1.047010,0.119921)(1.273571,0.106138)(1.500132,0.093609)(1.726693,0.082291)(1.953254,0.072084)(2.179815,0.062870)(2.406376,0.054537)(2.632938,0.046984)(2.859499,0.040122)(3.086060,0.033878)(3.312621,0.028194)(3.539182,0.023032)(3.765743,0.018372)(3.992304,0.014214)(4.218866,0.010588)(4.445427,0.007548)(4.671988,0.005142)(4.898549,0.003372)(5.125110,0.002163)(5.351671,0.001379)(5.578232,0.000883)
      };
      \addplot+[only marks,mark=x,RED] coordinates{ (6.8673,0) (-13.1429,0) };
      \coordinate (spy0) at (axis cs:7.2,0.001); \coordinate (spypoint0) at (axis cs:5,0.1);
      \coordinate (spy1) at (axis cs:-13.4,0.001); \coordinate (spypoint1) at (axis cs:-10,0.1);
    \begin{scope}
      \spy[black!50!white,size=.8cm,circle,connect spies,magnification=3] on (spy0) in node [fill=none] at (spypoint0);
      \spy[black!50!white,size=.8cm,circle,connect spies,magnification=3] on (spy1) in node [fill=none] at (spypoint1);
      \end{scope}
    \end{axis}
  \end{tikzpicture}
  \caption{ { Phase retrieval model} } \label{fig:intro-c}
\end{subfigure}
\begin{subfigure}[t]{0.45\columnwidth}
\centering
\begin{tikzpicture}[font=\footnotesize]
  \renewcommand{\axisdefaulttryminticks}{4} 
  \pgfplotsset{every major grid/.style={densely dashed}}
  \pgfplotsset{every axis legend/.style={cells={anchor=west},fill=none,
  at={(0.02,0.98)}, anchor=north west, font=\footnotesize}}
  \begin{axis}[
  width=1\linewidth,
  height=.9\linewidth,
  xmin=0,ymin=-0.1,xmax=400,ymax=0.15,
  scaled ticks = false,
  xticklabels={},
  yticklabels={},
  grid=major,
  ymajorgrids=false,
  xlabel={},
  ylabel={}]
  \addplot[smooth,BLUE,line width=.5pt] plot coordinates{
  (1,-0.068066)(2,-0.002438)(3,-0.062203)(4,-0.028113)(5,-0.049016)(6,-0.048447)(7,-0.052379)(8,-0.017244)(9,-0.003679)(10,-0.044999)(11,-0.013293)(12,-0.052129)(13,0.005828)(14,-0.038130)(15,-0.028524)(16,-0.025445)(17,-0.029140)(18,-0.006990)(19,-0.016768)(20,-0.037026)(21,-0.056784)(22,-0.078878)(23,-0.068925)(24,-0.009263)(25,-0.046696)(26,-0.041508)(27,0.001781)(28,-0.003824)(29,-0.042494)(30,-0.023271)(31,-0.003527)(32,-0.047596)(33,-0.026156)(34,-0.025454)(35,-0.036877)(36,0.016882)(37,-0.033169)(38,-0.070355)(39,-0.041812)(40,-0.025701)(41,-0.039810)(42,-0.021903)(43,-0.005721)(44,-0.001610)(45,-0.010121)(46,-0.007551)(47,-0.036752)(48,-0.091069)(49,-0.037062)(50,-0.017188)(51,-0.033275)(52,0.004860)(53,-0.040898)(54,-0.038007)(55,-0.036887)(56,-0.041260)(57,-0.050690)(58,-0.013382)(59,-0.019561)(60,0.011264)(61,-0.030668)(62,-0.046470)(63,-0.006688)(64,0.001885)(65,-0.046180)(66,-0.046682)(67,-0.038312)(68,-0.049789)(69,-0.056946)(70,-0.047605)(71,-0.019693)(72,-0.027999)(73,-0.042925)(74,-0.029077)(75,-0.035695)(76,-0.014675)(77,0.015605)(78,-0.040141)(79,-0.011980)(80,-0.023127)(81,-0.025620)(82,-0.042054)(83,0.040233)(84,-0.007725)(85,-0.037022)(86,-0.076370)(87,0.000813)(88,-0.001764)(89,-0.059541)(90,-0.018081)(91,-0.011911)(92,-0.030531)(93,-0.035166)(94,-0.037283)(95,-0.007543)(96,-0.001938)(97,-0.039271)(98,-0.009038)(99,-0.031785)(100,-0.098116)(101,-0.018869)(102,-0.027619)(103,-0.042024)(104,0.001931)(105,-0.024456)(106,-0.041694)(107,-0.034029)(108,-0.062338)(109,-0.021818)(110,-0.030640)(111,-0.044516)(112,-0.023555)(113,-0.063339)(114,-0.021374)(115,-0.024939)(116,-0.043033)(117,0.003628)(118,-0.049878)(119,-0.008397)(120,-0.049155)(121,-0.044701)(122,-0.029437)(123,-0.006716)(124,-0.002630)(125,-0.006823)(126,-0.044472)(127,-0.028761)(128,-0.015111)(129,-0.045827)(130,-0.056456)(131,-0.007805)(132,-0.024496)(133,-0.030307)(134,0.004247)(135,-0.050600)(136,-0.054611)(137,-0.055609)(138,0.004125)(139,-0.045526)(140,-0.040812)(141,-0.015304)(142,-0.032198)(143,-0.054217)(144,-0.025745)(145,-0.024771)(146,-0.017060)(147,-0.029839)(148,-0.013281)(149,-0.031167)(150,-0.016415)(151,-0.030770)(152,-0.003462)(153,-0.012742)(154,-0.005816)(155,-0.023400)(156,-0.033553)(157,-0.034301)(158,-0.051006)(159,-0.003145)(160,0.004062)(161,0.018782)(162,0.005635)(163,-0.025618)(164,-0.031431)(165,-0.021797)(166,-0.034114)(167,-0.021043)(168,-0.066626)(169,-0.040855)(170,-0.014191)(171,-0.030306)(172,-0.050796)(173,-0.043601)(174,-0.014139)(175,-0.052981)(176,-0.029580)(177,-0.037131)(178,-0.049030)(179,-0.026540)(180,-0.007317)(181,-0.029882)(182,-0.053969)(183,-0.045359)(184,-0.026460)(185,-0.024127)(186,-0.056145)(187,-0.024412)(188,-0.023593)(189,-0.029119)(190,-0.023727)(191,-0.057076)(192,-0.031650)(193,-0.058332)(194,-0.017907)(195,-0.065452)(196,-0.026435)(197,-0.049790)(198,-0.043325)(199,-0.063201)(200,-0.026089)(201,0.029468)(202,0.037537)(203,0.065472)(204,0.028975)(205,0.014974)(206,0.069570)(207,0.051164)(208,0.016505)(209,0.015056)(210,0.011895)(211,0.032690)(212,0.058720)(213,0.042627)(214,0.023196)(215,0.013812)(216,0.017193)(217,0.034041)(218,-0.018154)(219,0.020127)(220,0.019526)(221,0.029446)(222,0.031455)(223,0.046185)(224,-0.006378)(225,0.012394)(226,0.052486)(227,0.061231)(228,0.029473)(229,0.054861)(230,0.013878)(231,0.022133)(232,0.026110)(233,0.025030)(234,0.003529)(235,0.049747)(236,0.020494)(237,0.005551)(238,0.046222)(239,0.018281)(240,-0.026886)(241,0.073762)(242,0.005923)(243,0.035296)(244,0.006282)(245,0.023055)(246,0.026209)(247,-0.018374)(248,0.055830)(249,0.053099)(250,0.042271)(251,-0.005076)(252,0.004904)(253,0.018483)(254,0.036328)(255,0.016690)(256,0.078326)(257,0.007063)(258,-0.006019)(259,0.015301)(260,0.048241)(261,0.022986)(262,0.018674)(263,0.003078)(264,-0.009121)(265,0.009337)(266,0.055592)(267,0.062021)(268,0.055716)(269,0.023899)(270,0.004854)(271,0.031868)(272,0.030514)(273,0.003514)(274,0.028414)(275,0.036804)(276,0.029785)(277,0.028202)(278,0.013018)(279,0.019374)(280,0.038951)(281,0.005850)(282,0.064373)(283,0.041114)(284,0.019199)(285,0.014922)(286,0.010083)(287,0.038213)(288,0.016255)(289,0.048470)(290,0.003967)(291,0.037075)(292,-0.008017)(293,0.038440)(294,-0.007496)(295,0.005239)(296,0.033654)(297,0.006004)(298,0.033275)(299,0.037514)(300,0.051331)(301,0.002280)(302,0.046148)(303,0.024637)(304,0.012475)(305,0.071894)(306,0.015064)(307,0.044393)(308,0.024039)(309,0.032012)(310,0.003013)(311,0.003899)(312,0.029271)(313,0.058712)(314,0.046560)(315,0.020077)(316,0.021462)(317,0.017451)(318,0.015155)(319,-0.022485)(320,0.041048)(321,0.059769)(322,-0.020070)(323,0.010858)(324,0.052880)(325,-0.027235)(326,0.031520)(327,-0.013006)(328,0.068324)(329,0.016715)(330,0.023408)(331,0.012586)(332,0.013914)(333,0.046108)(334,0.053832)(335,0.005057)(336,0.058751)(337,0.022399)(338,0.015979)(339,0.039033)(340,0.029523)(341,-0.005187)(342,0.009304)(343,0.017871)(344,-0.007481)(345,0.030683)(346,0.022422)(347,0.055456)(348,0.023010)(349,0.049084)(350,0.007001)(351,0.033181)(352,0.048009)(353,-0.018972)(354,0.029308)(355,0.014789)(356,0.051428)(357,0.046503)(358,0.061910)(359,0.038468)(360,0.031679)(361,-0.006717)(362,0.055746)(363,0.002248)(364,0.023540)(365,0.050699)(366,0.033868)(367,0.027086)(368,0.047804)(369,-0.018982)(370,0.055782)(371,0.043099)(372,0.040735)(373,-0.001587)(374,0.046187)(375,-0.010979)(376,0.043372)(377,0.083049)(378,0.020425)(379,0.042747)(380,0.028813)(381,0.022094)(382,0.023174)(383,0.034847)(384,0.021824)(385,0.034054)(386,0.013693)(387,0.026017)(388,0.080552)(389,0.027038)(390,0.024947)(391,0.020550)(392,0.023130)(393,0.004840)(394,0.037517)(395,0.037781)(396,0.064583)(397,0.042302)(398,0.077596)(399,0.010762)(400,0.038815)
  };
  \addplot[densely dashed,RED,line width=1.5pt] plot coordinates{
  (1,-0.025092)(2,-0.026478)(3,-0.026752)(4,-0.026070)(5,-0.025260)(6,-0.025353)(7,-0.025866)(8,-0.025093)(9,-0.026432)(10,-0.025398)(11,-0.026373)(12,-0.026664)(13,-0.026777)(14,-0.025454)(15,-0.025246)(16,-0.026271)(17,-0.025251)(18,-0.025203)(19,-0.025930)(20,-0.024591)(21,-0.026905)(22,-0.026932)(23,-0.026075)(24,-0.026032)(25,-0.027091)(26,-0.025110)(27,-0.026153)(28,-0.025154)(29,-0.026464)(30,-0.026158)(31,-0.026061)(32,-0.026373)(33,-0.025352)(34,-0.025789)(35,-0.026044)(36,-0.025731)(37,-0.025361)(38,-0.025525)(39,-0.025356)(40,-0.025911)(41,-0.025934)(42,-0.025360)(43,-0.026367)(44,-0.024801)(45,-0.026321)(46,-0.026491)(47,-0.024126)(48,-0.025486)(49,-0.026877)(50,-0.026678)(51,-0.025217)(52,-0.027065)(53,-0.026096)(54,-0.026807)(55,-0.024812)(56,-0.026389)(57,-0.025634)(58,-0.026148)(59,-0.026358)(60,-0.026463)(61,-0.025887)(62,-0.026773)(63,-0.026351)(64,-0.026467)(65,-0.026175)(66,-0.025324)(67,-0.026149)(68,-0.025402)(69,-0.026210)(70,-0.025450)(71,-0.026089)(72,-0.025426)(73,-0.026004)(74,-0.025965)(75,-0.026026)(76,-0.026651)(77,-0.026154)(78,-0.026574)(79,-0.026226)(80,-0.026040)(81,-0.024964)(82,-0.026484)(83,-0.026053)(84,-0.026275)(85,-0.025527)(86,-0.025594)(87,-0.025322)(88,-0.025181)(89,-0.025381)(90,-0.025705)(91,-0.025793)(92,-0.027101)(93,-0.025593)(94,-0.026329)(95,-0.026089)(96,-0.025950)(97,-0.025091)(98,-0.025993)(99,-0.026041)(100,-0.025143)(101,-0.026645)(102,-0.025604)(103,-0.026157)(104,-0.025753)(105,-0.026472)(106,-0.025001)(107,-0.024621)(108,-0.024978)(109,-0.026052)(110,-0.026887)(111,-0.025731)(112,-0.026670)(113,-0.024678)(114,-0.025047)(115,-0.025763)(116,-0.026297)(117,-0.025940)(118,-0.026183)(119,-0.025515)(120,-0.025106)(121,-0.025865)(122,-0.026460)(123,-0.025512)(124,-0.025628)(125,-0.025643)(126,-0.025121)(127,-0.025278)(128,-0.026098)(129,-0.026177)(130,-0.026367)(131,-0.026553)(132,-0.026555)(133,-0.026386)(134,-0.026110)(135,-0.026685)(136,-0.026674)(137,-0.026158)(138,-0.025109)(139,-0.025420)(140,-0.025317)(141,-0.026000)(142,-0.026769)(143,-0.025797)(144,-0.025929)(145,-0.025830)(146,-0.025529)(147,-0.025485)(148,-0.026221)(149,-0.025271)(150,-0.025390)(151,-0.026983)(152,-0.026690)(153,-0.026205)(154,-0.026489)(155,-0.025010)(156,-0.026461)(157,-0.025737)(158,-0.025401)(159,-0.025648)(160,-0.026408)(161,-0.025662)(162,-0.026635)(163,-0.026613)(164,-0.027411)(165,-0.026183)(166,-0.025861)(167,-0.025718)(168,-0.025615)(169,-0.026083)(170,-0.025548)(171,-0.026200)(172,-0.026100)(173,-0.025847)(174,-0.026021)(175,-0.025707)(176,-0.024662)(177,-0.025365)(178,-0.025654)(179,-0.026448)(180,-0.024816)(181,-0.025693)(182,-0.025439)(183,-0.026056)(184,-0.026311)(185,-0.026480)(186,-0.026211)(187,-0.024924)(188,-0.025767)(189,-0.025791)(190,-0.025448)(191,-0.025683)(192,-0.026737)(193,-0.026265)(194,-0.026499)(195,-0.026726)(196,-0.026007)(197,-0.025187)(198,-0.025409)(199,-0.025015)(200,-0.026651)(201,0.026615)(202,0.024611)(203,0.024449)(204,0.025390)(205,0.026171)(206,0.026422)(207,0.026345)(208,0.026378)(209,0.026109)(210,0.024833)(211,0.025394)(212,0.025081)(213,0.024520)(214,0.025452)(215,0.026475)(216,0.025803)(217,0.025252)(218,0.025138)(219,0.026072)(220,0.025794)(221,0.026042)(222,0.025595)(223,0.025412)(224,0.026170)(225,0.026349)(226,0.026955)(227,0.027066)(228,0.026201)(229,0.025029)(230,0.026048)(231,0.024873)(232,0.026227)(233,0.025649)(234,0.025773)(235,0.025068)(236,0.026856)(237,0.025291)(238,0.025266)(239,0.025907)(240,0.025921)(241,0.026485)(242,0.026262)(243,0.026062)(244,0.026353)(245,0.025935)(246,0.024783)(247,0.025513)(248,0.025832)(249,0.026180)(250,0.026592)(251,0.026037)(252,0.026370)(253,0.025503)(254,0.025805)(255,0.025615)(256,0.025668)(257,0.025250)(258,0.026115)(259,0.026632)(260,0.025093)(261,0.025829)(262,0.025013)(263,0.025318)(264,0.025301)(265,0.026461)(266,0.025691)(267,0.026350)(268,0.026363)(269,0.026293)(270,0.025292)(271,0.024585)(272,0.026150)(273,0.025862)(274,0.025790)(275,0.025715)(276,0.026157)(277,0.025323)(278,0.025245)(279,0.026604)(280,0.026143)(281,0.025065)(282,0.025494)(283,0.026573)(284,0.025303)(285,0.026365)(286,0.026008)(287,0.026465)(288,0.026147)(289,0.026909)(290,0.025522)(291,0.025957)(292,0.025796)(293,0.025470)(294,0.026266)(295,0.026099)(296,0.025592)(297,0.026511)(298,0.026439)(299,0.024887)(300,0.025138)(301,0.025436)(302,0.026523)(303,0.025650)(304,0.025209)(305,0.024820)(306,0.025730)(307,0.026392)(308,0.025041)(309,0.026532)(310,0.026092)(311,0.025570)(312,0.025541)(313,0.025725)(314,0.025224)(315,0.026561)(316,0.025716)(317,0.025589)(318,0.026247)(319,0.027032)(320,0.024946)(321,0.026640)(322,0.024591)(323,0.026014)(324,0.026706)(325,0.025135)(326,0.025402)(327,0.024855)(328,0.025749)(329,0.024391)(330,0.025457)(331,0.024691)(332,0.026430)(333,0.026585)(334,0.026832)(335,0.025711)(336,0.025724)(337,0.026315)(338,0.025857)(339,0.025490)(340,0.025735)(341,0.026538)(342,0.024256)(343,0.026047)(344,0.026537)(345,0.024840)(346,0.025639)(347,0.025920)(348,0.025993)(349,0.027156)(350,0.025973)(351,0.025920)(352,0.025760)(353,0.025341)(354,0.027025)(355,0.026283)(356,0.025923)(357,0.025874)(358,0.026100)(359,0.026460)(360,0.025942)(361,0.025960)(362,0.025619)(363,0.024539)(364,0.026416)(365,0.026255)(366,0.025530)(367,0.025817)(368,0.026437)(369,0.026606)(370,0.025924)(371,0.026375)(372,0.025237)(373,0.026351)(374,0.025688)(375,0.025059)(376,0.025849)(377,0.026226)(378,0.026177)(379,0.025555)(380,0.025152)(381,0.024791)(382,0.025998)(383,0.025943)(384,0.025844)(385,0.026138)(386,0.026117)(387,0.025527)(388,0.025654)(389,0.027116)(390,0.025491)(391,0.026437)(392,0.027206)(393,0.024969)(394,0.025535)(395,0.025457)(396,0.025384)(397,0.025718)(398,0.025938)(399,0.026446)(400,0.026977)
  };
  \legend{ { Top eigenvector }, {Theory} }; 
  \end{axis}
  \end{tikzpicture}
  \caption{ { Top eigenvector in Figure~\ref{fig:intro-c} } } \label{fig:intro-d}
\end{subfigure}
\caption{ Illustration of our main results: eigenspectral properties of the Hessian $\H(\w)$ of G-GLM with $p=800$, $n = 6\,000$ and $\C = \I_p$. \textbf{(Top)}: presence versus absence of a right-hand side spike for different $\w$, logistic model \eqref{eq:def-logistic} with logistic loss $\ell(y, h) = \ln( 1 + e^{- y h} )$ and $\w_* = \bmu \sim \NN( \zo, \I_p/p)$. There is a nontrivial interplay between the response model \eqref{eq:def-GLM}, the feature statistics $\bmu, \C$, and $\w$, as reflected by the ``range'' of the Hessian eigenvalues and an additional right-hand spike, in Figure~\ref{fig:intro-b}~versus~\ref{fig:intro-a}, both captured by our theory. \textbf{(Bottom)}: the Hessian eigenspectrum may have a rather different shape (as opposed to the Mar\u{c}enko-Pastur-like in Figure~\ref{fig:intro-a}~and~\ref{fig:intro-b}) for the phase retrieval model \textbf{(left)} and the right-hand side isolated eigenvector is known here be an estimate of $\w_*$ \textbf{(right)}, as confirmed by our theory. With square loss $\ell(y,h) = (y - h^2)^2/4$, $\w_* = [-2 \cdot \one_{p/2};~2 \cdot \one_{p/2}]/\sqrt p$, $\w \sim \NN(\zo, \I_p/p)$ and $\bmu = \zo$.
}
\label{fig:intro}
\end{center}
\vskip -0.1in
\end{figure}

\subsection{Notation and organization of the paper}

\paragraph{Notations.} 
Throughout the paper, we follow the convention of denoting scalars by lowercase, vectors by lowercase boldface, and matrices by uppercase boldface letters. The notation $(\cdot)^\T$ denotes the transpose operator. The norm $\| \cdot \| $ is the Euclidean norm for vectors and the spectral norm for matrices. We use $\asto$ for almost sure convergence of random variables. We use $\Im[\cdot]$ to denote the imaginary part of a complex number and $\imath$ to present the imaginary unit.

\paragraph{Organization of the paper.} 

We discuss related previous efforts and preliminaries in Section~\ref{sec:background}. Our main technical results on the limiting Hessian eigenvalue distribution, as well as on the behavior of the Hessian isolated eigenvalues and associated eigenvectors are presented in Section~\ref{sec:main_results}. In Section~\ref{sec:discuss}, further discussions on the consequences of these technical results are made, together with numerical evaluations. We conclude the article in Section~\ref{sec:conclusion} with some possible future extensions.

\section{Background and Preliminaries}
\label{sec:background}

In this section, we discuss related previous efforts in Section~\ref{subsubsec:related} and some (technical) preliminaries on the spiked random matrix model and the deterministic equivalent technique in Section~\ref{sec:pre-spiked}~and~Section~\ref{subsec:pre-DE}, respectively.

\subsection{Related work}
\label{subsubsec:related}

\paragraph{Random matrix theory.} 
Random matrices of the type \eqref{eq:def-H} are related to the so-called \emph{separable covariance model} \cite{zhang2006spectral,couillet2014analysis}.
This model is of the form $\frac1n \C^{\frac12} \Z \D \Z^\T \C^{\frac12}$, for $\Z \in \RR^{p \times n}$ having, say, i.i.d.\@ standard Gaussian entries, and \emph{deterministic} or \emph{independent} matrices $\C \in \RR^{p \times p}$ and $ \D \in \RR^{n \times n}$.
Our results generalize this, in the sense that we allow the matrix $\D$ to \emph{depend} on the random matrix $\Z$, in a possibly nonlinear fashion, per \eqref{eq:def-H}. 
This is of direct interest for the Hessian matrix of the G-GLM.
We particularly show that, in the case where $\Z$ is a standard Gaussian random matrix, while the limiting spectral measure remains the same \emph{as if $\D$ was independent} (Theorem~\ref{theo:lsd}), spurious isolated eigenvalue may appear due to this dependence between $\D$ and $\Z$ (Theorem~\ref{theo:spike}~and~\ref{theo:eigenvector}).
This is directly related to the popular spectral initialization method used in non-convex optimization~\cite{lu2019phase}.

\paragraph{Hessian eigenspectra in ML models.}  
The eigenspectra of Hessian matrices arising in ML models (in particular, for NNs) have attracted considerable interest recently, both from a theoretical perspective \cite{sagun2016eigenvalues,sagun2017empirical,chaudhari2019entropy,fort2019emergent,granziol2020beyond,wu2020dissecting,pennington2017geometry,geiger2019jamming,jacot2019asymptotic} and from an empirical perspective \cite{dong2019hawq,shen2020q,yao2019pyhessian,yao2020adahessian}. 
For example, the Hessian eigenspectrum provides a characterization of the (local) loss curvature and therefore the so-called  ``sharpness'' (or ``flatness'') of a trained model.\footnote{Some claim that this corresponds to the generalization performance of the model~\cite{hochreiter1997flat,keskar2016large,chaudhari2019entropy,dinh2017sharp}.}
However, these investigations are either limited to empirical evaluation \cite{sagun2016eigenvalues,sagun2017empirical} or built upon somewhat unrealistic simplifying assumptions that reduce to the ``mixed'' behavior of Mar{\u c}enko–Pastur and semicircle law \cite{pennington2017geometry,choromanska2015loss}, or to that of the product of independent Gaussian matrices \cite{granziol2020beyond}. 
In contrast, we focus on the more tractable example of G-GLM in \eqref{eq:def-GLM}, and we provide \emph{precise} results on the Hessian eigenspectra for structural Gaussian feature on arbitrary loss functions. 
Also, instead of focusing solely on the limiting eigenvalue distribution, our results also shed novel light on the isolated eigenvalues (above and/or below the bulk) that are empirically observed in modern NNs \cite{sagun2016eigenvalues,fort2019emergent,granziol2020beyond,li2020hessian,papyan2020traces}.

\paragraph{Hessian eigenspectra: a statistical physics perspective.}

Statistical physicists are interested in the energy ``landscape'' of random functions, by means of studying the Hessian of such functions \cite{fyodorov2004complexity,fyodorov2007replica,fyodorov2020manifolds,yamada2018hessian,fyodorov2018hessian}. 
This is somewhat different from the Hessian of the loss function in optimization, and it is mainly in the context of spin glass theory.
Depending on the model, the associated Hessian eigenspectra may establish a phase transition behavior and take either a Mar\u{c}enko-Pastur or a semicircle shape, but a wide range of other phenomena have also been studied. 
These are known to have connections with ML systems such as NN models in high dimensions \cite{advani2013statistical,zdeborova2016statistical}.

\paragraph{Spectral initialization in non-convex problems.}

For non-convex problems, the performance of gradient-based methods can be very sensitive to initialization. 
%
A popular initialization method is the \emph{spectral initialization} method, which can be traced back to the Principal Hessian Directions method \cite{li1992principal}.
Similar techniques were then proposed to provide a ``good'' initialization for gradient descent in non-convex problems such as phase retrieval \cite{fienup1982phase,netrapalli2013phase,candes2013phaselift}, matrix completion \cite{keshavan2010matrix}, low-rank matrix recovery \cite{jain2013low}, blind deconvolution \cite{lee2016blind}, and sparse coding \cite{arora2015simple}.
In \cite{lu2019phase}, which was generalized by \cite{mondelli2018fundamental}, the authors evaluated the eigenspectrum asymptotics of the generalized sample covariance matrix $\frac1n \sum_{i=1}^n f(y_i) \x_i \x_i^\T$, for some processing function $f: \RR \to \RR$ and $\x_i \sim \NN(\zo, \I_p)$. 
Their technical approach is, however, limited to the case of very homogeneous features, $\x_i \sim \NN(\zo, \I_p)$, and it was particularly focused on the phase retrieval model, instead of the Hessian (that depends on the loss) more generally. 
In a sense, we generalize the analysis in \cite{lu2019phase} to the Hessian matrix of G-GLM, by developing a systematic approach to account for the data structures ($\bmu \neq \zo$ and $ \C \neq \I_p$). 

\paragraph{Scalable second-order methods.}

Second order methods are among the most powerful optimization methods that have been designed, and
there have been several attempts to use their many advantages for machine learning applications, including both theoretical developments  \cite{xu2016subsampled,wang2018giant,roosta2019subsampled}, as well as for training and/or quantizing state-of-the-art NNs \cite{yao2019pyhessian,yao2020adahessian,shen2020q,dong2019hawq,xu2020second}. We expect that our precise characterization of the Hessian sheds new light on the improved design of more efficient second-order methods.
For example, insights from the proposed analysis could shed light on improved optimization methods for particular classes of models, as well as improved strategies for model post-processing and~diagnostics.


%
\subsection{Spiked random matrix model}
\label{sec:pre-spiked}

Define the \emph{empirical spectral measure} of a random matrix $\M \in \RR^{p \times p}$ 
\begin{equation}\label{eq:esd}
  \mu_\M \equiv \frac1p \sum_{i=1}^p \delta_{\lambda_i(\M)}
\end{equation}
as the normalized counting measure of the eigenvalues $\lambda_i(\M)$ of $\M$. If the random $\mu_\M$ converges to a (non-random) probability measure $\mu$ as $p \to \infty$, we say $\mu$ is the \emph{limiting spectral measure} of the random matrix $\M$. Examples of such limiting laws include the popular Mar\u{c}enko-Pastur law \cite{marchenko1967distribution} and Wigner's semicircle law \cite{wigner1955characteristic}.

In many applications, instead of the (limiting) eigenvalue distribution, one may be more interested in the behavior of some particular eigenvalue-eigenvector pairs. For a standard (noise-like) random matrix $\M$, the term ``\emph{spiked random matrix model}'' \cite{johnstone2001distribution,benaych2011eigenvalues,bai2012sample} is used to refer to the \emph{low rank} update of $\M$, for instance the additive perturbation of the type $\M + \A$ for some (deterministic) $\A$ such that $\rank(\A) = K$ is fixed as $p \to \infty$. This type of random matrix model plays a fundamental role in the analysis of principle component analysis (PCA) \cite{johnstone2018pca,johnstone2009sparse}, sparse PCA \cite{amini2009high,ma2013sparse,cai2015optimal} and Canonical Correlation Analysis (CCA) \cite{hotelling1936relation,bao2019canonical,johnstone2020test} in high dimensions. As a concrete example, for $\x_i \sim \NN(\bmu, \I_p)$ we can write $\X = [\x_1, \ldots, \x_n] = \bmu \one_n^\T + \Z$ for $\Z \in \RR^{p \times n}$ having i.i.d.~standard Gaussian entries, so that the sample covariance matrix $\frac1n \X \X^\T = \frac1n \Z \Z^\T + \bmu \bmu^\T + \frac1n \bmu \one_n^\T \Z^\T + \frac1n \Z \one_n \bmu^\T$ follows a spiked model that is an (at most) rank-two additive perturbation of the standard Wishart random matrix $\frac1n \Z \Z^\T$. The spiked model is particularly popular and useful in charactering signal-plus-noise type behavior in the features, by focusing on the isolated eigenvalues (i.e., the ``spikes'') outside the main eigenspectrum bulk.
Precisely, if the signal strength $\| \bmu \|$ is larger than a certain threshold (that depends on the dimension ratio $p/n$), an informative eigenvalue $\hat \lambda$ of $\frac1n\X \X^\T$ can be observed to ``jump out'' from the noisy main bulk, with the associated eigenvector $\hat \uu$ ``aligned'' to $\bmu$, in the sense that the ``cosine-similarly'' $ \hat \uu^\T \bmu/\| \bmu \| \gg 0$ as $n,p \to \infty$.\footnote{Recall that for high dimensional random vector $\vv \in \RR^p$, say with i.i.d.\@ Gaussian entries of zero-mean and $1/p$ variance (so that $\| \vv \| = O(1)$ with high probability), we have, for $\| \bmu \| = O(1)$ that $\vv^\T \bmu \to 0$ as $p \to \infty$; that is, a random vector is \emph{almost orthogonal} to any deterministic or \emph{independent} random vector in high dimensions.}
This type of \emph{phase transition} behavior has been extensively studied in the RMT literature \cite{baik2005phase,benaych2011eigenvalues,lesieur2015phase}. In this article, we show two \emph{different} types of spikes and their corresponding phase transitions: the first well-known spike due to data signal $\bmu$, which follows the classical BPP phase transition \cite{baik2005phase}; and a second less-well-known spike, which arises from the underlying (nonlinear) response model \eqref{eq:def-GLM}, even in the absence of any signal, and which follows a very different phase transition behavior. This is discussed at length with numerical evaluations in Section~\ref{sec:discuss}.

\subsection{Deterministic equivalent}
\label{subsec:pre-DE}

In this work, a unified analysis framework is proposed to assess simultaneously the limiting spectral measure and the behavior of the (possible) isolated eigenvalue-eigenvector pairs of the Hessian-type matrix in \eqref{eq:def-H}. This analysis is based on the so-called deterministic equivalent technique \cite{hachem2007deterministic,couillet2011random} that provides asymptotic characterization of the (expectation of the) resolvent $\Q(z) \equiv (\H - z \I_p)^{-1}$ of the random Hessian matrix $\H \in \RR^{p \times p}$. With additional concentration of measure arguments \cite{ledoux2001concentration,pastur2005simple}, the limiting spectral measure, the isolated eigenvalue location, as well as the eigenvector projection onto a given deterministic vector, while all being random, become asymptotically deterministic and fully accessible via the analysis of $\Q(z)$ in the $n,p \to \infty$ limit. More precisely, we have the following.
\begin{enumerate}
  \item The normalized trace $\frac1p \tr \Q(z)$ characterizes the Stieltjes transform $m_\mu(z) = \int (t-z)^{-1} \mu(dt)$ of the limiting spectral measure $\mu$ of $\H$, from which $\mu$ can be recovered, for $a,b$ continuity points of $\mu$, via $ \mu([a,b]) = \frac1{\pi} \lim_{\varepsilon \downarrow 0} \int_a^b \Im[ m_\mu(x + \imath \varepsilon)]\,dx$.
  \item Solving the determinant equation $\det(\H- \lambda \I_p) = 0$ for $\lambda \in \RR$ outside the limiting spectral support $\supp(\mu)$ (that is, not in the bulk eigenspectrum) reveals the asymptotic location of the isolated eigenvalue $\hat \lambda$ of $\H$.
  \item For $(\hat \lambda ,\hat \uu)$ an eigenvalue-eigenvector pair of interest of $\H$, with Cauchy's integral formula we have $ |\w^\T \hat \uu|^2 = -\frac1{2\pi \imath} \oint_{\Gamma_{\lambda}} \w^\T \Q(z) \w\,dz$, for a deterministic vector $\w \in \RR^p$ and $\Gamma_{\lambda}$ a positively oriented contour surrounding \emph{only} $\hat \lambda$. 
\end{enumerate}
As such, for $\bar \Q(z)$ a deterministic equivalent of $\Q(z)$, that is,
\begin{equation}\label{eq:def-DE}
	\Q(z) \leftrightarrow \bar \Q(z),\quad \text{with~}\frac1p \tr \A (\Q(z) - \bar \Q(z)) \to 0,\text{~and~}\mathbf{a}^\T (\Q(z) - \bar \Q(z)) \mathbf{b} \to 0,
\end{equation}
almost surely as $n,p \to \infty$, for $\A \in \RR^{n \times n}$ and $\mathbf{a}, \mathbf{b} \in \RR^n$ of bounded (Euclidean and spectral) norms, the limiting spectral measure (via the associated Stieltjes transform) and the behavior of isolated eigenpairs are directly accessible via the study of $\bar \Q(z)$.

\section{Main Results}
\label{sec:main_results}

In this section, we present our main results:
on the limiting Hessian eigenspectrum (in Section~\ref{subsec:lsd}); and
on the behavior of the (possible) isolated eigenvalue-eigenvector(s) (in Section~\ref{subsec:spikes}).
These two main results depend on a key technical deterministic equivalent result for the Hessian resolvent (in Section~\ref{subsec:DE}), which is of independent interest.
We position ourselves in the following high dimensional regime.
\begin{Assumption}[High dimensional asymptotics]\label{ass:high-dimen}
As $n,p \to \infty$ with $p/n \to c \in (0,\infty)$, we have $\max \{ \| \w \|, \| \w_* \| \} = O(1)$ and $\x_i \overset{i.i.d.}{\sim} \NN(\bmu, \C)$ with $\max\{\| \bmu \|, \| \C \| \} = O(1)$.
\end{Assumption}

\subsection{Limiting spectral measure}\label{subsec:lsd}

Our first result concerns the limiting eigenvalue distribution of the Hessian matrix, given under the form of its Stieltjes transform. This is a direct consequence of our main technical Theorem~\ref{theo:DE} in Section~\ref{subsec:DE}, and it is proven in Appendix~\ref{subsec:proof-theo-lsd}.
\begin{Theorem}[Limiting spectral measure]\label{theo:lsd}
Let Assumption~\ref{ass:high-dimen} hold.  We have, as $n,p \to \infty$ with $p/n \to c \in (0, \infty)$, the empirical spectral measure $\mu_\H$ of the Hessian matrix $\H$ defined in \eqref{eq:def-H} converges weakly and almost surely to a probability measure $\mu$, defined through its Stieltjes transform $m(z) = \int (t-z)^{-1} \mu(dt)$ as the unique solution to\footnote{Uniqueness is ensured in such a way that $\Im[m(z)] \cdot \Im[z] > 0$ for $\Im[z] \neq 0$ and $z m(z) < 0$ for $\Im[z] = 0$, so that $(z, m(z))$ is a valid  Stieltjes transform couple; see \cite{hachem2007deterministic}.}
\begin{equation}\label{eq:def-m(z)}
	\textstyle m(z) = \frac1p \tr \bar \Q_b(z), \quad \delta(z) = \frac1n \tr \left(\C \bar \Q_b(z) \right),
\end{equation}
for $\bar \Q_b (z) \equiv ( \EE [\frac{g \cdot \C}{ 1 + g \cdot \delta(z) } ] - z \I_p )^{-1}$ and random variable
\begin{equation}\label{eq:def-g}
	g \equiv \partial^2 \ell(y, h)/\partial h^2,~h = \w^\T \x \sim \mathcal N(\w^\T \bmu, \w^\T \C \w) ,
\end{equation}
with $y$ and $\ell$ defined respectively in \eqref{eq:def-GLM} and \eqref{eq:optim}. Moreover, if we denote $\nu$ the law of $g$ and assume the empirical spectral measure of $\C$ converges to $\tilde \nu$ as $p \to \infty$, then \eqref{eq:def-m(z)} can be expressed~as
\begin{equation}
	m(z) =  \int \Big( -z + \tilde t\int \frac{t}{1 + t \delta(z)}\,\nu (dt) \Big)^{-1}\,\tilde \nu(d \tilde t), \quad \delta(z) = \int \frac{c \tilde t}{ -z + \tilde t \int \frac{t}{1 + t \delta(z)}\,\nu(dt) }\,\tilde \nu (d \tilde t). \label{eq:lsd-measure}
\end{equation}
\end{Theorem}

A few remarks on the consequences of Theorem~\ref{theo:lsd} are in order.

\begin{Remark}[Connection to other RMT models]\label{rem:connection-RMT}
In the form \eqref{eq:lsd-measure}, the (Stieltjes transform of the) limiting spectral measure $\mu$ is determined by the dimension ratio $c = \lim p/n$ and the two (probability) measures $\nu$ and $\tilde \nu$. As already briefly discussed in Section~\ref{subsubsec:related}, this formulation is closely connected to the separable covariance model \cite{khorunzhy1996eigenvalue,burda2005spectral,paul2009no,couillet2014analysis,yang2019edge}. Moreover, if $\nu(dt) = \delta_1(t)$ is a Dirac mass at one, this reduces to the classical sample covariance model \cite{silverstein1995empirical}; taking further $\tilde \nu(dt) = \delta_1(t)$ gives the Mar\u{c}enko-Pastur law. See Section~\ref{subsec:num-lsd} for numerical evaluations of these special cases. 
\end{Remark}

In particular, the support of the (limiting) Hessian eigenvalue distribution $\mu$ is directly linked to that of $\nu$ and $\tilde \nu$.
\begin{Remark}[Hessian eigen-support]\label{rem:limit-support}
Under Assumption~\ref{ass:high-dimen}, the (limiting) spectral measure $\tilde \nu$ of the covariance $\C$ has bounded support. However, this may not be the case for $\nu$, the law of the random variable $g$ defined in \eqref{eq:def-g}. Since the Hessian eigenvalue distribution $\mu$ is of compact support \emph{if and only if} both $\nu$ and $\tilde \nu$ have compact support \cite[Porposition~3.4]{couillet2014analysis}, $\mu$ may be of unbounded support, depending on that of $\nu$, and thus on the underlying model and loss function.

\end{Remark}

An example of unbounded $\mu$ is the phase retrieval model with $y = (\w_*^\T \x)^2$ and square loss $\ell(y, h) = (y - h^2 )^2/4$, for which we have $g = 3 (\w^\T \x)^2 - (\w_*^\T \x)^2$ for $\x \sim \NN(\bmu, \C)$. As a consequence, with say $\w_* = \w$, $g$ follows a chi-square distribution with one degree of freedom, and has thus \emph{unbounded support}. This corresponds to Figure~\ref{fig:intro-c}, where the Hessian spectrum has a ``heavier'' tail compared to Figure~\ref{fig:intro-a}.  In this case, the empirically observed ``isolated'' eigenvalue is essentially due to a finite-dimensional effect and will be ``buried'' in the noisy main bulk as $n,p$ grow large. Therefore, aiming for an (almost surely) isolated eigenvalue-eigenvector (e.g., to recover the model parameter $\w_*$ using the top Hessian eigenvector) for $n,p$ large, some preprocessing function $f$ must be applied. This has been discussed in previous work \cite{lu2019phase,mondelli2018fundamental} and corresponds to the so-called trimming strategy widely used in phase retrieval \cite{chen2015solving} (particularly in the ``linear measurements'' $n \sim p$ regime), with for instance the truncation function $f(t) = 1_{|t|\le \epsilon}$ for some $\epsilon > 0$. 

\medskip

Another example of unbounded $\mu$ appears when the exponential loss function \cite{freund1999short} is used. Precisely, consider the logistic model \eqref{eq:def-logistic} with the choice of loss function $\ell(y,h) = \exp(-yh)$.  In this case, we have that $g = \exp(-yh)$ for $h \sim \mathcal N(\w^\T \bmu, \w^\T \C \w)$, which follows a log-normal distribution and has \emph{unbounded support}. As such, the limiting Hessian eigenvalue distribution $\mu$ has also unbounded support.
On the other hand, with the logistic loss $\ell(y,h) = \ln(1+ e^{-yh})$, one has $g \le 1/4$ and $\mu$ is guaranteed to have bounded support. See Figure~\ref{fig:log-exp-lsd}, where the empirical Hessian eigenvalues and the limiting distributions are compared for logistic and exponential losses, with a more ``heavy-tailed'' behavior observed for the exponential loss.

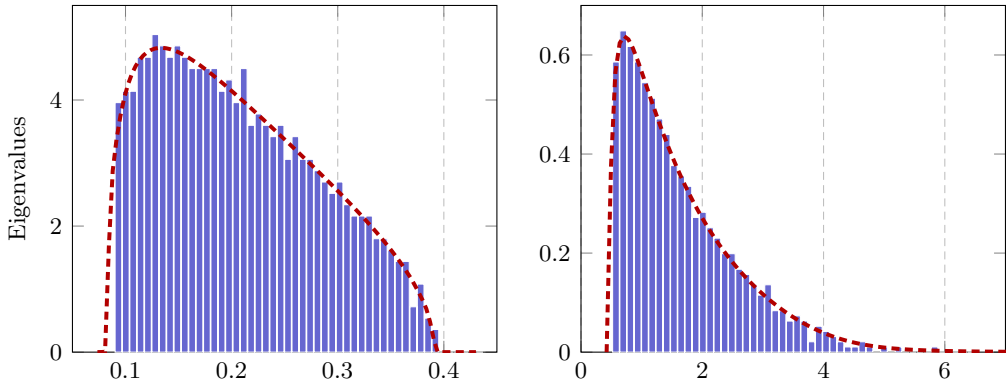
\begin{figure}[htb]
\vskip 0.1in
\begin{center}
\begin{tabular}{cc}
\centering
\begin{tikzpicture}[font=\footnotesize]
    \renewcommand{\axisdefaulttryminticks}{4} 
    \pgfplotsset{every major grid/.style={densely dashed}}       
    \pgfplotsset{every axis legend/.style={cells={anchor=west},fill=none, at={(0.98,0.98)}, anchor=north east, font=\footnotesize }}
    \begin{axis}[
      width=.45\columnwidth,
      xmin=0.05,
      xmax=.45,
      ymin=0,
      ymax=5.5,
      bar width=2.5pt,
      grid=major,
      ymajorgrids=false,
      scaled ticks=true,
      xlabel={},
      ylabel={Eigenvalues}
      ]
      \addplot+[ybar,mark=none,draw=white,fill=BLUE!60!white,area legend] coordinates{
      (0.093341,3.958695)(0.100287,4.138636)(0.107234,4.138636)(0.114181,4.678458)(0.121128,4.678458)(0.128074,5.038339)(0.135021,4.858399)(0.141968,4.678458)(0.148915,4.858399)(0.155861,4.678458)(0.162808,4.498517)(0.169755,4.498517)(0.176702,4.498517)(0.183648,4.498517)(0.190595,4.138636)(0.197542,4.318577)(0.204488,3.958695)(0.211435,4.498517)(0.218382,3.598814)(0.225329,3.778755)(0.232275,3.598814)(0.239222,3.418873)(0.246169,3.598814)(0.253116,3.058992)(0.260062,3.418873)(0.267009,3.058992)(0.273956,3.058992)(0.280903,2.879051)(0.287849,2.699110)(0.294796,2.519170)(0.301743,2.699110)(0.308689,2.339229)(0.315636,2.159288)(0.322583,2.159288)(0.329530,2.159288)(0.336476,1.799407)(0.343423,1.799407)(0.350370,1.619466)(0.357317,1.439526)(0.364263,1.439526)(0.371210,0.719763)(0.378157,1.079644)(0.385104,0.539822)(0.392050,0.359881)(0.398997,0.000000)(0.405944,0.000000)(0.412890,0.000000)(0.419837,0.000000)(0.426784,0.000000)(0.433731,0.000000)
      };
      \addplot[densely dashed,RED,line width=1.5pt] coordinates{
      (0.073528,0.001131)(0.077131,0.001779)(0.080735,0.005458)(0.084338,1.683898)(0.087941,2.876420)(0.091544,3.377044)(0.095148,3.741451)(0.098751,4.028896)(0.102354,4.247640)(0.105958,4.411551)(0.109561,4.538441)(0.113164,4.634179)(0.116768,4.705863)(0.120371,4.757662)(0.123974,4.793249)(0.127578,4.815306)(0.131181,4.826102)(0.134784,4.827379)(0.138388,4.820602)(0.141991,4.806975)(0.145594,4.787490)(0.149198,4.762982)(0.152801,4.734154)(0.156404,4.701599)(0.160008,4.665822)(0.163611,4.627259)(0.167214,4.586274)(0.170818,4.543188)(0.174421,4.498276)(0.178024,4.451775)(0.181628,4.403888)(0.185231,4.354793)(0.188834,4.304643)(0.192438,4.253579)(0.196041,4.201709)(0.199644,4.149137)(0.203248,4.095956)(0.206851,4.042233)(0.210454,3.988040)(0.214058,3.933431)(0.217661,3.878456)(0.221264,3.823156)(0.224868,3.767561)(0.228471,3.711707)(0.232074,3.655614)(0.235678,3.599303)(0.239281,3.542786)(0.242884,3.486075)(0.246488,3.429176)(0.250091,3.372094)(0.253694,3.314826)(0.257298,3.257374)(0.260901,3.199728)(0.264504,3.141881)(0.268108,3.083821)(0.271711,3.025534)(0.275314,2.967001)(0.278918,2.908205)(0.282521,2.849118)(0.286124,2.789716)(0.289728,2.729968)(0.293331,2.669840)(0.296934,2.609297)(0.300537,2.548295)(0.304141,2.486782)(0.307744,2.424716)(0.311347,2.362028)(0.314951,2.298646)(0.318554,2.234523)(0.322157,2.169547)(0.325761,2.103600)(0.329364,2.036632)(0.332967,1.968499)(0.336571,1.899020)(0.340174,1.828030)(0.343777,1.755326)(0.347381,1.680810)(0.350984,1.604039)(0.354587,1.524858)(0.358191,1.442358)(0.361794,1.357005)(0.365397,1.263878)(0.369001,1.171825)(0.372604,1.062719)(0.376207,0.950988)(0.379811,0.829762)(0.383414,0.691383)(0.387017,0.479956)(0.390621,0.192251)(0.394224,0.000643)(0.397827,0.000406)(0.401431,0.000312)(0.405034,0.000257)(0.408637,0.000221)(0.412241,0.000195)(0.415844,0.000174)(0.419447,0.000158)(0.423051,0.000144)(0.426654,0.000133)(0.430257,0.000124)
      };
    \end{axis}
  \end{tikzpicture}
  &
  \begin{tikzpicture}[font=\footnotesize]
    \renewcommand{\axisdefaulttryminticks}{4} 
    \pgfplotsset{every major grid/.append style={densely dashed}}       
    \pgfplotsset{every axis legend/.style={cells={anchor=west},fill=none, at={(0.98,0.98)}, anchor=north east, font=\footnotesize}}
    \begin{axis}[
      width=.45\columnwidth,
      xmin=0,
      xmax=7,
      ymin=0,
      ymax=.7,
      bar width=2.5pt,
      grid=major,
      ymajorgrids=false,
      scaled ticks=true,
      xlabel={},
      ylabel={}
      ]
      \addplot+[ybar,mark=none,draw=white,fill=BLUE!60!white,area legend] coordinates{
      (0.576366,0.585924)(0.695835,0.648701)(0.815305,0.617312)(0.934774,0.585924)(1.054244,0.544072)(1.173713,0.512683)(1.293183,0.470832)(1.412652,0.439443)(1.532122,0.376665)(1.651591,0.355739)(1.771060,0.334814)(1.890530,0.272036)(2.009999,0.282499)(2.129469,0.251110)(2.248938,0.230184)(2.368408,0.198796)(2.487877,0.198796)(2.607347,0.167407)(2.726816,0.156944)(2.846286,0.136018)(2.965755,0.115092)(3.085225,0.136018)(3.204694,0.083703)(3.324164,0.083703)(3.443633,0.062778)(3.563103,0.073240)(3.682572,0.062778)(3.802042,0.020926)(3.921511,0.052315)(4.040981,0.041852)(4.160450,0.031389)(4.279920,0.020926)(4.399389,0.010463)(4.518859,0.010463)(4.638328,0.020926)(4.757798,0.010463)(4.877267,0.000000)(4.996737,0.010463)(5.116206,0.000000)(5.235676,0.010463)(5.355145,0.000000)(5.474615,0.000000)(5.594084,0.000000)(5.713554,0.000000)(5.833023,0.010463)(5.952492,0.000000)(6.071962,0.000000)(6.191431,0.000000)(6.310901,0.000000)(6.430370,0.000000)
      };
      \addplot[densely dashed,RED,line width=1.5pt] coordinates{
      (0.422698,0.000030)(0.494478,0.374551)(0.566258,0.564286)(0.638038,0.624011)(0.709819,0.638258)(0.781599,0.631572)(0.853379,0.614430)(0.925159,0.592009)(0.996939,0.567076)(1.068719,0.541189)(1.140500,0.515241)(1.212280,0.489753)(1.284060,0.465012)(1.355840,0.441186)(1.427620,0.418346)(1.499400,0.396526)(1.571180,0.375716)(1.642961,0.355896)(1.714741,0.337032)(1.786521,0.319082)(1.858301,0.302003)(1.930081,0.285750)(2.001861,0.270282)(2.073642,0.255554)(2.145422,0.241526)(2.217202,0.228162)(2.288982,0.215423)(2.360762,0.203278)(2.432542,0.191696)(2.504322,0.180644)(2.576103,0.170098)(2.647883,0.160033)(2.719663,0.150425)(2.791443,0.141253)(2.863223,0.132499)(2.935003,0.124142)(3.006783,0.116168)(3.078564,0.108561)(3.150344,0.101306)(3.222124,0.094393)(3.293904,0.087811)(3.365684,0.081548)(3.437464,0.075595)(3.509245,0.069944)(3.581025,0.064590)(3.652805,0.059522)(3.724585,0.054737)(3.796365,0.050228)(3.868145,0.045990)(3.939925,0.042018)(4.011706,0.038306)(4.083486,0.034848)(4.155266,0.031638)(4.227046,0.028668)(4.298826,0.025932)(4.370606,0.023420)(4.442387,0.021124)(4.514167,0.019031)(4.585947,0.017132)(4.657727,0.015413)(4.729507,0.013862)(4.801287,0.012466)(4.873067,0.011213)(4.944848,0.010090)(5.016628,0.009085)(5.088408,0.008187)(5.160188,0.007384)(5.231968,0.006668)(5.303748,0.006028)(5.375529,0.005457)(5.447309,0.004946)(5.519089,0.004489)(5.590869,0.004080)(5.662649,0.003714)(5.734429,0.003386)(5.806209,0.003090)(5.877990,0.002825)(5.949770,0.002586)(6.021550,0.002371)(6.093330,0.002176)(6.165110,0.002000)(6.236890,0.001841)(6.308671,0.001696)(6.380451,0.001565)(6.452231,0.001446)(6.524011,0.001337)(6.595791,0.001238)(6.667571,0.001147)(6.739351,0.001065)(6.811132,0.000989)(6.882912,0.000920)(6.954692,0.000856)(7.026472,0.000797)(7.098252,0.000744)(7.170032,0.000694)(7.241812,0.000648)(7.313593,0.000606)(7.385373,0.000567)(7.457153,0.000531)(7.528933,0.000498)
      };
    \end{axis}
  \end{tikzpicture}
\end{tabular}
 \caption{ 
 Impact of loss function: bounded \textbf{(left)} versus unbounded \textbf{(right)} Hessian eigenvalue distribution, with $p=800$, $n = 6\,000$, logistic model in \eqref{eq:def-logistic} with $\bmu = \mathbf{0}$, $\C = \I_p$, $\w_* = \zo$ and $\w = [-\one_{p/2},~\one_{p/2}]/\sqrt p$. Emphasis on the \emph{bounded} Hessian eigen-support with the logistic loss \textbf{(left)} versus the \emph{unbounded} support with the exponential loss \textbf{(right)}.
 } 
\label{fig:log-exp-lsd}
\end{center}
\vskip -0.1in
\end{figure}

Clearly, depending on the measures $\nu$ (of the random variable $g$, which depends on the feature statistics $\bmu, 
\C$, the loss and the underlying model) and $\tilde \nu$ (of the eigenvalue distribution of feature covariance $\C$), the Hessian spectrum can have very different forms. See Figure~\ref{fig:logistic-lsd}, where we compare the empirical Hessian eigenvalues with their limiting behaviors, given in Theorem~\ref{theo:lsd} for different feature covariance structures.\footnote{Covariance describes the joint variability or the ``correlation'' between entries of the feature vector, and it is of great significance in the analysis of image (with local structure) and time series data.} 
In particular, we may observe a single main bulk with more ``compact'' Hessian eigenvalues as in Figure~\ref{fig:logistic-lsd}-(left) or multiple bulks (two in the case of Figure~\ref{fig:logistic-lsd}-right) with Hessian eigenvalues more ``spread-out'', depending on the feature covariance structure and the measure $\tilde \nu$. In the form of \eqref{eq:lsd-measure}, the condition for the existence of multi-bulk eigenspectrum has been throughly discussed in \cite[Section 3.2--3.4]{couillet2014analysis} and can be numerically evaluated with ease.

\smallskip

As a side remark, the ``multi-bulk'' behavior similar to Figure~\ref{fig:logistic-lsd}-(right) has been empirically observed in Hessians of NNs in \cite{li2020hessian,papyan2020traces}, suggesting the (quantitative) effectiveness of the proposed theory beyond simple G-GLM.

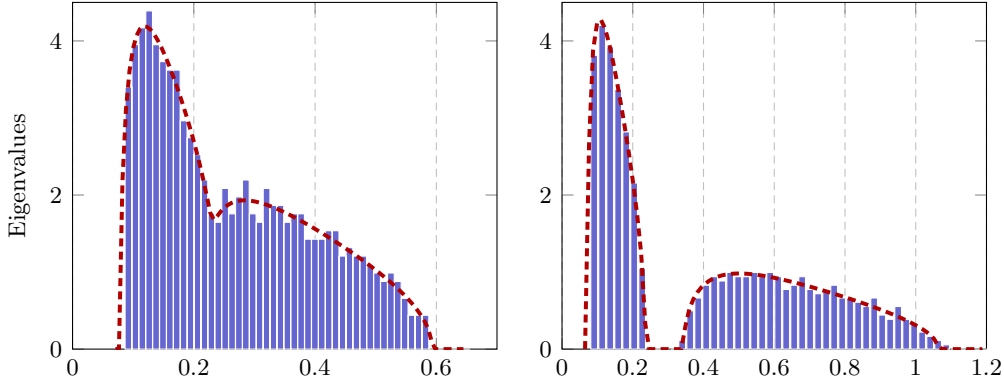
\begin{figure}[htb]
\vskip 0.1in
\begin{center}
\begin{tabular}{cc}
\centering
\begin{tikzpicture}[font=\footnotesize]
    \renewcommand{\axisdefaulttryminticks}{4} 
    \pgfplotsset{every major grid/.style={densely dashed}}       
    \pgfplotsset{every axis legend/.style={cells={anchor=west},fill=none, at={(0.98,0.98)}, anchor=north east, font=\footnotesize }}
    \begin{axis}[
      width=.45\columnwidth,
      xmin=0,
      xmax=.7,
      ymin=0,
      ymax=4.5,
      bar width=2.5pt,
      grid=major,
      ymajorgrids=false,
      scaled ticks=true,
      xlabel={},
      ylabel={Eigenvalues}
      ]
      \addplot+[ybar,mark=none,draw=white,fill=BLUE!60!white,area legend] coordinates{
      (0.092168,3.402000)(0.103558,3.950710)(0.114949,4.170194)(0.126339,4.389678)(0.137729,3.950710)(0.149120,3.731226)(0.160510,3.621484)(0.171900,3.621484)(0.183291,2.963032)(0.194681,2.743549)(0.206072,2.524065)(0.217462,2.194839)(0.228852,1.755871)(0.240243,1.646129)(0.251633,2.085097)(0.263023,1.755871)(0.274414,1.975355)(0.285804,2.194839)(0.297194,1.755871)(0.308585,1.646129)(0.319975,2.085097)(0.331366,1.865613)(0.342756,1.865613)(0.354146,1.646129)(0.365537,1.755871)(0.376927,1.755871)(0.388317,1.426645)(0.399708,1.426645)(0.411098,1.426645)(0.422488,1.536387)(0.433879,1.536387)(0.445269,1.207161)(0.456659,1.316903)(0.468050,1.207161)(0.479440,1.207161)(0.490831,1.097419)(0.502221,0.987677)(0.513611,0.877936)(0.525002,0.987677)(0.536392,0.877936)(0.547782,0.658452)(0.559173,0.438968)(0.570563,0.438968)(0.581953,0.438968)(0.593344,0.000000)(0.604734,0.000000)(0.616124,0.000000)(0.627515,0.000000)(0.638905,0.000000)(0.650296,0.000000)
      };
      \addplot[densely dashed,RED,line width=1.5pt] coordinates{
      (0.070750,0.001075)(0.076547,0.003150)(0.082343,2.187235)(0.088140,3.116822)(0.093936,3.611148)(0.099733,3.900471)(0.105529,4.068199)(0.111326,4.156300)(0.117122,4.189209)(0.122919,4.182293)(0.128715,4.145926)(0.134512,4.087194)(0.140308,4.011153)(0.146105,3.921426)(0.151901,3.820610)(0.157697,3.710627)(0.163494,3.592773)(0.169290,3.467939)(0.175087,3.336608)(0.180883,3.198948)(0.186680,3.054765)(0.192476,2.903500)(0.198273,2.744148)(0.204069,2.575093)(0.209866,2.393929)(0.215662,2.197246)(0.221459,1.982226)(0.227255,1.766291)(0.233051,1.679566)(0.238848,1.728831)(0.244644,1.788385)(0.250441,1.835873)(0.256237,1.871243)(0.262034,1.896616)(0.267830,1.913923)(0.273627,1.924664)(0.279423,1.930008)(0.285220,1.930863)(0.291016,1.927940)(0.296813,1.921812)(0.302609,1.912943)(0.308405,1.901714)(0.314202,1.888419)(0.319998,1.873336)(0.325795,1.856686)(0.331591,1.838648)(0.337388,1.819371)(0.343184,1.799011)(0.348981,1.777667)(0.354777,1.755440)(0.360574,1.732419)(0.366370,1.708670)(0.372167,1.684265)(0.377963,1.659266)(0.383760,1.633690)(0.389556,1.607598)(0.395352,1.581024)(0.401149,1.553980)(0.406945,1.526499)(0.412742,1.498601)(0.418538,1.470281)(0.424335,1.441571)(0.430131,1.412457)(0.435928,1.382942)(0.441724,1.353036)(0.447521,1.322705)(0.453317,1.291974)(0.459114,1.260790)(0.464910,1.229170)(0.470706,1.197065)(0.476503,1.164451)(0.482299,1.131283)(0.488096,1.097530)(0.493892,1.063155)(0.499689,1.028072)(0.505485,0.992206)(0.511282,0.955512)(0.517078,0.917851)(0.522875,0.879097)(0.528671,0.839128)(0.534468,0.797755)(0.540264,0.754722)(0.546060,0.709717)(0.551857,0.662383)(0.557653,0.612103)(0.563450,0.558134)(0.569246,0.499271)(0.575043,0.433454)(0.580839,0.356895)(0.586636,0.260243)(0.592432,0.094528)(0.598229,0.000217)(0.604025,0.000133)(0.609822,0.000102)(0.615618,0.000085)(0.621415,0.000073)(0.627211,0.000064)(0.633007,0.000058)(0.638804,0.000053)(0.644600,0.000048)
      };
    \end{axis}
  \end{tikzpicture}
  &
  \begin{tikzpicture}[font=\footnotesize]
    \renewcommand{\axisdefaulttryminticks}{4} 
    \pgfplotsset{every major grid/.append style={densely dashed}}       
    \pgfplotsset{every axis legend/.style={cells={anchor=west},fill=none, at={(0.98,0.98)}, anchor=north east, font=\footnotesize}}
    \begin{axis}[
      width=.45\columnwidth,
      xmin=0,
      ymin=0,
      xmax=1.2,
      ymax=4.5,
      bar width=2.5pt,
      grid=major,
      ymajorgrids=false,
      scaled ticks=true,
      xlabel={},
      ylabel={}
      ]
      \addplot+[ybar,mark=none,draw=white,fill=BLUE!60!white,area legend] coordinates{
      (0.090868,3.813746)(0.113483,4.200647)(0.136099,3.924289)(0.158714,3.371572)(0.181330,2.818855)(0.203946,2.155595)(0.226561,1.050162)(0.249177,0.000000)(0.271792,0.000000)(0.294408,0.000000)(0.317023,0.000000)(0.339639,0.110543)(0.362255,0.497445)(0.384870,0.663260)(0.407486,0.829075)(0.430101,0.939618)(0.452717,0.884347)(0.475332,0.994890)(0.497948,0.939618)(0.520563,0.939618)(0.543179,0.994890)(0.565795,0.939618)(0.588410,0.994890)(0.611026,0.939618)(0.633641,0.773803)(0.656257,0.829075)(0.678872,0.939618)(0.701488,0.773803)(0.724104,0.718532)(0.746719,0.773803)(0.769335,0.829075)(0.791950,0.663260)(0.814566,0.663260)(0.837181,0.607988)(0.859797,0.552717)(0.882412,0.663260)(0.905028,0.442173)(0.927644,0.386902)(0.950259,0.552717)(0.972875,0.386902)(0.995490,0.331630)(1.018106,0.221087)(1.040721,0.165815)(1.063337,0.110543)(1.085953,0.055272)(1.108568,0.000000)(1.131184,0.000000)(1.153799,0.000000)(1.176415,0.000000)(1.199030,0.000000)
      };
      \addplot[densely dashed,RED,line width=1.5pt] coordinates{
      (0.065095,0.001370)(0.076434,2.674794)(0.087774,3.879693)(0.099114,4.230877)(0.110453,4.270192)(0.121793,4.159312)(0.133133,3.967206)(0.144472,3.727421)(0.155812,3.457002)(0.167152,3.164106)(0.178491,2.851284)(0.189831,2.516498)(0.201171,2.152223)(0.212510,1.740820)(0.223850,1.235106)(0.235190,0.334440)(0.246529,0.000408)(0.257869,0.000247)(0.269209,0.000184)(0.280548,0.000151)(0.291888,0.000133)(0.303228,0.000126)(0.314567,0.000129)(0.325907,0.000149)(0.337247,0.000238)(0.348586,0.293843)(0.359926,0.518582)(0.371266,0.647223)(0.382606,0.736083)(0.393945,0.801349)(0.405285,0.850623)(0.416625,0.888254)(0.427964,0.917019)(0.439304,0.938823)(0.450644,0.955041)(0.461983,0.966696)(0.473323,0.974574)(0.484663,0.979292)(0.496002,0.981341)(0.507342,0.981119)(0.518682,0.978952)(0.530021,0.975108)(0.541361,0.969812)(0.552701,0.963252)(0.564040,0.955587)(0.575380,0.946951)(0.586720,0.937460)(0.598059,0.927212)(0.609399,0.916291)(0.620739,0.904770)(0.632078,0.892711)(0.643418,0.880167)(0.654758,0.867185)(0.666097,0.853803)(0.677437,0.840055)(0.688777,0.825971)(0.700116,0.811574)(0.711456,0.796884)(0.722796,0.781918)(0.734135,0.766688)(0.745475,0.751204)(0.756815,0.735473)(0.768155,0.719499)(0.779494,0.703283)(0.790834,0.686825)(0.802174,0.670120)(0.813513,0.653160)(0.824853,0.635937)(0.836193,0.618436)(0.847532,0.600643)(0.858872,0.582535)(0.870212,0.564089)(0.881551,0.545274)(0.892891,0.526054)(0.904231,0.506387)(0.915570,0.486218)(0.926910,0.465484)(0.938250,0.444106)(0.949589,0.421985)(0.960929,0.398999)(0.972269,0.374982)(0.983608,0.349724)(0.994948,0.322929)(1.006288,0.294171)(1.017627,0.262803)(1.028967,0.227744)(1.040307,0.186919)(1.051646,0.135187)(1.062986,0.043300)(1.074326,0.000054)(1.085665,0.000035)(1.097005,0.000027)(1.108345,0.000022)(1.119685,0.000019)(1.131024,0.000017)(1.142364,0.000015)(1.153704,0.000014)(1.165043,0.000013)(1.176383,0.000012)(1.187723,0.000011)
      };
    \end{axis}
  \end{tikzpicture}
\end{tabular}
 \caption{ Impact of feature covariance: Hessian spectrum of single- \textbf{(left)} versus multi-bulk \textbf{(right)} , with $p=800$, $n = 6\,000$, logistic model in \eqref{eq:def-logistic} with $\w^* = \zo_p$, $\w = \bmu \sim \NN(\zo, \I_p/p)$; for $\C = \diag[\one_{p/2};~2 \cdot \one_{p/2}]$ \textbf{(left)} and $\C = \diag[\one_{p/2};~4\cdot \one_{p/2}]$ \textbf{(right)}. }  
\label{fig:logistic-lsd}
\end{center}
\vskip -0.1in
\end{figure}

\subsection{Isolated eigenvalues and eigenvectors}\label{subsec:spikes}

As discussed in Remark~\ref{rem:limit-support}, the Hessian has bounded (limiting) eigen-support if and only if $\nu$, the law of $g$, has bounded support. Under this assumption (or, after the application of some function $f$ so that $f(g)$ is bounded), we can then talk about the (possible) isolated Hessian eigenvalues, as in the following result, the proof of which is in Appendix~\ref{subsec:proof-theo-spike}.

\begin{Theorem}[Isolated eigenvalues]\label{theo:spike}
In the setting of Theorem~\ref{theo:lsd}, assume that the probability measure $\nu$ of the random variable $g$ defined in \eqref{eq:def-g} is of bounded supported, and define
\begin{equation}\label{eq:def-G(z)}
    \G(z) = \I_3 + \bLambda(z) \V^\T \bar \Q_b(z) \V \in \RR^{3 \times 3},
\end{equation}
with $\bar \Q_b(z),\delta(z)$ defined in \eqref{eq:def-m(z)}, $\V \equiv [\bmu,~\C \w_*,~\C \w] \in \RR^{p \times 3}$, $\U \equiv \C^{\frac12} [\w_*,~\w] \in \RR^{p \times 2}$ and 
\begin{equation}\label{eq:def-V-Lambda}
  \bLambda(z) \equiv \EE \left[ \frac{g}{1+g \cdot \delta(z)} \begin{bmatrix} 1 & (\U^+ \z)^\T \\ \U^+ \z & \U^+ \z (\U^+ \z)^\T - (\U^\T \U)^+ \end{bmatrix} \right],
\end{equation}
for $\z = \C^{-\frac12} (\x - \bmu) \sim \mathcal N(\mathbf{0},\I_p)$, where we denote $\U^+$ the Moore-Penrose pseudoinverse of $\U$ with $\U^+ = (\U^\T \U)^{-1} \U^\T$ for invertible $\U^\T \U$. Then, for $\lambda$ such that $\G(\lambda)$ has a zero eigenvalue (of multiplicity one), there exists an eigenvalue $\hat \lambda$ of $\H$ such that $ \hat \lambda - \lambda~\asto~0$.
\end{Theorem}

\begin{Remark}[Connection to isolated eigenvalue phase transition]\label{rem:connection-isolated-eigenvalue-phase-transition}
Theorem~\ref{theo:spike} provides an asymptotic characterization of the possible isolated Hessian eigenvalues by computing the determinant of the much smaller (three-by-three) deterministic matrix $\G$ closely related to the key quantity $\delta(z)$.
Note that Theorem~\ref{theo:spike} does not provide, at least explicitly, the \emph{phase transition} condition under which these eigenvalues/spikes become ``isolated'' from the main bulk. As we shall see in more detail in Section~\ref{sec:discuss}, two types of quantitatively different phase transitions can be characterized: one due to the data ``signal'' $\bmu$; and one due to the structure of the underlying model (even in the absence of any signal in the data).
\end{Remark}



\medskip

We can also analyze the associated eigenvectors. First note that, for $n \to \infty$ with $p$ fixed, we have, by the strong law of large numbers, that $\H~\asto~\EE[\H]$, with
\begin{align*}
  \EE[\H] &= \EE[ \ell''(y, \w^\T \x) \x \x^\T ] =  \EE[g] \cdot \C + \V \begin{bmatrix} 1 & \EE[g \cdot \U^+ \z]^\T \\ \EE[g \cdot \U^+ \z] & \U^+ \EE[g \cdot (\z \z^\T - \I_p)] (\U^+)^\T \end{bmatrix} \V^\T
\end{align*}
with $\z = \C^{-\frac12} (\x - \bmu) \sim \mathcal N(\mathbf{0},\I_p)$ and $g$ defined in \eqref{eq:def-g}. 
As a consequence, it is expected that in the large $n,p \to \infty$ limit, the top eigenvectors of $\H$ should be related to the columns of $\V$. 
This is the case in Figure~\ref{fig:intro-d}, where the top eigenvector is observed to be a ``noisy'' version of the model parameter $\w_*$. More precisely, for $(\hat \lambda, \hat \uu)$ an isolated eigenpair of $\H$, the \emph{random} projection $\V^\T \hat \uu \hat \uu^\T \V \in \RR^{3 \times 3}$ can be shown to be asymptotically close to a \emph{deterministic} matrix as $n,p \to \infty$. 
This measures the ``cosine-similarly'' between the Hessian isolated eigenvector $\hat \uu$ with any column of $\V$ and consequently the performance of using $\hat \uu$ as an estimate of, for instance the model parameter $\w_*$ for $\C = \I_p$. 
This result is given in the following theorem, which is proven in Appendix~\ref{subsec:proof-theo-eigenvector}.

\begin{Theorem}[Isolated eigenvectors]\label{theo:eigenvector}
In the setting of Theorem~\ref{theo:spike}, for an isolated eigenvalue-eigenvector pair $(\hat \lambda, \hat \uu)$ of $\H$ and $\lambda$ the asymptotic position given in Theorem~\ref{theo:spike}, it holds that
\[
  \V^\T \hat \uu \hat \uu^\T \V = -\V^\T \bar \Q_b(\lambda) \V \cdot \Xi(\lambda) + o(1),
\]
with $\Xi(\lambda) = \frac{\vv_{r,\G} \vv_{l,\G}^\T }{ \vv_{l,\G}^\T \G'(\lambda) \vv_{r,\G} } $, for $\bar \Q_b(z)$ and $\G(z)$ defined in \eqref{eq:def-m(z)}~and~\eqref{eq:def-G(z)}, $\vv_{l,\G}, \vv_{r,\G} \in \RR^3$ the left and right eigenvectors of $\G(\lambda)$ associated with eigenvalue zero, and $\G'(z)$ the derivative of $\G(z)$ with respect to $z$ evaluated at $z = \lambda$.
\end{Theorem}

\subsection{Technical tool: deterministic equivalent}\label{subsec:DE}

As already discussed in Section~\ref{subsec:pre-DE}, our main technical tool to derive Theorems~\ref{theo:lsd},~\ref{theo:spike}~and~\ref{theo:eigenvector} is the so-called deterministic equivalent approach \cite{hachem2007deterministic,couillet2011random} for the Hessian resolvent $\Q(z) = (\H - z \I_p)^{-1}$, from which follows our main results on the Hessian limiting eigenvalue distribution and the behavior of the possible isolated eigenpairs. 
This result is given in the following theorem, with the proof deferred to Appendix~\ref{sec:SM-proof-theo-DE}.


\begin{Theorem}[Deterministic equivalent]\label{theo:DE}
Let $\Q(z) \equiv (\H - z \I_p)^{-1}$ be the resolvent of the Hessian $\H$ defined in \eqref{eq:def-H}. Then, under Assumption~\ref{ass:high-dimen}, we have, as $n,p \to \infty$ with $p/n \to c \in (0, \infty)$, that $\Q(z) \leftrightarrow \bar \Q(z)$, with
\[
   \bar \Q^{-1}(z) = \textstyle \EE \frac{g}{1+g \cdot \delta(z)} ( \C^{\frac12} (\I_p - \P_\U) \C^{\frac12} +  \balpha \balpha^\T) - z \I_p,
\]
for random vector $\balpha \equiv \bmu + \C^{\frac12} \P_\U \z \in \RR^p$ and $g = \ell''(y, \w^\T \bmu + \w^\T \C^{\frac12} \z)$ for $\z \sim \NN(\zo, \I_p)$ as in \eqref{eq:def-g}, $y$ defined in \eqref{eq:def-GLM}, $\P_\U$ the projection onto the subspace spanned by the columns of $\U$ so that $\P_\U = \U (\U^\T \U)^{-1} \U^\T$ for $\U^\T \U$ invertible, and $\delta(z)$ the unique solution to \eqref{eq:def-m(z)}.
\end{Theorem}

\section{Empirical Evaluation and Discussion}
\label{sec:discuss}

In this section, we provide further discussion on the consequences of Theorems~\ref{theo:lsd},~\ref{theo:spike}~and~\ref{theo:eigenvector}, together with  numerical evaluations. 
The implications of Theorem~\ref{theo:lsd} on the (limiting) Hessian eigenvalue distribution is discussed in Section~\ref{subsec:num-lsd}. 
Then, we discuss the consequences of Theorems~\ref{theo:spike}~and~\ref{theo:eigenvector} on possible isolated eigenpairs, and we show that two fundamentally different phase transition behaviors occur.
%
Recall that classical spiked models, extensively studied in RMT literature \cite{baik2005phase,benaych2011eigenvalues,lesieur2015phase}, exhibit the following: (i) the isolated spike appears due to the presence of some statistical ``signal'' in the data; and (ii) a ``monotonic'' phase transition behavior, as a function of the signal strength, can be characterized.
Here, we show that, in addition to this classical ``signal'' spike, another type of spike may arise due to the underlying G-GLM model itself (i.e., $\w_*$ and $\w$), and it exhibits a very different transition behavior.
In particular,
in Section~\ref{subsec:spike-mu}, we discuss eigenvalue spikes due to data signal; and
in Section~\ref{subsec:spike-w}, we discuss eigenvalue spikes due to model structure, even when there is no signal in the data.

\subsection{Hessian eigenvalue distribution}
\label{subsec:num-lsd}


Here, for a better interpretation of Theorem~\ref{theo:lsd} on the Hessian spectrum, we consider  the special case where $\C = \I_p$, in which case \eqref{eq:def-m(z)} can be further simplified as $z \delta(z) = 1-c - \EE [1/(1 + g \cdot \delta(z))]$, with $m(z) = \delta(z)/c$ and $c = \lim p/n$.

Let us start with the simplest setting where the random variable $g$ in \eqref{eq:def-g} is constant, say $g = 1$.
This happens, for instance, in the case where the square loss $\ell(y,h) = (y - h)^2/2$ is employed for the logistic model \eqref{eq:def-logistic}. In this case, the Hessian is \emph{independent} of the parameters $\w,\w_*$, and we have the Stieltjes transform $m(z)$ as the solution to
\begin{equation}\label{eq:MP-equation}
  z c m^2(z) - (1 - c - z) m(z) + 1 = 0,
\end{equation}
that corresponds to the popular Mar\u{c}enko-Pastur law. 

As long as $g$ is \emph{not} constant, the limiting Hessian spectrum is, a priori, different from the Mar\u{c}enko-Pastur law, even in the simple $\C = \I_p$ setting, since the associated Stieltjes transform $m(z)$ is different from the solution of \eqref{eq:MP-equation}. However, we see in Figure~\ref{fig:logistic-fit-MP}-(top left) that, for the logistic model \eqref{eq:def-logistic} with logistic loss, the Hessian spectrum is close, at least visually, to a rescaled Mar\u{c}enko-Pastur law. This is due to the fact that, the distribution of $g$ is more ``concentrated'' (around some constant, see Figure~\ref{fig:logistic-fit-MP}-(top right) versus (bottom right) for a comparison between different cases) in this specific model. This is in sharp contrast to, for example, Figure~\ref{fig:logistic-fit-MP}-(bottom) where with the exponential loss, the law of $g$ has a much larger ``spread'', and the Hessian eigenspectrum is therefore far away from a Mar\u{c}enko-Pastur-shape.

\begin{figure}[!htb]
\vskip 0.1in
\begin{center}
\begin{tabular}{cc}
\centering
  \begin{tikzpicture}[font=\footnotesize]
    \renewcommand{\axisdefaulttryminticks}{4} 
    \pgfplotsset{every major grid/.append style={densely dashed}}       
    \pgfplotsset{every axis legend/.style={cells={anchor=west},fill=none, at={(1.02,1.02)}, anchor=north east, font=\footnotesize }}
    \begin{axis}[
      width=.45\columnwidth,
      xmin=0,
      ymin=0,
      xmax=0.45,
      ymax=6,
      ymajorticks=false,
      bar width=2.5pt,
      grid=major,
      ymajorgrids=false,
      scaled ticks=true,
      xlabel={},
      ylabel={Eigenvalues}
      ]
      \addlegendimage{smooth,line width=1pt,color=GREEN};
      \addlegendentry{{ Rescaled MP }};
      \addplot+[ybar,mark=none,draw=white,fill=BLUE!60!white,area legend] coordinates{
      (0.077668,0.000000)(0.084827,1.920632)(0.091987,3.492059)(0.099146,4.190471)(0.106305,4.365074)(0.113464,4.714279)(0.120623,4.888882)(0.127782,4.888882)(0.134941,4.888882)(0.142100,4.365074)(0.149259,5.238088)(0.156418,4.190471)(0.163578,5.063485)(0.170737,4.539677)(0.177896,4.190471)(0.185055,4.190471)(0.192214,4.539677)(0.199373,4.190471)(0.206532,4.015868)(0.213691,3.841265)(0.220850,3.666662)(0.228009,4.190471)(0.235169,3.492059)(0.242328,3.317456)(0.249487,3.492059)(0.256646,2.968250)(0.263805,3.492059)(0.270964,2.968250)(0.278123,2.968250)(0.285282,2.619044)(0.292441,2.619044)(0.299600,2.619044)(0.306760,2.444441)(0.313919,2.269838)(0.321078,2.095235)(0.328237,2.095235)(0.335396,1.920632)(0.342555,1.746029)(0.349714,1.746029)(0.356873,1.571426)(0.364032,1.047618)(0.371191,1.047618)(0.378351,1.047618)(0.385510,0.523809)(0.392669,0.000000)(0.399828,0.000000)(0.406987,0.000000)(0.414146,0.000000)(0.421305,0.000000)(0.428464,0.000000)
      };
      \addplot[densely dashed,RED,line width=1.5pt] coordinates{
      (0.074089,0.001198)(0.077632,0.001941)(0.081176,0.008712)(0.084719,1.638756)(0.088262,2.830310)(0.091806,3.419787)(0.095349,3.746725)(0.098893,4.028202)(0.102436,4.259827)(0.105979,4.417524)(0.109523,4.537202)(0.113066,4.631219)(0.116610,4.704540)(0.120153,4.755335)(0.123696,4.791691)(0.127240,4.813478)(0.130783,4.825494)(0.134326,4.827450)(0.137870,4.822090)(0.141413,4.809600)(0.144957,4.791234)(0.148500,4.768107)(0.152043,4.740594)(0.155587,4.709298)(0.159130,4.674768)(0.162674,4.637569)(0.166217,4.597842)(0.169760,4.556014)(0.173304,4.512368)(0.176847,4.467129)(0.180391,4.420489)(0.183934,4.372581)(0.187477,4.323653)(0.191021,4.273748)(0.194564,4.223067)(0.198108,4.171643)(0.201651,4.119584)(0.205194,4.066984)(0.208738,4.013905)(0.212281,3.960398)(0.215825,3.906512)(0.219368,3.852291)(0.222911,3.797770)(0.226455,3.742982)(0.229998,3.687954)(0.233541,3.632708)(0.237085,3.577258)(0.240628,3.521614)(0.244172,3.465780)(0.247715,3.409762)(0.251258,3.353565)(0.254802,3.297190)(0.258345,3.240635)(0.261889,3.183897)(0.265432,3.126968)(0.268975,3.069831)(0.272519,3.012447)(0.276062,2.954788)(0.279606,2.896935)(0.283149,2.838823)(0.286692,2.780270)(0.290236,2.721544)(0.293779,2.662380)(0.297323,2.602692)(0.300866,2.542611)(0.304409,2.482098)(0.307953,2.421215)(0.311496,2.359412)(0.315040,2.297197)(0.318583,2.233752)(0.322126,2.169782)(0.325670,2.105492)(0.329213,2.039237)(0.332756,1.973098)(0.336300,1.904447)(0.339843,1.834324)(0.343387,1.762073)(0.346930,1.692066)(0.350473,1.615979)(0.354017,1.536775)(0.357560,1.452672)(0.361104,1.367961)(0.364647,1.279695)(0.368190,1.182076)(0.371734,1.101511)(0.375277,0.976339)(0.378821,0.876267)(0.382364,0.761410)(0.385907,0.533967)(0.389451,0.291993)(0.392994,0.000865)(0.396538,0.000461)(0.400081,0.000340)(0.403624,0.000276)(0.407168,0.000234)(0.410711,0.000205)(0.414255,0.000183)(0.417798,0.000165)(0.421341,0.000151)(0.424885,0.000138)
      };
      \addplot[smooth,GREEN,line width=1pt] coordinates{
      (0.074089,0.000000)(0.077632,0.000000)(0.081176,0.000000)(0.084719,1.086935)(0.088262,2.357726)(0.091806,3.028529)(0.095349,3.485937)(0.098893,3.821527)(0.102436,4.076164)(0.105979,4.272643)(0.109523,4.425305)(0.113066,4.543879)(0.116610,4.635312)(0.120153,4.704754)(0.123696,4.756131)(0.127240,4.792504)(0.130783,4.816301)(0.134326,4.829478)(0.137870,4.833626)(0.141413,4.830057)(0.144957,4.819860)(0.148500,4.803946)(0.152043,4.783082)(0.155587,4.757919)(0.159130,4.729011)(0.162674,4.696831)(0.166217,4.661787)(0.169760,4.624229)(0.173304,4.584460)(0.176847,4.542745)(0.180391,4.499311)(0.183934,4.454357)(0.187477,4.408057)(0.191021,4.360562)(0.194564,4.312004)(0.198108,4.262499)(0.201651,4.212148)(0.205194,4.161039)(0.208738,4.109248)(0.212281,4.056842)(0.215825,4.003879)(0.219368,3.950409)(0.222911,3.896475)(0.226455,3.842112)(0.229998,3.787352)(0.233541,3.732219)(0.237085,3.676734)(0.240628,3.620914)(0.244172,3.564769)(0.247715,3.508309)(0.251258,3.451537)(0.254802,3.394455)(0.258345,3.337060)(0.261889,3.279347)(0.265432,3.221308)(0.268975,3.162930)(0.272519,3.104200)(0.276062,3.045098)(0.279606,2.985603)(0.283149,2.925691)(0.286692,2.865334)(0.290236,2.804498)(0.293779,2.743148)(0.297323,2.681242)(0.300866,2.618735)(0.304409,2.555572)(0.307953,2.491698)(0.311496,2.427044)(0.315040,2.361538)(0.318583,2.295094)(0.322126,2.227618)(0.325670,2.158999)(0.329213,2.089112)(0.332756,2.017811)(0.336300,1.944925)(0.339843,1.870253)(0.343387,1.793558)(0.346930,1.714549)(0.350473,1.632873)(0.354017,1.548088)(0.357560,1.459635)(0.361104,1.366779)(0.364647,1.268530)(0.368190,1.163500)(0.371734,1.049626)(0.375277,0.923615)(0.378821,0.779565)(0.382364,0.604677)(0.385907,0.356015)(0.389451,0.000000)(0.392994,0.000000)(0.396538,0.000000)(0.400081,0.000000)(0.403624,0.000000)(0.407168,0.000000)(0.410711,0.000000)(0.414255,0.000000)(0.417798,0.000000)(0.421341,0.000000)(0.424885,0.000000)
      };
    \end{axis}
  \end{tikzpicture}
&
  \begin{tikzpicture}[font=\footnotesize]
    \renewcommand{\axisdefaulttryminticks}{4} 
    \pgfplotsset{every major grid/.append style={densely dashed}}       
    \pgfplotsset{every axis legend/.style={cells={anchor=west},fill=none, at={(0.02,1.02)}, anchor=north west, font=\footnotesize }}
    \begin{axis}[
      width=.45\columnwidth,
      xmin=0,
      ymin=0,
      xmax=0.27,
      ymax=30,
      xtick={0,0.1,0.2},
      ymajorticks=false,
      bar width=2.5pt,
      grid=major,
      ymajorgrids=false,
      scaled ticks=true,
      xlabel={},
      ylabel={Distribution of $g$}
      ]
      \addplot+[ybar,mark=none,draw=white,fill=BLUE!60!white,area legend] coordinates{
      (0.020649,0.031780)(0.025893,0.000000)(0.031138,0.063560)(0.036382,0.158901)(0.041626,0.190681)(0.046871,0.349583)(0.052115,0.286022)(0.057360,0.381363)(0.062604,0.540264)(0.067848,0.826286)(0.073093,0.444923)(0.078337,0.730946)(0.083581,0.635605)(0.088826,0.889847)(0.094070,1.334770)(0.099314,0.858067)(0.104559,1.525452)(0.109803,1.430111)(0.115047,1.557232)(0.120292,1.811474)(0.125536,1.557232)(0.130780,2.161056)(0.136025,2.161056)(0.141269,2.447079)(0.146513,2.764881)(0.151758,3.019123)(0.157002,3.178024)(0.162246,3.209804)(0.167491,4.194992)(0.172735,4.163212)(0.177980,3.718288)(0.183224,4.385673)(0.188468,4.798817)(0.193713,5.243740)(0.198957,6.133587)(0.204201,5.911125)(0.209446,6.133587)(0.214690,8.389984)(0.219934,8.676006)(0.225179,8.517105)(0.230423,10.201458)(0.235667,12.267173)(0.240912,15.635879)(0.246156,27.108546)(0.251400,20.657157)(0.256645,0.000000)(0.261889,0.000000)(0.267133,0.000000)(0.272378,0.000000)(0.277622,0.000000)
      };
      \addplot[smooth,GREEN,line width=3pt] coordinates{
      (0.2081,0)(0.2081,30)
      };
      \legend{ {Empirical $g$}, { Scale factor} };
    \end{axis}
  \end{tikzpicture}
  \\
  \begin{tikzpicture}[font=\footnotesize]
    \renewcommand{\axisdefaulttryminticks}{4} 
    \pgfplotsset{every major grid/.append style={densely dashed}}       
    \pgfplotsset{every axis legend/.style={cells={anchor=west},fill=none, at={(0.98,0.98)}, anchor=north east, font=\footnotesize}}
    \begin{axis}[
      width=.45\columnwidth,
      xmin=0,
      ymin=0,
      xmax=9,
      ymax=0.7,
      ymajorticks=false,
      bar width=2.5pt,
      grid=major,
      ymajorgrids=false,
      scaled ticks=true,
      xlabel={},
      ylabel={Eigenvalues}
      ]
      \addplot+[ybar,mark=none,draw=white,fill=BLUE!60!white,area legend] coordinates{
      (0.506334,0.339290)(0.679490,0.613609)(0.852646,0.613609)(1.025802,0.555857)(1.198958,0.483668)(1.372114,0.447574)(1.545270,0.375384)(1.718426,0.332071)(1.891582,0.295976)(2.064738,0.259881)(2.237893,0.223787)(2.411049,0.209349)(2.584205,0.158816)(2.757361,0.144379)(2.930517,0.137160)(3.103673,0.108284)(3.276829,0.086627)(3.449985,0.086627)(3.623141,0.064970)(3.796296,0.050532)(3.969452,0.036095)(4.142608,0.036095)(4.315764,0.036095)(4.488920,0.014438)(4.662076,0.014438)(4.835232,0.007219)(5.008388,0.007219)(5.181544,0.007219)(5.354700,0.000000)(5.527855,0.000000)(5.701011,0.000000)(5.874167,0.007219)(6.047323,0.007219)(6.220479,0.000000)(6.393635,0.000000)(6.566791,0.000000)(6.739947,0.007219)(6.913103,0.000000)(7.086259,0.000000)(7.259414,0.000000)(7.432570,0.000000)(7.605726,0.000000)(7.778882,0.000000)(7.952038,0.000000)(8.125194,0.007219)(8.298350,0.000000)(8.471506,0.000000)(8.644662,0.000000)(8.817818,0.000000)(8.990973,0.000000)
      };
      \addplot[densely dashed,RED,line width=1.5pt] coordinates{
      (0.419756,0.000029)(0.505460,0.421962)(0.591163,0.592947)(0.676867,0.635292)(0.762570,0.634662)(0.848274,0.615869)(0.933977,0.589040)(1.019680,0.558927)(1.105384,0.527905)(1.191087,0.497209)(1.276791,0.467478)(1.362494,0.439025)(1.448198,0.411987)(1.533901,0.386398)(1.619604,0.362240)(1.705308,0.339457)(1.791011,0.317989)(1.876715,0.297757)(1.962418,0.278686)(2.048122,0.260707)(2.133825,0.243747)(2.219528,0.227739)(2.305232,0.212623)(2.390935,0.198342)(2.476639,0.184845)(2.562342,0.172082)(2.648046,0.160011)(2.733749,0.148591)(2.819452,0.137789)(2.905156,0.127569)(2.990859,0.117904)(3.076563,0.108768)(3.162266,0.100135)(3.247969,0.091986)(3.333673,0.084302)(3.419376,0.077066)(3.505080,0.070264)(3.590783,0.063884)(3.676487,0.057913)(3.762190,0.052341)(3.847893,0.047158)(3.933597,0.042357)(4.019300,0.037928)(4.105004,0.033860)(4.190707,0.030142)(4.276411,0.026762)(4.362114,0.023706)(4.447817,0.020958)(4.533521,0.018501)(4.619224,0.016313)(4.704928,0.014375)(4.790631,0.012664)(4.876335,0.011159)(4.962038,0.009839)(5.047741,0.008683)(5.133445,0.007673)(5.219148,0.006790)(5.304852,0.006019)(5.390555,0.005345)(5.476259,0.004756)(5.561962,0.004240)(5.647665,0.003787)(5.733369,0.003390)(5.819072,0.003041)(5.904776,0.002733)(5.990479,0.002461)(6.076182,0.002221)(6.161886,0.002008)(6.247589,0.001818)(6.333293,0.001650)(6.418996,0.001499)(6.504700,0.001365)(6.590403,0.001245)(6.676106,0.001137)(6.761810,0.001040)(6.847513,0.000953)(6.933217,0.000874)(7.018920,0.000803)(7.104624,0.000739)(7.190327,0.000681)(7.276030,0.000628)(7.361734,0.000580)(7.447437,0.000536)(7.533141,0.000496)(7.618844,0.000460)(7.704548,0.000427)(7.790251,0.000396)(7.875954,0.000368)(7.961658,0.000342)(8.047361,0.000319)(8.133065,0.000297)(8.218768,0.000277)(8.304472,0.000259)(8.390175,0.000242)(8.475878,0.000226)(8.561582,0.000212)(8.647285,0.000198)(8.732989,0.000186)(8.818692,0.000175)(8.904395,0.000164)
      };
      \addplot[smooth,GREEN,line width=1pt] coordinates{
      (0.419756,0.000000)(0.505460,0.000000)(0.591163,0.000000)(0.676867,0.000000)(0.762570,0.000000)(0.848274,0.000000)(0.933977,0.000000)(1.019680,0.000000)(1.105384,0.000000)(1.191087,0.000000)(1.276791,0.000000)(1.362494,0.000000)(1.448198,0.000000)(1.533901,0.000000)(1.619604,0.000000)(1.705308,0.000000)(1.791011,0.000000)(1.876715,0.000000)(1.962418,0.000000)(2.048122,0.000000)(2.133825,0.050118)(2.219528,0.095516)(2.305232,0.120908)(2.390935,0.138501)(2.476639,0.151543)(2.562342,0.161526)(2.648046,0.169295)(2.733749,0.175383)(2.819452,0.180159)(2.905156,0.183884)(2.990859,0.186754)(3.076563,0.188919)(3.162266,0.190496)(3.247969,0.191579)(3.333673,0.192242)(3.419376,0.192548)(3.505080,0.192547)(3.590783,0.192282)(3.676487,0.191788)(3.762190,0.191095)(3.847893,0.190230)(3.933597,0.189213)(4.019300,0.188063)(4.105004,0.186796)(4.190707,0.185428)(4.276411,0.183968)(4.362114,0.182429)(4.447817,0.180819)(4.533521,0.179146)(4.619224,0.177417)(4.704928,0.175639)(4.790631,0.173817)(4.876335,0.171954)(4.962038,0.170057)(5.047741,0.168128)(5.133445,0.166170)(5.219148,0.164187)(5.304852,0.162180)(5.390555,0.160153)(5.476259,0.158106)(5.561962,0.156041)(5.647665,0.153960)(5.733369,0.151864)(5.819072,0.149754)(5.904776,0.147630)(5.990479,0.145493)(6.076182,0.143345)(6.161886,0.141184)(6.247589,0.139012)(6.333293,0.136829)(6.418996,0.134634)(6.504700,0.132427)(6.590403,0.130209)(6.676106,0.127978)(6.761810,0.125735)(6.847513,0.123478)(6.933217,0.121208)(7.018920,0.118923)(7.104624,0.116623)(7.190327,0.114307)(7.276030,0.111972)(7.361734,0.109619)(7.447437,0.107246)(7.533141,0.104850)(7.618844,0.102431)(7.704548,0.099986)(7.790251,0.097513)(7.875954,0.095009)(7.961658,0.092471)(8.047361,0.089897)(8.133065,0.087283)(8.218768,0.084623)(8.304472,0.081914)(8.390175,0.079150)(8.475878,0.076325)(8.561582,0.073430)(8.647285,0.070456)(8.732989,0.067393)(8.818692,0.064227)(8.904395,0.060941)
      };
    \end{axis}
  \end{tikzpicture}
&
  \begin{tikzpicture}[font=\footnotesize]
    \renewcommand{\axisdefaulttryminticks}{4} 
    \pgfplotsset{every major grid/.append style={densely dashed}}       
    \pgfplotsset{every axis legend/.style={cells={anchor=west},fill=none, at={(0,0.98)}, anchor=north west, font=\footnotesize}}
    \begin{axis}[
      width=.45\columnwidth,
      xmin=0,
      ymin=0,
      xmax=10,
      ymax=.5,
      ymajorticks=false,
      bar width=2.5pt,
      grid=major,
      ymajorgrids=false,
      scaled ticks=true,
      xlabel={},
      ylabel={Distribution of $g$}
      ]
      \addplot+[ybar,mark=none,draw=white,fill=BLUE!60!white,area legend] coordinates{
      (0.138467,0.370366)(0.346370,0.650144)(0.554273,0.610061)(0.762175,0.506647)(0.970078,0.387200)(1.177981,0.327076)(1.385883,0.290200)(1.593786,0.231679)(1.801689,0.187588)(2.009591,0.130670)(2.217494,0.153117)(2.425397,0.109827)(2.633300,0.091389)(2.841202,0.083372)(3.049105,0.069744)(3.257008,0.057719)(3.464910,0.044091)(3.672813,0.052909)(3.880716,0.048901)(4.088618,0.032066)(4.296521,0.032066)(4.504424,0.027256)(4.712327,0.028860)(4.920229,0.019240)(5.128132,0.020041)(5.336035,0.021645)(5.543937,0.024851)(5.751840,0.012025)(5.959743,0.012025)(6.167645,0.016033)(6.375548,0.012025)(6.583451,0.011223)(6.791353,0.006413)(6.999256,0.007215)(7.207159,0.009620)(7.415062,0.008818)(7.622964,0.003207)(7.830867,0.004810)(8.038770,0.008818)(8.246672,0.003207)(8.454575,0.002405)(8.662478,0.008818)(8.870380,0.004008)(9.078283,0.005612)(9.286186,0.004008)(9.494088,0.004008)(9.701991,0.003207)(9.909894,0.002405)(10.117797,0.005612)(10.325699,0.001603)(10.533602,0.001603)(10.741505,0.001603)(10.949407,0.002405)(11.157310,0.001603)(11.365213,0.003207)(11.573115,0.003207)(11.781018,0.002405)(11.988921,0.002405)(12.196824,0.002405)(12.404726,0.000802)(12.612629,0.002405)(12.820532,0.000000)(13.028434,0.000802)(13.236337,0.000802)(13.444240,0.000000)(13.652142,0.000802)(13.860045,0.001603)(14.067948,0.000000)(14.275850,0.000000)(14.483753,0.000000)(14.691656,0.000802)(14.899559,0.000000)(15.107461,0.000000)(15.315364,0.002405)(15.523267,0.000802)(15.731169,0.000000)(15.939072,0.001603)(16.146975,0.000802)(16.354877,0.000802)(16.562780,0.000000)(16.770683,0.001603)(16.978585,0.000000)(17.186488,0.000000)(17.394391,0.000000)(17.602294,0.000000)(17.810196,0.000000)(18.018099,0.000802)(18.226002,0.000000)(18.433904,0.000000)(18.641807,0.000802)(18.849710,0.000000)(19.057612,0.000000)(19.265515,0.000000)(19.473418,0.000000)(19.681320,0.000000)(19.889223,0.000000)(20.097126,0.000000)(20.305029,0.000000)(20.512931,0.000000)(20.720834,0.000000)(20.928737,0.000000)(21.136639,0.000000)(21.344542,0.000000)(21.552445,0.000000)(21.760347,0.000000)(21.968250,0.000000)(22.176153,0.000000)(22.384056,0.000000)(22.591958,0.000802)(22.799861,0.000802)(23.007764,0.000802)(23.215666,0.000000)(23.423569,0.000000)(23.631472,0.000000)(23.839374,0.000000)(24.047277,0.000000)(24.255180,0.000000)(24.463082,0.000000)(24.670985,0.000000)(24.878888,0.000000)(25.086791,0.000000)(25.294693,0.000000)(25.502596,0.000000)(25.710499,0.000000)(25.918401,0.000802)(26.126304,0.000000)(26.334207,0.000000)(26.542109,0.000000)(26.750012,0.000000)(26.957915,0.000000)(27.165817,0.000000)(27.373720,0.000000)(27.581623,0.000000)(27.789526,0.000000)(27.997428,0.000000)(28.205331,0.000000)(28.413234,0.000000)(28.621136,0.000000)(28.829039,0.000802)(29.036942,0.000000)(29.244844,0.000000)(29.452747,0.000000)(29.660650,0.000000)(29.868553,0.000000)(30.076455,0.000000)(30.284358,0.000000)(30.492261,0.000000)(30.700163,0.000000)(30.908066,0.000000)(31.115969,0.000000)(31.323871,0.000000)(31.531774,0.000000)(31.739677,0.000000)(31.947579,0.000000)(32.155482,0.000000)(32.363385,0.000000)(32.571288,0.000000)(32.779190,0.000000)(32.987093,0.000000)(33.194996,0.000000)(33.402898,0.000000)(33.610801,0.000000)(33.818704,0.000000)(34.026606,0.000000)(34.234509,0.000000)(34.442412,0.000000)(34.650314,0.000000)(34.858217,0.000000)(35.066120,0.000000)(35.274023,0.000000)(35.481925,0.000802)(35.689828,0.000000)(35.897731,0.000000)(36.105633,0.000000)(36.313536,0.000000)(36.521439,0.000000)(36.729341,0.000000)(36.937244,0.000000)(37.145147,0.000000)(37.353050,0.000000)(37.560952,0.000000)(37.768855,0.000000)(37.976758,0.000000)(38.184660,0.000000)(38.392563,0.000000)(38.600466,0.000000)(38.808368,0.000000)(39.016271,0.000000)(39.224174,0.000000)(39.432076,0.000000)(39.639979,0.000000)(39.847882,0.000000)(40.055785,0.000000)(40.263687,0.000000)(40.471590,0.000000)(40.679493,0.000000)(40.887395,0.000000)(41.095298,0.000000)(41.303201,0.000000)(41.511103,0.000802)
      };
      \addplot[smooth,GREEN,line width=3pt] coordinates{
      (5.22,0)(5.22,30)
      };
    \end{axis}
  \end{tikzpicture}
\end{tabular}
 \caption{ Comparison of Hessian eigenspectra with the Mar\u{c}enko-Pastur law in the setting of Figure~\ref{fig:log-exp-lsd}. \textbf{(Top)}: Mar\u{c}enko-Pastur-like Hessian with logistic loss; \textbf{(left)} Hessian eigenvalues and \textbf{(right)} empirical distribution of the $g_i$s versus the scaling factor (empirically obtained from the distance between the maximum and minimum Hessian eigenvalues to fit the Mar\u{c}enko-Pastur law). 
 Comparing to Figure~\ref{fig:log-exp-lsd}-(right) and Figure~\ref{fig:logistic-lsd}, we see that changes in the choice of loss and/or feature statistics (e.g., heterogeneous features) result in very different and non-Mar\u{c}enko-Pastur-like Hessian spectral behavior. 
 \textbf{(Bottom)}: an example of a very non-Mar\u{c}enko-Pastur-like Hessian with exponential loss \textbf{(left)} and the associated $g_i$s \textbf{(right)}. 
 Note that the scales of the axes are different in different subfigures.
}
\label{fig:logistic-fit-MP}
\end{center}
\vskip -0.1in
\end{figure}
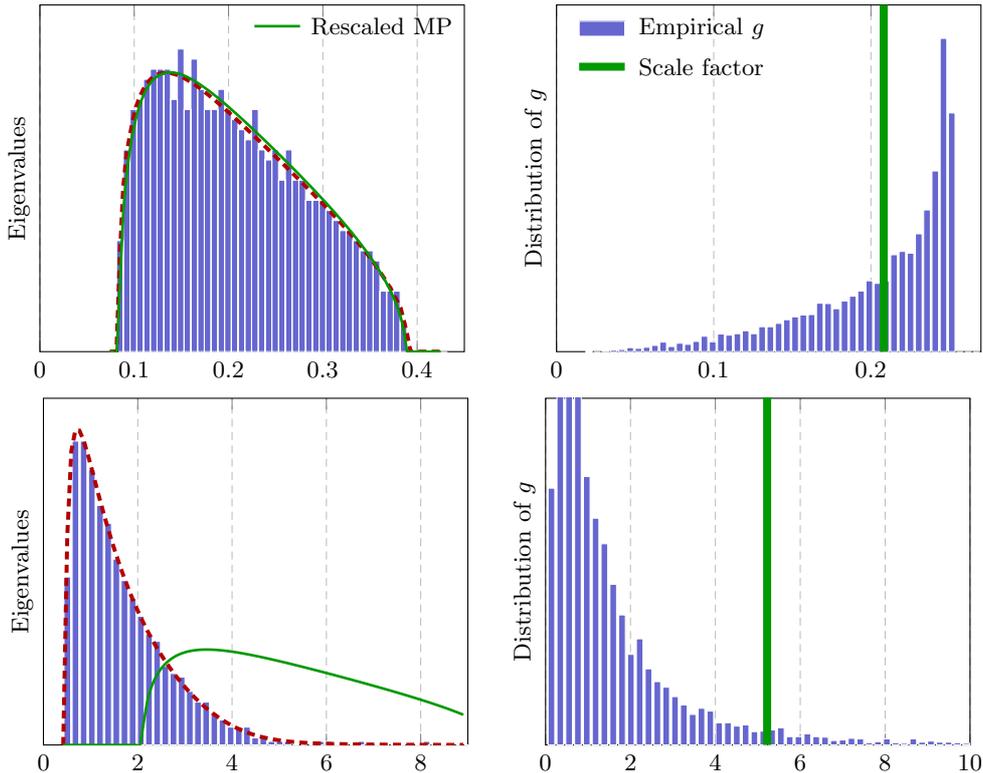

\medskip

It should be mentioned that this ``empirical fit'' has already been observed in \cite[Figure~5]{pedregosa2020average}, where acceleration methods proposed for a Mar\u{c}enko-Pastur distributed Hessian (in linear least squares) work reasonably well on logistic regression problems. 
Our theory proposes a convincing explanation of this empirical observation on logistic regression, and possibly for others more involved models such as NNs.

However, it must be pointed out that this ``visual approximation'' by  Mar\u{c}enko-Pastur law is \emph{not robust}, in the sense that it ``visually'' holds only for (yet still is formally different from) the case of (i) logistic model with (ii) logistic loss and (iii) identity covariance $\C = \I_p$.
Any change in the response model (e.g., with the phase retrieval model in Figure~\ref{fig:intro-c}), the choice of loss function (e.g., with the exponential loss in Figure~\ref{fig:log-exp-lsd}-right), or going beyond the simple identity covariance (that models very homogeneous features) setting, as in Figure~\ref{fig:logistic-lsd}, would induce a Hessian spectrum that is \emph{very} different from the Mar\u{c}enko-Pastur law. 
In this vein, our Theorem~\ref{theo:lsd} goes beyond such a ``loose'' Mar\u{c}enko-Pastur approximation.  
It provides a first example toward the understanding of Hessians of deep NNs and other much more realistic models that exploit nonlinear transformations (such as activation function) and feature statistics.

\medskip

In the proofs of Theorem~\ref{theo:lsd}--\ref{theo:DE}, we explicitly used the fact that, for jointly Gaussian random variables, uncorrelatedness implies independence, and therefore our main results (on the limiting Hessian eigenvalue distribution and the possible isolated eigenpairs) are proven to hold only for Gaussian features. On the other hand, universality phenomena commonly arise in random matrix theory, and it has been proven for various models that the limiting spectral measure \cite{tao2010random}, the spectrum edges \cite{bourgade2014edge,yang2019edge}, and even involved functionals of the eigenvalues \cite{bai2008clt} \emph{do not} depend on the distribution of the random matrix, so long as its entries are independent and have zero mean, unit variance, and finite higher, say fourth, order moment.\footnote{Counterexample exists, though, for instance it has been shown in \cite{fan2019spectral,liao2020sparse} that, even in the presence of signal, nonlinear random matrices may produce ``noisy spikes'' (that are purely random and carry no information about the data signal and) that, in particular, depends on the fourth moment of the distribution.}

Figure~\ref{fig:lsd-non-gauss} compares the Hessian eigenvalues for feature following Gaussian, Bernoulli, and Student-t distribution, where we empirically confirm that Theorem~\ref{theo:lsd} provides consistent predictions for non-Gaussian features with reasonably large $n,p$.
We conjecture that Theorem~\ref{theo:lsd}, proven here only for Gaussian distribution, holds more generally. 

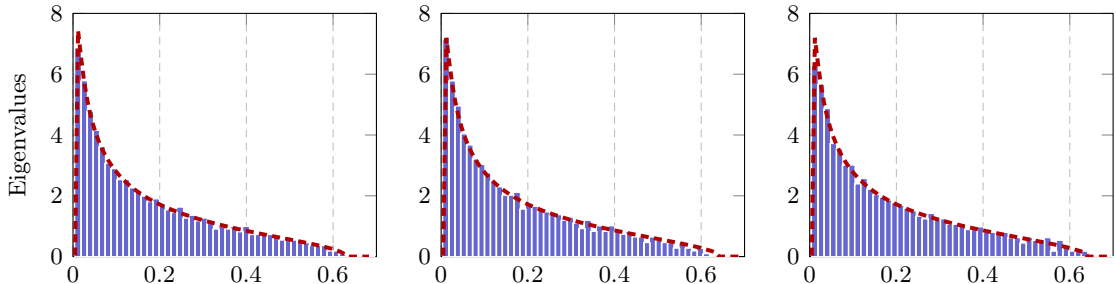
\begin{figure}[htb]
\centering
\begin{tabular}{ccc}
\centering
  \begin{tikzpicture}[font=\footnotesize]
    \renewcommand{\axisdefaulttryminticks}{4} 
    \pgfplotsset{every major grid/.append style={densely dashed}}       
    \pgfplotsset{every axis legend/.style={cells={anchor=west},fill=none, at={(0.98,0.98)}, anchor=north east, font=\footnotesize}}
    \begin{axis}[
      width=.35\textwidth,
      xmin=0,
      xmax=0.7,
      ymin=0,
      ymax=8,
      bar width=2.5pt,
      grid=major,
      ymajorgrids=false,
      scaled ticks=true,
      xlabel={},
      ylabel={Eigenvalues}
      ]
      \addplot+[ybar,mark=none,draw=white,fill=BLUE!60!white,area legend] coordinates{
      (0.011721,6.864374)(0.025561,5.780526)(0.039400,4.786998)(0.053240,4.154753)(0.067080,3.612828)(0.080919,3.070904)(0.094759,2.890263)(0.108598,2.528980)(0.122438,2.528980)(0.136277,2.258018)(0.150117,2.167697)(0.163957,1.987056)(0.177796,1.806414)(0.191636,1.896735)(0.205475,1.716094)(0.219315,1.535452)(0.233154,1.445131)(0.246994,1.625773)(0.260834,1.264490)(0.274673,1.354811)(0.288513,1.174169)(0.302352,1.264490)(0.316192,1.174169)(0.330031,0.903207)(0.343871,0.993528)(0.357711,0.903207)(0.371550,0.993528)(0.385390,0.812886)(0.399229,0.993528)(0.413069,0.722566)(0.426908,0.722566)(0.440748,0.722566)(0.454588,0.722566)(0.468427,0.632245)(0.482267,0.541924)(0.496106,0.632245)(0.509946,0.541924)(0.523785,0.541924)(0.537625,0.451604)(0.551465,0.361283)(0.565304,0.361283)(0.579144,0.361283)(0.592983,0.180641)(0.606823,0.180641)(0.620663,0.090321)(0.634502,0.000000)(0.648342,0.000000)(0.662181,0.000000)(0.676021,0.000000)(0.689860,0.000000)
      };
      \addplot[densely dashed,RED,line width=1.5pt] coordinates{
      (0.004802,0.036413)(0.011651,7.429039)(0.018501,6.559183)(0.025351,5.795190)(0.032201,5.202344)(0.039051,4.736158)(0.045901,4.359785)(0.052751,4.048462)(0.059601,3.785683)(0.066451,3.560097)(0.073300,3.363713)(0.080150,3.190709)(0.087000,3.036761)(0.093850,2.898576)(0.100700,2.773587)(0.107550,2.659798)(0.114400,2.555589)(0.121250,2.459659)(0.128100,2.370924)(0.134949,2.288513)(0.141799,2.211668)(0.148649,2.139767)(0.155499,2.072276)(0.162349,2.008741)(0.169199,1.948762)(0.176049,1.891997)(0.182899,1.838149)(0.189749,1.786958)(0.196598,1.738189)(0.203448,1.691640)(0.210298,1.647125)(0.217148,1.604491)(0.223998,1.563587)(0.230848,1.524282)(0.237698,1.486456)(0.244548,1.450009)(0.251398,1.414835)(0.258247,1.380856)(0.265097,1.347979)(0.271947,1.316141)(0.278797,1.285274)(0.285647,1.255308)(0.292497,1.226189)(0.299347,1.197863)(0.306197,1.170282)(0.313047,1.143400)(0.319896,1.117175)(0.326746,1.091566)(0.333596,1.066538)(0.340446,1.042047)(0.347296,1.018066)(0.354146,0.994564)(0.360996,0.971516)(0.367846,0.948884)(0.374696,0.926642)(0.381545,0.904775)(0.388395,0.883244)(0.395245,0.862033)(0.402095,0.841124)(0.408945,0.820479)(0.415795,0.800092)(0.422645,0.779932)(0.429495,0.759987)(0.436345,0.740222)(0.443194,0.720625)(0.450044,0.701174)(0.456894,0.681847)(0.463744,0.662627)(0.470594,0.643476)(0.477444,0.624391)(0.484294,0.605332)(0.491144,0.586276)(0.497994,0.567202)(0.504843,0.548072)(0.511693,0.528854)(0.518543,0.509519)(0.525393,0.490019)(0.532243,0.470312)(0.539093,0.450342)(0.545943,0.430060)(0.552793,0.409366)(0.559643,0.388190)(0.566492,0.366394)(0.573342,0.343871)(0.580192,0.320455)(0.587042,0.295679)(0.593892,0.269462)(0.600742,0.240502)(0.607592,0.208800)(0.614442,0.171695)(0.621292,0.126701)(0.628141,0.062560)(0.634991,0.000110)(0.641841,0.000063)(0.648691,0.000049)(0.655541,0.000041)(0.662391,0.000035)(0.669241,0.000032)(0.676091,0.000029)(0.682941,0.000026)
      };
    \end{axis}
  \end{tikzpicture}
&
  \begin{tikzpicture}[font=\footnotesize]
    \renewcommand{\axisdefaulttryminticks}{4} 
    \pgfplotsset{every major grid/.append style={densely dashed}}       
    \pgfplotsset{every axis legend/.append style={cells={anchor=west},fill=none, at={(0.98,0.98)}, anchor=north east, font=\footnotesize }}
    \begin{axis}[
      width=.35\textwidth,
      xmin=0,
      xmax=0.7,
      ymin=0,
      ymax=8,
      bar width=2.5pt,
      grid=major,
      ymajorgrids=false,
      scaled ticks=true,
      xlabel={},
      ylabel={}
      ]
      \addplot+[ybar,mark=none,draw=white,fill=BLUE!60!white,area legend] coordinates{
      (0.011377,7.154081)(0.025006,5.778296)(0.038634,4.952826)(0.052263,4.035636)(0.065891,3.668760)(0.079520,3.210165)(0.093148,3.026727)(0.106777,2.751570)(0.120406,2.476413)(0.134034,2.292975)(0.147663,2.017818)(0.161291,2.017818)(0.174920,2.109537)(0.188549,1.559223)(0.202177,1.742661)(0.215806,1.650942)(0.229434,1.559223)(0.243063,1.467504)(0.256691,1.375785)(0.270320,1.375785)(0.283949,1.192347)(0.297577,1.284066)(0.311206,1.100628)(0.324834,0.917190)(0.338463,1.192347)(0.352092,0.825471)(0.365720,1.008909)(0.379349,0.825471)(0.392977,1.008909)(0.406606,0.825471)(0.420234,0.733752)(0.433863,0.825471)(0.447492,0.642033)(0.461120,0.642033)(0.474749,0.458595)(0.488377,0.642033)(0.502006,0.550314)(0.515635,0.458595)(0.529263,0.458595)(0.542892,0.275157)(0.556520,0.458595)(0.570149,0.275157)(0.583777,0.183438)(0.597406,0.275157)(0.611035,0.091719)(0.624663,0.000000)(0.638292,0.000000)(0.651920,0.000000)(0.665549,0.000000)(0.679178,0.000000)
      };
      \addplot[densely dashed,RED,line width=1.5pt] coordinates{
      (0.004939,0.030255)(0.011931,7.196261)(0.018924,6.370512)(0.025916,5.642833)(0.032908,5.074135)(0.039900,4.625377)(0.046892,4.262261)(0.053884,3.961409)(0.060876,3.707114)(0.067868,3.488583)(0.074860,3.298139)(0.081852,3.130226)(0.088844,2.980685)(0.095836,2.846348)(0.102828,2.724766)(0.109820,2.614002)(0.116812,2.512499)(0.123804,2.419003)(0.130796,2.332481)(0.137788,2.252066)(0.144780,2.177048)(0.151772,2.106821)(0.158765,2.040872)(0.165757,1.978752)(0.172749,1.920082)(0.179741,1.864530)(0.186733,1.811812)(0.193725,1.761667)(0.200717,1.713876)(0.207709,1.668242)(0.214701,1.624590)(0.221693,1.582759)(0.228685,1.542603)(0.235677,1.504009)(0.242669,1.466850)(0.249661,1.431031)(0.256653,1.396448)(0.263645,1.363026)(0.270637,1.330686)(0.277629,1.299347)(0.284621,1.268950)(0.291613,1.239434)(0.298606,1.210740)(0.305598,1.182816)(0.312590,1.155616)(0.319582,1.129101)(0.326574,1.103220)(0.333566,1.077934)(0.340558,1.053214)(0.347550,1.029020)(0.354542,1.005315)(0.361534,0.982078)(0.368526,0.959280)(0.375518,0.936881)(0.382510,0.914870)(0.389502,0.893208)(0.396494,0.871883)(0.403486,0.850862)(0.410478,0.830123)(0.417470,0.809652)(0.424462,0.789418)(0.431454,0.769404)(0.438447,0.749587)(0.445439,0.729948)(0.452431,0.710463)(0.459423,0.691112)(0.466415,0.671879)(0.473407,0.652728)(0.480399,0.633655)(0.487391,0.614618)(0.494383,0.595603)(0.501375,0.576580)(0.508367,0.557515)(0.515359,0.538384)(0.522351,0.519152)(0.529343,0.499778)(0.536335,0.480227)(0.543327,0.460436)(0.550319,0.440359)(0.557311,0.419916)(0.564303,0.399038)(0.571295,0.377655)(0.578288,0.355673)(0.585280,0.332791)(0.592272,0.308707)(0.599264,0.283251)(0.606256,0.257320)(0.613248,0.225176)(0.620240,0.198472)(0.627232,0.164826)(0.634224,0.127548)(0.641216,0.000254)(0.648208,0.000075)(0.655200,0.000054)(0.662192,0.000043)(0.669184,0.000037)(0.676176,0.000033)(0.683168,0.000029)(0.690160,0.000027)(0.697152,0.000025)
      };
    \end{axis}
  \end{tikzpicture}
  &
  \begin{tikzpicture}[font=\footnotesize]
    \renewcommand{\axisdefaulttryminticks}{4} 
    \pgfplotsset{every major grid/.append style={densely dashed}}       
    \pgfplotsset{every axis legend/.append style={cells={anchor=west},fill=none, at={(0.98,0.98)}, anchor=north east, font=\footnotesize }}
    \begin{axis}[
      width=.35\textwidth,
      xmin=0,
      xmax=0.7,
      ymin=0,
      ymax=8,
      bar width=2.5pt,
      grid=major,
      ymajorgrids=false,
      scaled ticks=true,
      xlabel={},
      ylabel={}
      ]
      \addplot+[ybar,mark=none,draw=white,fill=BLUE!60!white,area legend] coordinates{
      (0.012003,6.282388)(0.026130,5.662997)(0.040256,4.866638)(0.054383,3.716342)(0.068510,3.539373)(0.082637,3.008467)(0.096764,3.008467)(0.110890,2.389077)(0.125017,2.566046)(0.139144,2.212108)(0.153271,2.035140)(0.167398,1.946655)(0.181524,1.858171)(0.195651,1.769687)(0.209778,1.681202)(0.223905,1.592718)(0.238032,1.504234)(0.252158,1.327265)(0.266285,1.238781)(0.280412,1.415749)(0.294539,1.238781)(0.308666,1.238781)(0.322792,1.061812)(0.336919,1.061812)(0.351046,1.061812)(0.365173,0.884843)(0.379299,0.884843)(0.393426,0.973328)(0.407553,0.796359)(0.421680,0.796359)(0.435807,0.796359)(0.449933,0.796359)(0.464060,0.619390)(0.478187,0.619390)(0.492314,0.442422)(0.506441,0.619390)(0.520567,0.530906)(0.534694,0.530906)(0.548821,0.619390)(0.562948,0.176969)(0.577075,0.530906)(0.591201,0.353937)(0.605328,0.176969)(0.619455,0.176969)(0.633582,0.176969)(0.647709,0.000000)(0.661835,0.000000)(0.675962,0.000000)(0.690089,0.000000)(0.704216,0.000000)
      };
      \addplot[densely dashed,RED,line width=1.5pt] coordinates{
      (0.004939,0.030255)(0.011931,7.196261)(0.018924,6.370512)(0.025916,5.642833)(0.032908,5.074135)(0.039900,4.625377)(0.046892,4.262261)(0.053884,3.961409)(0.060876,3.707114)(0.067868,3.488583)(0.074860,3.298139)(0.081852,3.130226)(0.088844,2.980685)(0.095836,2.846348)(0.102828,2.724766)(0.109820,2.614002)(0.116812,2.512499)(0.123804,2.419003)(0.130796,2.332481)(0.137788,2.252066)(0.144780,2.177048)(0.151772,2.106821)(0.158765,2.040872)(0.165757,1.978752)(0.172749,1.920082)(0.179741,1.864530)(0.186733,1.811812)(0.193725,1.761667)(0.200717,1.713876)(0.207709,1.668242)(0.214701,1.624590)(0.221693,1.582759)(0.228685,1.542603)(0.235677,1.504009)(0.242669,1.466850)(0.249661,1.431031)(0.256653,1.396448)(0.263645,1.363026)(0.270637,1.330686)(0.277629,1.299347)(0.284621,1.268950)(0.291613,1.239434)(0.298606,1.210740)(0.305598,1.182816)(0.312590,1.155616)(0.319582,1.129101)(0.326574,1.103220)(0.333566,1.077934)(0.340558,1.053214)(0.347550,1.029020)(0.354542,1.005315)(0.361534,0.982078)(0.368526,0.959280)(0.375518,0.936881)(0.382510,0.914870)(0.389502,0.893208)(0.396494,0.871883)(0.403486,0.850862)(0.410478,0.830123)(0.417470,0.809652)(0.424462,0.789418)(0.431454,0.769404)(0.438447,0.749587)(0.445439,0.729948)(0.452431,0.710463)(0.459423,0.691112)(0.466415,0.671879)(0.473407,0.652728)(0.480399,0.633655)(0.487391,0.614618)(0.494383,0.595603)(0.501375,0.576580)(0.508367,0.557515)(0.515359,0.538384)(0.522351,0.519152)(0.529343,0.499778)(0.536335,0.480227)(0.543327,0.460436)(0.550319,0.440359)(0.557311,0.419916)(0.564303,0.399038)(0.571295,0.377655)(0.578288,0.355673)(0.585280,0.332791)(0.592272,0.308707)(0.599264,0.283251)(0.606256,0.257320)(0.613248,0.225176)(0.620240,0.198472)(0.627232,0.164826)(0.634224,0.127548)(0.641216,0.000254)(0.648208,0.000075)(0.655200,0.000054)(0.662192,0.000043)(0.669184,0.000037)(0.676176,0.000033)(0.683168,0.000029)(0.690160,0.000027)(0.697152,0.000025)
      };
    \end{axis}
  \end{tikzpicture}
\end{tabular}
 \caption{ While proved here only in the Gaussian case, the empirical Hessian eigenvalue distribution appears rather ``stable'' beyond Gaussian distribution: with $p=800$, $n = 1\,200$, logistic model in \eqref{eq:def-logistic} with $\w_* = \zo_p$, $\w = \bmu \sim \NN(\mathbf{0}, \I_p/p)$ and  $\C = \I_p$, for Gaussian \textbf{(left)}, symmetric Bernoulli \textbf{(middle)}, and Student-t distribution with $7$ degrees of freedom \textbf{(right)}. }  
\label{fig:lsd-non-gauss}
\end{figure}

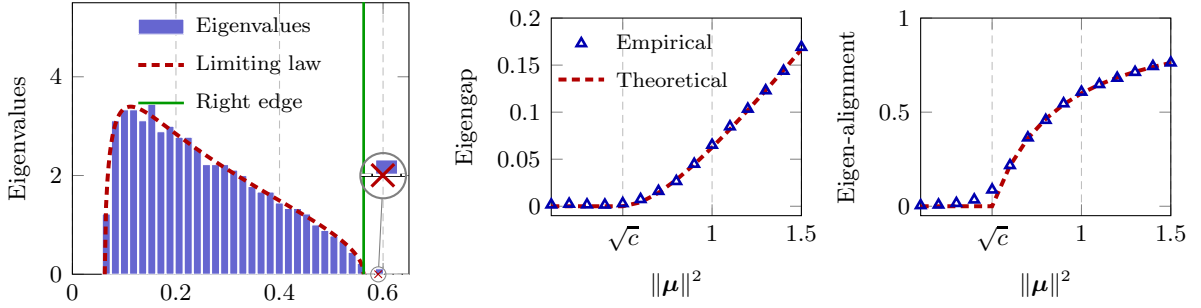
\begin{figure*}[htb]
\vskip 0.1in
  \begin{center}
  \begin{minipage}[c]{0.37\linewidth}
  \centering
  \begin{tikzpicture}[font=\footnotesize,spy using outlines]
    \renewcommand{\axisdefaulttryminticks}{4} 
    \pgfplotsset{every major grid/.append style={densely dashed}}          
    \tikzstyle{every axis y label}+=[yshift=-10pt] 
    \tikzstyle{every axis x label}+=[yshift=5pt]
    \pgfplotsset{every axis legend/.style={cells={anchor=west},fill=none,at={(.8,1.02)}, anchor=north east, font=\scriptsize}}
    \begin{axis}[
    width=1.02\linewidth,
      xmin=0,
      ymin=0,
      xmax=0.65,
      ymax=5.5,
      bar width=3pt,
      grid=major,
      ymajorgrids=false,
      xlabel={},
      ylabel={Eigenvalues}
      ]
      \addplot+[ybar,mark=none,draw=white,fill=BLUE!60!white,area legend] coordinates{
      (0.065405,1.221631)(0.082991,3.109606)(0.100578,3.331721)(0.118165,3.331721)(0.135751,3.109606)(0.153338,3.442778)(0.170924,2.887491)(0.188511,2.998549)(0.206098,2.776434)(0.223684,2.776434)(0.241271,2.554319)(0.258858,2.221147)(0.276444,2.221147)(0.294031,2.221147)(0.311618,2.110090)(0.329204,1.999033)(0.346791,1.776918)(0.364377,1.665860)(0.381964,1.665860)(0.399551,1.443746)(0.417137,1.332688)(0.434724,1.332688)(0.452311,1.221631)(0.469897,0.999516)(0.487484,0.888459)(0.505071,0.777402)(0.522657,0.666344)(0.540244,0.444229)(0.557830,0.222115)(0.575417,0.000000)(0.593004,0.111057)(0.610590,0.000000)(0.628177,0.000000)(0.645764,0.000000)(0.663350,0.000000)
      };
      \def\cc{1/4}
      \addplot[densely dashed,samples=200,domain=1/16:9/16,RED,line width=1.5pt] {2/(pi*\cc*x)*sqrt(max(((1+sqrt(\cc))^2/4-x)*(x-(1-sqrt(\cc))^2/4),0))};
      \addplot[smooth,GREEN,line width=1pt] coordinates{ ((1+sqrt(\cc))^2/4,0)((1+sqrt(\cc))^2/4,8)};
      \addplot+[only marks,mark=x,RED] coordinates{ (0.5906,0) };
      \coordinate (spy0) at (axis cs:0.5906,0.001); \coordinate (spypoint0) at (axis cs:0.6,2);
    \begin{scope}
      \spy[black!50!white,size=.6cm,circle,connect spies,magnification=3] on (spy0) in node [fill=none] at (spypoint0);
      \end{scope}
    \legend{ {Eigenvalues}, {Limiting law}, {Right edge} };
    \end{axis}
  \end{tikzpicture}
  \end{minipage}%
  \hfill{}%
  \begin{minipage}[c]{0.3\linewidth}
  \centering
  	\begin{tikzpicture}[font=\footnotesize]
    \renewcommand{\axisdefaulttryminticks}{4} 
    \pgfplotsset{every major grid/.append style={densely dashed}}          
    \tikzstyle{every axis y label}+=[yshift=-10pt] 
    \tikzstyle{every axis x label}+=[yshift=5pt]
    \pgfplotsset{every axis legend/.style={cells={anchor=west},fill=none,at={(-.02,1.02)}, anchor=north west, font=\scriptsize}}
    \begin{axis}[
    width=1.02\linewidth,
    xmin=.1,xmax=1.5,
    ymin=-0.01,ymax=0.2,
    ylabel={ Eigengap },
    xlabel={$ \| \bmu \|^2 $},
    xtick={0,0.5,1,1.5},
    xticklabels={0,$\sqrt{c}$,1,1.5},
    grid=major,
    ymajorgrids=false,
    yticklabel style={
    /pgf/number format/fixed,
    /pgf/number format/precision=2},
	scaled y ticks=false,
    ]
    \addplot+[mark=triangle,only marks,BLUE,mark size=2pt,line width=1pt] plot coordinates{
    (0.10,0.001601)(0.20,0.002402)(0.30,0.001647)(0.40,0.001435)(0.50,0.003028)(0.60,0.007345)(0.70,0.015996)(0.80,0.026511)(0.90,0.044835)(1.00,0.065055)(1.10,0.084634)(1.20,0.103394)(1.30,0.122933)(1.40,0.143734)(1.50,0.169093)
    };
    \addlegendentry{ Empirical };
    \addplot[densely dashed,RED,line width=1.5pt] plot coordinates{
    (0.10,0.000000)(0.20,0.000000)(0.30,0.000000)(0.40,0.000000)(0.50,0.000000)(0.60,0.004167)(0.70,0.014286)(0.80,0.028125)(0.90,0.044444)(1.00,0.062500)(1.10,0.081818)(1.20,0.102083)(1.30,0.123077)(1.40,0.144643)(1.50,0.166667)
    };
    \addlegendentry{ Theoretical };
    \end{axis}
  \end{tikzpicture}
  \end{minipage}%
  \hfill{}%
  \begin{minipage}[c]{0.3\linewidth}
  \centering
  \begin{tikzpicture}[font=\footnotesize]
    \renewcommand{\axisdefaulttryminticks}{4} 
    \pgfplotsset{every major grid/.append style={densely dashed}}          
    \tikzstyle{every axis y label}+=[yshift=-10pt] 
    \tikzstyle{every axis x label}+=[yshift=5pt]
    \pgfplotsset{every axis legend/.style={cells={anchor=west},fill=none,at={(-.02,1.02)}, anchor=north west, font=\scriptsize}}
    \begin{axis}[
    width=1.02\linewidth,
    xmin=.1,xmax=1.5,
    ymin=-0.05,ymax=1,
    ylabel={ Eigen-alignment },
    xlabel={$ \| \bmu \|^2$},
    xtick={0,0.5,1,1.5},
    xticklabels={0,$\sqrt{c}$,1,1.5},
    grid=major,
    ymajorgrids=false,
    scaled ticks=true,
    ]
    \addplot+[mark=triangle,only marks,BLUE,mark size=2pt,line width=1pt] plot coordinates{
    (0.10,0.004287)(0.20,0.007537)(0.30,0.015833)(0.40,0.034488)(0.50,0.087469)(0.60,0.217211)(0.70,0.364159)(0.80,0.457585)(0.90,0.545469)(1.00,0.607250)(1.10,0.647698)(1.20,0.681405)(1.30,0.712360)(1.40,0.743657)(1.50,0.762516)
    };
    \addplot[densely dashed,RED,line width=1.5pt] plot coordinates{
    (0.10,0.000000)(0.20,0.000000)(0.30,0.000000)(0.40,0.000000)(0.50,0.000000)(0.60,0.215686)(0.70,0.360902)(0.80,0.464286)(0.90,0.541063)(1.00,0.600000)(1.10,0.646465)(1.20,0.683908)(1.30,0.714640)(1.40,0.740260)(1.50,0.761905)
    };
    \end{axis}
  \end{tikzpicture}
  \end{minipage}
  \caption{ Spike (on right-hand side of bulk) due to data signal in Corollary~\ref{coro:special-logistic-w=0}: theory versus practice of \textbf{(left)} Hessian eigenspectrum with $\| \bmu \|^2 = 0.8$, \textbf{(middle)} eigengap $\dist(\lambda_{\bmu}, \supp(\mu))$, and \textbf{(right)} top eigenvector alignment $\alpha$ in \eqref{eq:mu-phase-transition-align}, as a function of the signal strength $ \| \bmu \|^2$, on logistic model with logistic loss, for $\bmu \propto [-\mathbf{1}_{p/2},~\mathbf{1}_{p/2}]$, $\w = \w_* = \mathbf{0}$, $\C = \I_p$, $p=512$ and $n = 2\,048$. Results averaged over $50$ runs. }
  \label{fig:mu-spikes}
  \end{center}
  \vskip -0.1in
\end{figure*}

\begin{figure*}[htb]
\vskip 0.1in
\begin{center}
  \begin{minipage}[c]{0.37\linewidth}
  \centering
  \begin{tikzpicture}[font=\footnotesize,spy using outlines]
    \renewcommand{\axisdefaulttryminticks}{4} 
    \pgfplotsset{every major grid/.append style={densely dashed}}          
    \tikzstyle{every axis y label}+=[yshift=-10pt] 
    \tikzstyle{every axis x label}+=[yshift=5pt]
    \pgfplotsset{every axis legend/.style={cells={anchor=west},fill=none,at={(1.02,1.02)}, anchor=north east, font=\scriptsize}}
    \begin{axis}[
    width=1.02\linewidth,
    xmin=0,
      ymin=0,
      xmax=0.35,
      ymax=10,
      bar width=2pt,
      grid=major,
      ymajorgrids=false,
      xlabel={},
      ylabel={Eigenvalues}
      ]
      \addplot+[ybar,mark=none,draw=white,fill=BLUE!60!white,area legend] coordinates{
      (0.050684,0.163217)(0.058342,0.000000)(0.066001,2.121820)(0.073659,5.059725)(0.081318,6.528678)(0.088976,6.691895)(0.096635,7.018329)(0.104293,6.855112)(0.111952,6.691895)(0.119610,6.855112)(0.127269,6.528678)(0.134927,6.202244)(0.142586,6.202244)(0.150244,5.712593)(0.157903,5.712593)(0.165561,5.386159)(0.173220,4.896508)(0.180878,4.896508)(0.188537,4.570075)(0.196196,4.570075)(0.203854,3.917207)(0.211513,3.917207)(0.219171,3.590773)(0.226830,3.101122)(0.234488,3.101122)(0.242147,2.937905)(0.249805,2.285037)(0.257464,1.958603)(0.265122,1.632169)(0.272781,1.142519)(0.280439,0.326434)(0.288098,0.000000)(0.295756,0.000000)(0.303415,0.000000)(0.311073,0.000000)
      };
      \addplot[densely dashed,RED,line width=1.5pt] coordinates{
      (0.057375,0.001967)(0.060005,0.002880)(0.062636,0.005778)(0.065266,2.047917)(0.067896,3.680700)(0.070526,4.595776)(0.073156,5.223225)(0.075787,5.681680)(0.078417,6.026003)(0.081047,6.287711)(0.083677,6.486965)(0.086307,6.637584)(0.088938,6.749508)(0.091568,6.830139)(0.094198,6.885142)(0.096828,6.918943)(0.099458,6.935058)(0.102089,6.936316)(0.104719,6.925026)(0.107349,6.903085)(0.109979,6.872067)(0.112609,6.833287)(0.115240,6.787851)(0.117870,6.736694)(0.120500,6.680609)(0.123130,6.620274)(0.125760,6.556266)(0.128391,6.489085)(0.131021,6.419157)(0.133651,6.346851)(0.136281,6.272486)(0.138911,6.196339)(0.141542,6.118648)(0.144172,6.039620)(0.146802,5.959435)(0.149432,5.878246)(0.152062,5.796189)(0.154693,5.713378)(0.157323,5.629913)(0.159953,5.545879)(0.162583,5.461346)(0.165213,5.376376)(0.167844,5.291019)(0.170474,5.205316)(0.173104,5.119299)(0.175734,5.032994)(0.178364,4.946419)(0.180995,4.859585)(0.183625,4.772498)(0.186255,4.685159)(0.188885,4.597561)(0.191515,4.509693)(0.194146,4.421540)(0.196776,4.333082)(0.199406,4.244292)(0.202036,4.155137)(0.204666,4.065584)(0.207297,3.975588)(0.209927,3.885102)(0.212557,3.794072)(0.215187,3.702434)(0.217817,3.610122)(0.220447,3.517055)(0.223078,3.423147)(0.225708,3.328298)(0.228338,3.232396)(0.230968,3.135315)(0.233598,3.036909)(0.236229,2.937013)(0.238859,2.835434)(0.241489,2.731950)(0.244119,2.626298)(0.246749,2.518167)(0.249380,2.407185)(0.252010,2.292896)(0.254640,2.174737)(0.257270,2.051985)(0.259900,1.923733)(0.262531,1.788752)(0.265161,1.645238)(0.267791,1.491001)(0.270421,1.320848)(0.273051,1.132808)(0.275682,0.910229)(0.278312,0.606618)(0.280942,0.002874)(0.283572,0.001035)(0.286202,0.000723)(0.288833,0.000573)(0.291463,0.000481)(0.294093,0.000417)(0.296723,0.000369)(0.299353,0.000332)(0.301984,0.000302)(0.304614,0.000277)(0.307244,0.000256)
      };
      \addplot[smooth,GREEN,line width=1pt] coordinates{ (0.062636,0)(0.062636,10)};
      \addplot+[only marks,mark=x,RED] coordinates{ (0.0501,0) };
      \coordinate (spy0) at (axis cs:0.0501,0.001); \coordinate (spypoint0) at (axis cs:0.03,4);
    \begin{scope}
      \spy[black!50!white,size=.6cm,circle,connect spies,magnification=3] on (spy0) in node [fill=none] at (spypoint0);
      \end{scope}
    \legend{ {Eigenvalues}, {Limiting law}, {Left edge} };
    \end{axis}
  \end{tikzpicture}
  \end{minipage}%
  \hfill{}%
  \begin{minipage}[c]{0.3\linewidth}
  \centering
  	\begin{tikzpicture}[font=\footnotesize]
    \renewcommand{\axisdefaulttryminticks}{4} 
    \pgfplotsset{every major grid/.append style={densely dashed}}          
    \tikzstyle{every axis y label}+=[yshift=-10pt] 
    \tikzstyle{every axis x label}+=[yshift=5pt]
    \pgfplotsset{every axis legend/.style={cells={anchor=east},fill=none,at={(1.02,1.02)}, anchor=north east, font=\scriptsize}}
    \begin{axis}[
    width=1.02\linewidth,
    xmin=.1,xmax=8,
    ymin=0,ymax=0.035,
    ylabel={ Eigengap },
    xlabel={$\| \w \| $},
    grid=major,
    ymajorgrids=false,
    yticklabel style={
    /pgf/number format/fixed,
    /pgf/number format/precision=2},
	scaled y ticks=false,
    ]
    \addplot+[mark=triangle,only marks,BLUE,mark size=2pt,line width=1pt] plot coordinates{
    (0.10,0.001152)(0.37,0.001226)(0.64,0.001076)(0.92,0.000961)(1.19,0.000947)(1.46,0.002860)(1.73,0.008010)(2.01,0.012796)(2.28,0.015863)(2.55,0.017782)(2.82,0.019014)(3.10,0.019470)(3.37,0.019547)(3.64,0.019091)(3.91,0.018495)(4.19,0.017718)(4.46,0.017029)(4.73,0.016258)(5.00,0.015375)(5.28,0.014538)(5.55,0.013768)(5.82,0.013081)(6.09,0.012317)(6.37,0.011527)(6.64,0.010926)(6.91,0.010408)(7.18,0.009847)(7.46,0.009283)(7.73,0.008791)(8.00,0.008277)
    };
    \addlegendentry{ Empirical };
    \addplot[densely dashed,RED,line width=1.5pt] plot coordinates{
    (0.10,0.0000)(0.37,0.0000)(0.64,0.0000)(0.92,0.0000)(1.19,0.0000)(1.46,0.0020)(1.73,0.0072)(2.01,0.0116)(2.28,0.0147)(2.55,0.0175)(2.82,0.0186)(3.10,0.0190)(3.37,0.0190)(3.64,0.0187)(3.91,0.0181)(4.19,0.0174)(4.46,0.0166)(4.73,0.0158)(5.00,0.0150)(5.28,0.0142)(5.55,0.0134)(5.82,0.0127)(6.09,0.0119)(6.37,0.0112)(6.64,0.0105)(6.91,0.0099)(7.18,0.0092)(7.46,0.0087)(7.73,0.0081)(8.00,0.0076)
    };
    \addlegendentry{ Theoretical };
    \end{axis}
  \end{tikzpicture}
  \end{minipage}%
  \hfill{}%
  \begin{minipage}[c]{0.3\linewidth}
  \centering
  \begin{tikzpicture}[font=\footnotesize]
    \renewcommand{\axisdefaulttryminticks}{4} 
    \pgfplotsset{every major grid/.append style={densely dashed}}          
    \tikzstyle{every axis y label}+=[yshift=-10pt] 
    \tikzstyle{every axis x label}+=[yshift=5pt]
    \pgfplotsset{every axis legend/.style={cells={anchor=west},fill=none,at={(-.02,1.02)}, anchor=north west, font=\scriptsize}}
    \begin{axis}[
    width=1.02\linewidth,
    xmin=.1,xmax=8,
    ymin=0,ymax=1,
    ylabel={ Eigen-alignment },
    xlabel={$\| \w \|$},
    grid=major,
    ymajorgrids=false,
    scaled ticks=true,
    ]
    \addplot+[mark=triangle,only marks,BLUE,mark size=2pt,line width=1pt] plot coordinates{
    (0.10,0.001344)(0.37,0.001143)(0.64,0.002564)(0.92,0.007487)(1.19,0.055744)(1.46,0.534765)(1.73,0.759985)(2.01,0.852001)(2.28,0.898468)(2.55,0.923821)(2.82,0.939363)(3.10,0.949487)(3.37,0.958220)(3.64,0.962837)(3.91,0.967790)(4.19,0.971214)(4.46,0.973757)(4.73,0.975565)(5.00,0.977312)(5.28,0.978893)(5.55,0.979770)(5.82,0.980935)(6.09,0.981913)(6.37,0.982530)(6.64,0.983232)(6.91,0.983864)(7.18,0.984328)(7.46,0.985102)(7.73,0.985395)(8.00,0.985530)
    };
    \addplot[densely dashed,RED,line width=1.5pt] plot coordinates{
    (0.10,0.0000)(0.37,0.0000)(0.64,0.0000)(0.92,0.0000)(1.19,0.000)(1.46,0.5376)(1.73,0.7677)(2.01,0.8540)(2.28,0.8981)(2.55,0.9250)(2.82,0.9407)(3.10,0.9509)(3.37,0.9581)(3.64,0.9653)(3.91,0.9690)(4.19,0.9711)(4.46,0.9744)(4.73,0.9772)(5.00,0.9779)(5.28,0.9808)(5.55,0.9822)(5.82,0.9843)(6.09,0.9828)(6.37,0.9841)(6.64,0.9841)(6.91,0.9863)(7.18,0.9864)(7.46,0.9883)(7.73,0.9871)(8.00,0.9872)
    };
    \end{axis}
  \end{tikzpicture}
  \end{minipage}
  \caption{ Spike (on left-hand side of bulk) due to response model in Corollary~\ref{coro:special-logistic-mu=0} \emph{in the absence of data signal}: \textbf{(left)} Hessian eigenspectrum for $\| \w \| = 2$, with an emphasis on the \emph{left} isolated eigenvalue $\hat \lambda_\w$, \textbf{(middle)} eigengap $\dist(\lambda_{\w}, \supp(\mu))$, and \textbf{(right)} dominant eigenvector alignment, as a function $ \| \w \|$ with $\w \propto [-\mathbf{1}_{p/2},~\mathbf{1}_{p/2}]$, $\w_* = \bmu = \zo$, $\C = \I_p$, $p=800$ and $n = 8\,000$. Results averaged over $50$ runs. }
  \label{fig:logistic-spike-w}
  \end{center}
  \vskip -0.1in
\end{figure*}

\subsection{Isolated eigenvalues: Spike due to data signal}
\label{subsec:spike-mu}

To study the spike due to data ``signal'' $\bmu$, as well as its phase transition behavior, we focus here on the case $\w_* = \w = \mathbf{0}$. This, in the case of logistic model \eqref{eq:def-logistic}, gives rise to a much simpler form of limiting eigenspectrum (per Theorem~\ref{theo:lsd}) and possible isolated eigenpairs (per Theorem~\ref{theo:spike}~and~\ref{theo:eigenvector}), as summarized in the following corollary, the proof of which is given in Appendix~\ref{subsec:proof-special-logistic-w=0}.

\begin{Corollary}[Spike due to data signal: logistic model]\label{coro:special-logistic-w=0}
Consider the logistic model in \eqref{eq:def-logistic} with logistic loss, for $\w = \w_* = \mathbf{0}$ and $\C = \I_p$, we have $g = 1/4$ and the limiting Hessian eigenvalue distribution is the Mar\u{c}enko-Pastur law, but rescaled by a factor of $1/4$. 
Moreover, it follows from Theorem~\ref{theo:spike}~and~\ref{theo:eigenvector} that there is \emph{at most one} isolated eigenpair $(\hat \lambda_{\bmu}, \hat \uu_{\bmu})$ of $\H$ 
and it satisfies
\begin{equation}
  \hat \lambda_{\bmu}~\asto~\left\{
        \begin{array}{ll}
            \lambda_{\bmu} = \frac14 ( 1 + \rho + c \cdot \frac{ \rho + 1}{\rho} ) &\rho >\sqrt{c},  \\
            \frac14 (1+\sqrt{c})^2 &\rho \leq\sqrt{c};
        \end{array}
        \right.
\end{equation}
and
\begin{equation}\label{eq:mu-phase-transition-align}
  \frac{ |\bmu^\T \hat \uu_{\bmu}|^2}{\| \bmu \|^2}~\asto~\left\{
        \begin{array}{ll}
            \alpha =\frac{\rho^2 - c}{\rho^2 + c \rho} &\rho >\sqrt{c},  \\
            0 &\rho \leq\sqrt{c};
        \end{array}
        \right. 
\end{equation}
with the signal strength $\rho = \lim \| \bmu \|^2$ and $c = \lim p/n$.
\end{Corollary}

The isolated eigenvalue-eigenvector behavior described in Corollary~\ref{coro:special-logistic-w=0} follows the classical phase transition \cite{baik2006eigenvalues,baik2005phase,paul2007asymptotics}.
In particular, it exhibits the following properties: (i) the isolated eigenvalue always appears on the right-hand side of the main (Mar\u{c}enko-Pastur) bulk; and (ii) the eigenvalue location and associated eigenvector alignment is ``monotonic'' with respect to the signal strength $\| \bmu \|^2$.
By ``monotonic,'' we mean that, for a fixed dimension ratio $c$, the largest Hessian eigenvalue is bound to become `isolated'' once $\| \bmu \|^2$ exceeds $\sqrt c$, and also that its value, as well as the eigenvector alignment, increase monotonically as $\| \bmu \|^2$ grows. 
This behavior is confirmed in Figure~\ref{fig:mu-spikes}. 

As we shall see below in Section~\ref{subsec:spike-w}, this behavior is \emph{not} exhibited for a spike due to the model parameter $\w$.


\subsection{Isolated eigenvalues: Spike due to model structure}
\label{subsec:spike-w}

To study the spike due to the underlying model structure (i.e., $\w_*$ and $\w$), as well as its phase transition behavior, we position ourselves in the situation where $\bmu = \mathbf{0}$, that is, in the absence of any data signal, and we consider again the logistic model. This leads to the following corollary, the proof of which is given in Appendix~\ref{subsec:proof-special-logistic-mu=0}.

\begin{Corollary}[Spike due to model: logistic model]\label{coro:special-logistic-mu=0}
Consider the logistic model in \eqref{eq:def-logistic} with logistic loss, $\bmu = \mathbf{0}$ and $\C = \I_p$, then the Stieltjes transform $m(z)$ satisfies $m(z ) = 1/( \EE[ f(r,z) ] - z )$ for $f(t,z) = 1/( c m(z) + 2 + e^{-t} + e^t )$ and $r \sim \mathcal N(0, \| \w \|^2)$ that \emph{depends} on $\w$ but \emph{not} on $\w_*$. Moreover, there is \emph{at most one} isolated eigenvalue $\hat \lambda_\w$ of $\H$ that is due to $\w$ and satisfies $\hat \lambda_\w - \lambda_\w~\asto~0$ with $\lambda_\w$ solution to $0 = \det \G(\lambda_\w) = 1 + m(\lambda_\w) \frac{ \EE [f(r,\lambda_\w) (r^2 - \| \w\|^2)] }{\| \w \|^2}$.
\end{Corollary}

The situation here (for a spike due to model) is more subtle than for the spike due to data signal (discussed in Section~\ref{subsec:spike-mu}): as $\w$ changes (e.g., as the ``energy'' $\| \w \|$ grows), both the Hessian (limiting) eigenvalue distribution and the possible spike location are impacted.
Figure~\ref{fig:logistic-spike-w} illustrates the behavior of the spike due to $\w$ in the setting of Corollary~\ref{coro:special-logistic-mu=0}. 
First, note that, different from the case of spike due to data signal $\bmu$, the spike in Figure~\ref{fig:logistic-spike-w}-(left) appears on the \emph{left-hand side} of the main bulk.
In particular, this means that the Hessian may admit an eigenvalue that is \emph{significantly smaller} than all the other eigenvalues. 
Second, note from Figure~\ref{fig:logistic-spike-w}-(middle) that, different from the spike due to $\bmu$, the spike due to $\w$ exhibits here a ``non-monotonic'' behavior in the sense that, it is absent for small values of $\| \w\|$ (as for small $\| \bmu \|$ in Figure~\ref{fig:mu-spikes}-middle) and it becomes ``isolated'' as $\| \w\|$ increases, but then again gets closer to the main bulk, as the norm $\| \w\|$ continues to increase, resulting a unimodal eigengap. 

Finally, it is perhaps even more surprising to observe in Figure~\ref{fig:logistic-spike-w}-(right) that, the alignment between the associated isolated eigenvector and the parameter $\w$ is, unlike the eigengap in Figure~\ref{fig:logistic-spike-w}-(middle), \emph{monotonically increasing} as $\| \w \|$ grows large.
This is as in the case of spike due to data signal $\bmu$, illustrated in Figure~\ref{fig:mu-spikes}-(right). This particularly suggests that, in the case of spike due to model, a \emph{smaller} eigengap may not always imply \emph{less} statistical ``information'' contained in the associated eigenvector, which is somehow contrary to the conventional eigengap heuristic \cite{von2007tutorial,joseph2016impact}.

\medskip

We already see that there are spikes due to model on the \emph{left-hand side} of the main bulk (see, e.g., Figure~\ref{fig:logistic-spike-w}-left). One may wonder if it is always the case that the ``model spikes'' are smaller than the main bulk of eigenvalues.  
It turns out that this is \emph{not} true, and these spikes may appear on the right-hand side of the bulk. 

In Figure~\ref{fig:phase-spike-w}, we evaluate the noiseless phase retrieval model $y = (\w_*^\T \x)^2$ with the square loss $\ell(y, h) = (y - h^2)^2/4$. As already mentioned in Remark~\ref{rem:limit-support}, the Hessian eigenvalue distribution of phase retrieval model has unbounded support, due to the unboundedness of the support of the chi-square distribution. As a result, some preprocessing function must be applied for a spike to be observed. Here in Figure~\ref{fig:phase-spike-w} we particularly focus on the preprocessing function
\begin{equation}\label{eq:optimla-f-phase}
	f(t) = \frac{\max(t,0) - 1}{ \max(t,0) +\sqrt{2/c} -1 }
\end{equation}
proposed in \cite{mondelli2018fundamental} in the case of $\w \propto \w_*$, $\bmu = \zo$ and $\C = \I_p$. 
Here, we observe a spike on the \emph{right-hand side} of the main bulk, the ``eigengap'' of which follows again a non-monotonic behavior as a function of the norm $\| \w_* \|$. Perhaps more surprisingly, in contrast to Figure~\ref{fig:logistic-spike-w}-(right) where the eigenvector alignment is monotonically increasing, here we see in Figure~\ref{fig:phase-spike-w}-(right) that the eigenvector alignment follows a similar \emph{non-monotonic} trend as the eigengap.

\begin{figure*}[thb]
  \centering
  \begin{minipage}[c]{0.37\linewidth}
  \centering
  \begin{tikzpicture}[font=\footnotesize,spy using outlines]
    \renewcommand{\axisdefaulttryminticks}{4} 
    \pgfplotsset{every major grid/.append style={densely dashed}}          
    \tikzstyle{every axis y label}+=[yshift=-10pt] 
    \tikzstyle{every axis x label}+=[yshift=5pt]
    \pgfplotsset{every axis legend/.style={cells={anchor=west},fill=none,at={(1.02,1.02)}, anchor=north east, font=\scriptsize}}
    \begin{axis}[
    width=1.02\linewidth,
    xmin=-0.25,
      ymin=0,
      xmax=1,
      ymax=2.5,
      bar width=2pt,
      grid=major,
      ymajorgrids=false,
      xlabel={},
      ylabel={Eigenvalue distribution}
      ]
      \addplot+[ybar,mark=none,draw=white,fill=BLUE!60!white,area legend] coordinates{
      (-0.194765,0.052452)(-0.170934,0.472064)(-0.147102,0.734321)(-0.123270,0.839224)(-0.099439,1.153933)(-0.075607,1.468642)(-0.051776,1.363739)(-0.027944,1.625997)(-0.004113,1.835803)(0.019719,1.993157)(0.043550,1.940706)(0.067382,1.888254)(0.091213,1.993157)(0.115045,1.783351)(0.138876,1.730900)(0.162708,1.730900)(0.186540,1.573545)(0.210371,1.573545)(0.234203,1.416191)(0.258034,1.258836)(0.281866,1.258836)(0.305697,1.206385)(0.329529,1.049030)(0.353360,1.049030)(0.377192,1.153933)(0.401023,0.839224)(0.424855,0.891676)(0.448686,0.839224)(0.472518,0.681870)(0.496350,0.734321)(0.520181,0.576967)(0.544013,0.734321)(0.567844,0.524515)(0.591676,0.524515)(0.615507,0.367161)(0.639339,0.419612)(0.663170,0.314709)(0.687002,0.209806)(0.710833,0.104903)(0.734665,0.000000)(0.758496,0.000000)(0.782328,0.000000)(0.806160,0.000000)(0.829991,0.000000)(0.853823,0.000000)(0.877654,0.052452)(0.901486,0.000000)(0.925317,0.000000)(0.949149,0.000000)(0.972980,0.000000)
      };
      \addplot[densely dashed,RED,line width=1.5pt] coordinates{
      (-0.206681,0.000110)(-0.194885,0.000240)(-0.183090,0.290251)(-0.171295,0.468273)(-0.159499,0.603453)(-0.147704,0.720873)(-0.135908,0.828594)(-0.124113,0.930390)(-0.112318,1.028332)(-0.100522,1.123623)(-0.088727,1.216928)(-0.076931,1.308510)(-0.065136,1.398276)(-0.053341,1.485776)(-0.041545,1.570190)(-0.029750,1.650306)(-0.017954,1.724552)(-0.006159,1.791077)(0.005636,1.847951)(0.017432,1.893438)(0.029227,1.926314)(0.041023,1.946109)(0.052818,1.953187)(0.064614,1.948636)(0.076409,1.934027)(0.088204,1.911143)(0.100000,1.881748)(0.111795,1.847446)(0.123591,1.809606)(0.135386,1.769348)(0.147181,1.727560)(0.158977,1.684927)(0.170772,1.641970)(0.182568,1.599076)(0.194363,1.556527)(0.206158,1.514525)(0.217954,1.473210)(0.229749,1.432675)(0.241545,1.392978)(0.253340,1.354150)(0.265135,1.316200)(0.276931,1.279127)(0.288726,1.242917)(0.300522,1.207545)(0.312317,1.172986)(0.324112,1.139207)(0.335908,1.106173)(0.347703,1.073849)(0.359499,1.042196)(0.371294,1.011177)(0.383090,0.980753)(0.394885,0.950884)(0.406680,0.921533)(0.418476,0.892660)(0.430271,0.864225)(0.442067,0.836190)(0.453862,0.808515)(0.465657,0.781158)(0.477453,0.754079)(0.489248,0.727235)(0.501044,0.700581)(0.512839,0.674069)(0.524634,0.647651)(0.536430,0.621271)(0.548225,0.594871)(0.560021,0.568384)(0.571816,0.541736)(0.583611,0.514840)(0.595407,0.487595)(0.607202,0.459878)(0.618998,0.431538)(0.630793,0.402384)(0.642589,0.372161)(0.654384,0.340523)(0.666179,0.306966)(0.677975,0.270705)(0.689770,0.230392)(0.701566,0.183279)(0.713361,0.121392)(0.725156,0.000129)(0.736952,0.000049)(0.748747,0.000035)(0.760543,0.000028)(0.772338,0.000024)(0.784133,0.000021)
      };
      \addplot[smooth,GREEN,line width=1pt] coordinates{ (0.7252,0)(0.7252,10)};
      \addplot+[only marks,mark=x,RED] coordinates{ (0.8745,0) };
      \coordinate (spy0) at (axis cs:0.8745,0.001); \coordinate (spypoint0) at (axis cs:0.9,1);
    \begin{scope}
      \spy[black!50!white,size=.6cm,circle,connect spies,magnification=3] on (spy0) in node [fill=none] at (spypoint0);
      \end{scope}
    \end{axis}
  \end{tikzpicture}
  \end{minipage}%
  \hfill{}%
  \begin{minipage}[c]{0.3\linewidth}
  \centering
  	\begin{tikzpicture}[font=\footnotesize]
    \renewcommand{\axisdefaulttryminticks}{4} 
    \pgfplotsset{every major grid/.append style={densely dashed}}          
    \tikzstyle{every axis y label}+=[yshift=-10pt] 
    \tikzstyle{every axis x label}+=[yshift=5pt]
    \pgfplotsset{every axis legend/.style={cells={anchor=east},fill=none,at={(1.02,0.02)}, anchor=south east, font=\scriptsize}}
    \begin{axis}[
    width=1.02\linewidth,
    xmin=.1,xmax=2,
    ymin=0,ymax=0.2,
    ylabel={ Eigengap },
    xlabel={$\| \w_* \| $},
    grid=major,
    ymajorgrids=false,
    yticklabel style={
    /pgf/number format/fixed,
    /pgf/number format/precision=2},
	scaled y ticks=false,
    ]
    \addplot+[mark=triangle,only marks,BLUE,mark size=2pt,line width=1pt] plot coordinates{
    (0.10,0.002151)(0.17,0.002199)(0.23,0.002019)(0.30,0.002263)(0.36,0.001602)(0.43,0.002738)(0.49,0.015680)(0.56,0.040560)(0.62,0.069317)(0.69,0.091760)(0.76,0.114021)(0.82,0.133266)(0.89,0.147526)(0.95,0.157390)(1.02,0.171163)(1.08,0.178426)(1.15,0.181567)(1.21,0.187488)(1.28,0.187938)(1.34,0.180424)(1.41,0.185382)(1.48,0.185732)(1.54,0.179268)(1.61,0.178528)(1.67,0.176906)(1.74,0.175813)(1.80,0.171528)(1.87,0.165638)(1.93,0.158157)(2.00,0.156341)
    };
    \addplot[densely dashed,RED,line width=1.5pt] plot coordinates{
    (0.10,0.002426)(0.17,0.002102)(0.23,0.002109)(0.30,0.002076)(0.36,0.001705)(0.43,0.002244)(0.49,0.015516)(0.56,0.041311)(0.62,0.067072)(0.69,0.093239)(0.76,0.114311)(0.82,0.133712)(0.89,0.145974)(0.95,0.156250)(1.02,0.167544)(1.08,0.175640)(1.15,0.178167)(1.21,0.182524)(1.28,0.185941)(1.34,0.182625)(1.41,0.181878)(1.48,0.179262)(1.54,0.180072)(1.61,0.178986)(1.67,0.175897)(1.74,0.171855)(1.80,0.167286)(1.87,0.166734)(1.93,0.161310)(2.00,0.153040)
    };
    \end{axis}
  \end{tikzpicture}
  \end{minipage}%
  \hfill{}%
  \begin{minipage}[c]{0.3\linewidth}
  \centering
  \begin{tikzpicture}[font=\footnotesize]
    \renewcommand{\axisdefaulttryminticks}{4} 
    \pgfplotsset{every major grid/.append style={densely dashed}}          
    \tikzstyle{every axis y label}+=[yshift=-10pt] 
    \tikzstyle{every axis x label}+=[yshift=5pt]
    \pgfplotsset{every axis legend/.style={cells={anchor=west},fill=none,at={(-.02,1.02)}, anchor=north west, font=\scriptsize}}
    \begin{axis}[
    width=1.02\linewidth,
    xmin=.1,xmax=2,
    ymin=0,ymax=.8,
    ylabel={ Eigen-alignment },
    xlabel={$\| \w_* \|$},
    grid=major,
    ymajorgrids=false,
    scaled ticks=true,
    ]
    \addplot+[mark=triangle,only marks,BLUE,mark size=2pt,line width=1pt] plot coordinates{
    (0.10,0.000825)(0.17,0.001584)(0.23,0.001429)(0.30,0.003044)(0.36,0.005254)(0.43,0.082656)(0.49,0.513597)(0.56,0.683068)(0.62,0.741886)(0.69,0.767674)(0.76,0.774038)(0.82,0.770126)(0.89,0.765693)(0.95,0.768827)(1.02,0.756879)(1.08,0.744152)(1.15,0.736277)(1.21,0.726607)(1.28,0.719440)(1.34,0.711806)(1.41,0.701065)(1.48,0.690804)(1.54,0.680563)(1.61,0.675872)(1.67,0.666140)(1.74,0.654435)(1.80,0.644350)(1.87,0.637784)(1.93,0.628896)(2.00,0.612980)
    };
    \addplot[densely dashed,RED,line width=1.5pt] plot coordinates{
    (0.10,0.001758)(0.17,0.001655)(0.23,0.001663)(0.30,0.001770)(0.36,0.002758)(0.43,0.080027)(0.49,0.522082)(0.56,0.688529)(0.62,0.746706)(0.69,0.764318)(0.76,0.773721)(0.82,0.773233)(0.89,0.762425)(0.95,0.759114)(1.02,0.755574)(1.08,0.742741)(1.15,0.738802)(1.21,0.729065)(1.28,0.718755)(1.34,0.708856)(1.41,0.701384)(1.48,0.695142)(1.54,0.687205)(1.61,0.676982)(1.67,0.666415)(1.74,0.654791)(1.80,0.647210)(1.87,0.641582)(1.93,0.627013)(2.00,0.616557)
    };
    \end{axis}
  \end{tikzpicture}
  \end{minipage}
  \caption{ Spike (on right-hand side of bulk) due to response model \emph{in the absence of data signal}: \textbf{(left)} Hessian eigenspectrum for $\| \w_* \| = 1$, with an emphasis on the \emph{right} isolated eigenvalue $\hat \lambda_{\w_*}$, \textbf{(middle)} eigengap, and \textbf{(right)} dominant eigenvector alignment, as a function $ \| \w_* \|$ with $\w_* \propto [-\mathbf{1}_{p/2},~\mathbf{1}_{p/2}]$, $\w = \sqrt{2/3} \cdot \w_*$, $\bmu = \zo$, $\C = \I_p$, $p=800$ and $n = 4\,000$. Results averaged over $50$ runs. }
  \label{fig:phase-spike-w}
\end{figure*}

As a side remark, note in Figure~\ref{fig:phase-spike-w} that negative eigenvalues are observed for $\w \propto \w_*$: this is not surprising since the modified matrix $\H$ with the preprocessing function $f(t)$ in \eqref{eq:optimla-f-phase} \emph{no longer} corresponds to the Hessian of the original optimization problem.

\section{Conclusions}\label{sec:conclusion}

In this paper, we precisely characterized the Hessian eigenspectra of the nonlinear G-GLMs, in the high dimensional regime where both feature dimension and sample size are large and comparable. 
By focusing on this more realistic, yet still tractable, family of models, our results provide theoretical justifications for many empirical observations in large-scale machine learning models such as deep neural networks.
The deterministic equivalent technique that we employ is very powerful, and we expect that it will enable analysis of a range of more realistic machine learning models.
Of particular future interest would be to use this technique to understand heavy-tailed structure in the correlations of state-of-the-art-neural network models \cite{MM18_TR,MM20a_trends_TR}, and the connections with recent results on heavy-tailed behavior in stochastic optimization algorithms that are used to train these models \cite{hodgkinson2020multiplicative,gurbuzbalaban2020heavy}.

\subsubsection*{Acknowledgments}
We would like to acknowledge DARPA, IARPA (contract W911NF20C0035), NSF, and ONR via its BRC on RandNLA for providing partial support of this work. Our conclusions do not necessarily reflect the position or the policy of our sponsors, and no official endorsement should be inferred.

\printbibliography

\appendix

\section{Proofs}
\label{sec:proofs}

In this section, we present the proofs of our main technical results, Theorem~\ref{theo:lsd}-\ref{theo:DE}. We will start with the key technical Theorem~\ref{theo:DE}, which, as discussed in Section~\ref{subsec:pre-DE}, will serve as the cornerstone in proving Theorem~\ref{theo:lsd},~\ref{theo:spike}~and~\ref{theo:eigenvector}.

In the sequel, the big $O$ notation \(O(u_n)\) is understood in an almost sure sense when the objects are random. When multidimensional objects are concerned, $O(u_n)$ is understood entry-wise and \(O_{\| \cdot \|}(\cdot)\) is understood, for a vector $\vv = O_{\| \cdot \|}(u_n)$, as its Euclidean norm satisfying $\| \vv \| = O(u_n)$ and for a square matrix $\M$, as its spectral norm being of the order $O(u_n)$. The small $o$ notation is understood likewise. Note that under Assumption~\ref{ass:high-dimen} it is equivalent to using either $O(u_n)$ or $O(u_p)$ as $n, p$ scales linearly, and we will use constantly $O(u_n)$ for simplicity of exposition.

\medskip

\subsection{Proof of Theorem~\ref{theo:DE}}
\label{sec:SM-proof-theo-DE}

The proof of the deterministic equivalent (recall from Section~\ref{subsec:pre-DE}) result of the type $\Q(z) \leftrightarrow \EE[\Q(z)]$ generally comes in two steps: 
\begin{enumerate}
  \item establish the convergence in mean, that is, to show that the proposed $\bar \Q(z)$ is an asymptotic approximation for the expected resolvent $\EE[\Q(z)]$ in the sense that
	\begin{equation}
		\| \bar \Q(z) - \EE[\Q(z)] \| \to 0,
	\end{equation}
	as $n,p \to \infty$; and
	\item show that the random quantities $\frac1p \tr \A \Q(z)$ and $\mathbf{a}^\T \Q(z) \mathbf{b}$ of interest concentrate around their expectations in the sense that
	\begin{equation}\label{eq:SM-DE-concentration}
		\frac1p \tr \A (\Q(z) - \EE \Q(z)) \to 0, \quad\text{~and~}\mathbf{a}^\T (\Q(z) - \EE \Q(z)) \mathbf{b} \to 0,
	\end{equation}
	almost surely as $n,p \to \infty$. This can be done with concentration of measure arguments \cite{ledoux2001concentration} or, in the Gaussian case, the Nash--Poincar\'e inequality to bound the variance (see for example \cite[Proposition~2.4]{pastur2005simple}) and then applying Markov's inequality and the Borel Cantelli lemma.
\end{enumerate}

Here we will focus on the first part of the proof and show for $\Q(z) = (\H - z \I_p)^{-1}$, $\H \in \RR^{p \times p}$ defined in \eqref{eq:def-H} that
\begin{equation}\label{eq:SM-barQ}
	\| \EE[\Q(z)] - \bar \Q(z) \| \to 0, \quad \bar \Q(z) \equiv \left(\EE \left[\frac{g}{1+g \cdot \delta(z)} ( \C^{\frac12} (\I_p - \P_\U) \C^{\frac12} +  \balpha \balpha^\T) \right] - z \I_p \right)^{-1},
\end{equation}
as $n,p \to \infty$, for random vector $\balpha \equiv \bmu + \C^{\frac12} \P_\U \z \in \RR^p$ and $g = \ell''(y, \w^\T \bmu + \w^\T \C^{\frac12} \z)$ for $\z \sim \NN(\zo, \I_p)$ as in \eqref{eq:def-g}, $y$ defined in \eqref{eq:def-GLM}, $\P_\U$ the projection onto the subspace spanned by the columns of $\U = \C^{\frac12} [\w_*,~\w] \in \RR^{p \times 2}$ so that $\P_\U = \U (\U^\T \U)^{-1} \U^\T$ for $\U^\T \U$ invertible, and $\delta(z)$ the unique solution to \eqref{eq:def-m(z)}.


\medskip
Before diving into the detailed proof of Theorem~\ref{theo:DE}, we will illustrate the main idea leading to Theorem~\ref{theo:DE}, notably the form of $\bar \Q$ in \eqref{eq:SM-barQ}, in the simple setting of $\bmu = \zo$ and $\C = \I_p$. As discussed above in Section~\ref{subsubsec:related}, the main technical difficulty lies in the fact that the diagonal matrix $\D = \diag \{ \ell''(y_i, \w^\T \x_i) \}_{i=1}^n$ depends on the $\x_i$'s in a rather involved nonlinear manner per \eqref{eq:def-GLM}. Nonetheless, note that standard Gaussian vector $\x \sim \NN(\zo, \I_p)$ can decomposed as
\begin{equation}\label{eq:intuition-decomp}
  \x = \U (\U^\T \U)^{-1} \U^\T \x + \x^\perp \equiv \P_\U \x + \x^\perp, \quad \text{for} \quad \x^\perp \equiv (\I_p - \P_\U) \x, 
\end{equation}
the sum of its projection onto the subspace spanned by the columns of $\U = [\w_*,~\w]$ (recall that $\C = \I_p$ and assume for the moment that $\w, \w_*$ are linearly independent) and the complementary subspace. Due to the orthogonal invariance of standard multivariate Gaussian, we have that both $\P_\U \x$ and $\x^\perp = (\I_p - \P_\U) \x$ are Gaussian, and they are \emph{independent}, for example by checking that $\EE[\P_\U \x \x^\T (\I_p - \P_\U)] = \zo$.

Since $\D$ depends on $\X$ only through $\U^\T \X$, we have that $\D$ is \emph{independent} of $\X^\perp \equiv [\x_1^\perp, \ldots, \x_n^\perp] = (\I_p - \P_\U) \X \in \RR^{p \times n}$ so that
\begin{equation}
	\H = \frac1n \X \D \X^\T = \frac1n (\P_\U \X + \X^\perp) \D (\P_\U \X + \X^\perp)^\T = \frac1n \X^\perp \D (\X^\perp)^\T + \A
\end{equation}
for some $\A \in \RR^{p \times p}$ of low rank (since $\rank(\P_\U) \le 2$). As a consequence, our object of interest $\H$ is shown to follow a spiked random matrix model as discussed in Section~\ref{sec:pre-spiked}, with the main bulk $\frac1n \X^\perp \D (\X^\perp)^\T$ now more accessible due to the independence between $\X^\perp$ and $\D$. 
Moreover, since $\X^\perp$ is also a low rank update of the original $\X$, the above derivation shows that, asymptotically, the limiting spectral measure of $\H = \frac1n \X \D \X^\T$ in \eqref{eq:def-H} is the same as that of $\frac1n \X^\perp \D (\X^\perp)^\T$ and thus the same as if $\D$ was \emph{independent} of $\X$. 
While the analysis of limiting eigenvalue distribution of $\H$ is largely simplified based on the argument above, the characterization of the low rank matrix $\A$ requires more effort, so as to determine the behavior of the possible isolated eigenvalue-eigenvector pairs of $\H$.

\medskip

The idea above can be generalized, under Assumption~\ref{ass:high-dimen}, to arbitrary profile of the statistical mean $\bmu $ and covariance $\C$, leading to Theorem~\ref{theo:DE}, the proof of which is given as follows.
\begin{proof}[Proof of Theorem~\ref{theo:DE}]
For mathematical convenience, for the resolvent $\Q(z) = (\H - z \I_p)^{-1}$, we will take $z<0$ outside the spectral support of $\H$ in what follows. Since $\Q(z)$ and $\bar\Q(z)$ from the theorem statement are complex analytic functions for $z\notin \RR^+$ (matrix-valued Stieltjes transforms are analytic \cite{hachem2007deterministic}), obtaining the convergence results on $\RR^-$ is equivalent to obtaining the result on all of $\CC\setminus\RR^+$.

Following the orthogonal decomposition idea in \eqref{eq:intuition-decomp}, for $\U = \C^{\frac12} [\w_*,~\w] \in \RR^{p \times 2}$ and $\x_i \sim \NN(\bmu, \C)$ we have $\z_i = \C^{-\frac12} (\x_i - \bmu) \sim \NN(\zo, \I_p)$ and
\begin{equation}
  \z_i = \P_\U \z_i + \z_i^\perp, \quad \text{~for~}\z_i^\perp \equiv (\I_p - \P_\U) \z_i.
\end{equation}
Note that (i) $\U^\T \z_i^\perp = \zo$, (ii) $\U^\T \z_i$ is independent of $\z_i^\perp$ due to the Gaussianity of $\z_i$ and
\begin{equation}\label{eq:cov-z-perp}
  \EE[\z_i^\perp] = \zo, \quad \EE[\z_i^\perp (\z_i^\perp)^\T] = \I_p - \P_\U.
\end{equation}
With these results, we now move on to the proof of \eqref{eq:SM-barQ}. Recall the definition $\Q = (\frac1n \X \D \X^\T - z \I_p)^{-1}$ with $\D = \diag\{ g_i\}_{i=1}^n$ and
\begin{equation}\label{eq:SM-def-gi}
	g_i \equiv \ell''(y_i, \w^\T \x_i),
\end{equation}
we first introduce the following intermediate deterministic quantities
\begin{equation}\label{eq:def-theta}
	\theta \equiv \frac1n \tr \C\EE[\Q_{-1}] > 0,\quad \M_\theta \equiv \EE \left[ \frac{g_1 \cdot \x_1 \x_1}{1 + g_1 \cdot \theta} \right], \quad \bar \Q_\theta \equiv \left( \M_\theta - z \I_p \right)^{-1}  ,
\end{equation}
where we denote, for i.i.d.~$\x_i,g_i$ and without loss of generality, $\Q_{-1} \equiv ( \frac1n \sum_{i \neq 1} g_i \x_i \x_i^\T - z\I_p )^{-1} \in \RR^{p \times p}$ the ``leave-one-out'' version of $\Q$ that is \emph{independent} of $\x_1$ and $g_1$. Our proof strategy is to show that
\[
	\| \EE [\Q] - \bar \Q_\theta \| \to 0, \quad \| \bar \Q_\theta - \bar \Q \| \to 0,
\]
and thus the conclusion of \eqref{eq:SM-barQ}.

\medskip

With the resolvent identity $\A^{-1} - \B^{-1} = \A^{-1} (\B - \A) \B^{-1}$ for invertible $\A, \B$, we write, 
\begin{align*}
  \EE[\Q - \bar \Q_\theta] &= \EE \left[ \Q \left( \M_\theta - \frac1n \X \D \X^\T \right) \right] \bar \Q_\theta = \EE[\Q] \M_\theta \bar \Q_\theta - \frac1n \sum_{i=1}^n \EE[ g_i \Q \x_i \x_i^\T] \bar \Q_\theta \\ 
  &= \EE[\Q] \M_\theta \bar \Q - \frac1n \sum_{i=1}^n \EE \left[ \frac{g_i \cdot \Q_{-i} \x_i \x_i^\T}{ 1 + g_i \cdot \frac1n \x_i^\T \Q_{-i} \x_i } \right] \bar \Q_\theta,
\end{align*}
where we used the Sherman–Morrison formula.

Since the random variable $\frac1n \x_i^\T \Q_{-i} \x_i$ is expected to concentrate around its expectation $\theta = \frac1n \tr \C \EE[\Q_{-i}]$, we write
\begin{align*}
	\frac1n \sum_{i=1}^n \EE \left[ \frac{g_i \cdot \Q_{-i} \x_i \x_i^\T}{ 1 + g_i \cdot \frac1n \x_i^\T \Q_{-i} \x_i } \right] &= \frac1n \sum_{i=1}^n \EE \left[ \frac{g_i \cdot \Q_{-i} \x_i \x_i^\T}{ 1 + g_i \cdot \theta } \right] + \frac1n \sum_{i=1}^n \EE \left[ \frac{g_i^2 \cdot \Q_{-i} \x_i \x_i^\T (\theta - \frac1n \x_i^\T \Q_{-i} \x_i)}{ (1 + g_i \cdot \frac1n \x_i^\T \Q_{-i} \x_i) (1 + g_i \cdot \theta) } \right] \\ 
	&= \frac1n \sum_{i=1}^n \EE[\Q_{-i}] \M_\theta + \frac1n \sum_{i=1}^n \EE \left[ \frac{ g_i^2 \cdot \Q \x_i \x_i^\T (\theta - \frac1n \x_i^\T \Q_{-i} \x_i) }{1 + g_i \cdot \theta} \right] \\ 
	&\equiv \EE[\Q_{-1}] \M_\theta + \frac1n \EE[ \Q \X \D \D_1 \X^\T]
\end{align*}
for $\D_1 \equiv \diag\{ \frac{g_i}{1 + g_i \cdot \theta} (\theta - \frac1n \x_i^\T \Q_{-i} \x_i) \}_{i=1}^n$ and $\D = \diag\{ g_i \}_{i=1}^n$, where we used the fact that $\Q_{-i}$ is independent of both $g_i$  and $\x_i$. As such, 
\begin{equation}\label{eq:proof-1}
	\EE[\Q - \bar \Q_\theta] = \EE[\Q - \Q_{-1}] \M_\theta \bar \Q_\theta - \frac1n \EE[ \Q \X \D \D_1 \X^\T] \bar \Q_\theta.
\end{equation}

For the first right-hand side term in \eqref{eq:proof-1}, again by the resolvent identity
\begin{align*}
	\EE[\Q_{-1} - \Q] &= \frac1n \sum_{i=1}^n \EE \left[\Q_{-i} - \Q \right] = \frac1n \sum_{i=1}^n \EE \left[ \Q_{-i} \frac1n g_i \x_i \x_i^\T \Q \right] = \frac1n \sum_{i=1}^n \EE \left[  \frac1n \frac{g_i \cdot \Q_{-i} \x_i \x_i^\T \Q_{-i}}{1 + g_i \cdot \frac1n \x_i^\T \Q_{-i} \x_i} \right] \\ 
	&= \frac1n \EE \left[ \Q_{-1} \frac{g_1 \x_1 \x_1^\T}{1 + g_i \cdot \theta} \Q_{-1} \right] + \frac1n \EE \left[ \Q_{-1} \frac{g_1^2 \x_1 \x_1^\T (\theta - \frac1n \x_1^\T \Q_{-1} \x_1)}{ (1 + g_1 \cdot \theta) (1 + g_1 \cdot \frac1n \x_1^\T \Q_{-1} \x_1)} \Q_{-1} \right] \\ 
	& = \frac1n \EE[\Q_{-1} \M_\theta \Q_{-1}] + \frac1n \EE \left[ \Q \frac1n \X \D \D_1 \X^\T \Q_{-1} \right]
\end{align*}
where we used again the Sherman–Morrison formula. The first term is easily seen to be of order $O(n^{-1})$ in spectral norm (with $\| \Q(z) \| \leq \dist(z, \supp \mu_\H)^{-1}$ and $\| \M_\theta \| = O(1)$ under Assumption~\ref{ass:high-dimen}). For the second term, note that with the additional $1/n$ in front, it suffices to bound the spectral norm of the second right-hand side term in \eqref{eq:proof-1}. Since $\| \Q \frac1n \X \D \X^\T \| \leq 1$, it thus remains to control
\begin{equation}
	\EE \left[ \max_i \left| \theta - \frac1n \x_i^\T \Q_{-i} \x_i \right| \right]  ,
\end{equation}
that appears in $\D_1$. To this end, note that
\begin{align*}
	&\EE \left[ \left| \theta - \frac1n \x_1^\T \Q_{-1} \x_1 \right|^4 \right] \leq \EE \left[ \left| \frac1n \x_1^\T \Q_{-1} \x_1 - \frac1n \tr \C\Q_{-1} + \frac1n \tr \C(\Q_{-1} - \EE[\Q_{-1}] ) \right|^4 \right] \\ 
	&= \EE \left[ \left|\frac1n \x_1^\T \Q_{-1} \x_1 - \frac1n \tr \C\Q_{-1} \right|^4 \right] + \EE \left[ \left| \frac1n \tr \C(\Q_{-1} - \EE[\Q_{-1}]) \right|^4 \right] \\ 
	&+ 6 \cdot \EE \left[ \left|\frac1n \x_1^\T \Q_{-1} \x_1 - \frac1n \tr \C\Q_{-1} \right|^2 \cdot \left| \frac1n \tr \C(\Q_{-1} - \EE[\Q_{-1}]) \right|^2 \right] + O(n^{-4}) =  O(n^{-2})  ,
\end{align*}
with $\alpha \equiv \frac1n \tr \C \EE[\Q_{-1}]$ by definition, where we used, for the last line, \cite[Lemma~B.26]{bai2010spectral}, the Burkholder inequality (see e.g.,~\cite[Lemma~2.13]{bai2010spectral}), as well as the fact that $\bmu^\T \Q_{-i} \bmu \le \| \bmu \|^2 \| \Q_{-i} \| = O(1)$ under Assumption~\ref{ass:high-dimen}.

As a result, by Markov's inequality $\mathbb P(X>t)\leq \EE[X^k]/t^k$ for every $k$ (for $X,t>0$) and the union bound, we have that 
\begin{align*}
	\EE \left[ \max_i \left| \theta - \frac1n \x_i^\T \Q_{-i} \x_i \right| \right] &= \int_0^{C' n^{-1/4}} \mathbb P \left( \max_i \left| \theta - \frac1n \x_i^\T \Q_{-i} \x_i \right| >t\right)\,dt \\ 
	&+ n \int_{C' n^{-1/4}}^\infty \mathbb P \left( \left| \theta - \frac1n \x_1^\T \Q_{-1} \x_1 \right| >t\right)\,dt \leq C n^{-\frac14} 
\end{align*}
holds for some constant $C, C' > 0$ that depend on $p,n$ only via the ratio $p/n$.

Putting all together, we obtain the first approximation
\begin{equation}
	\| \EE[\Q] - \bar \Q_\theta \| = O(n^{-\frac14}), \quad\text{~for~}\bar \Q_\theta \equiv \left( \M_\theta - z \I_p \right)^{-1},\quad \M_\theta \equiv \EE \left[ \frac{g_1 \cdot \x_1 \x_1}{1 + g_1 \cdot \theta} \right] ,
\end{equation}
with $\theta \equiv \frac1n \tr \C\EE[\Q_{-1}] > 0$.

Note in passing we also proved that $\| \EE[\Q - \Q_{-1}] \| = O(n^{-1})$.

It thus remains to establish the second approximation by showing
\begin{equation}
	\| \bar \Q_\theta - \bar \Q \| \to 0,
\end{equation}
as $n,p \to \infty$, with $\bar \Q = (\M - z \I_p)^{-1}$ and $\M = \EE \left[\frac{g}{1+g \cdot \delta(z)} ( \C^{\frac12} (\I_p - \P_\U) \C^{\frac12} +  \balpha \balpha^\T) \right]$ as defined in \eqref{eq:SM-barQ}. To this end, we first compute the expectation in $\M_\theta$. Note that $g_i = \ell''(y_i, \w^\T \x_i)$ depends on $\x_i$ \emph{only} via the projection $\U^\T \z_i$ under \eqref{eq:def-GLM}. We thus decompose $\x_i \x_i^\T$ as
\begin{align*}
  \x_i \x_i^\T &= (\bmu + \C^{\frac12} \z_i) (\bmu + \C^{\frac12} \z_i)^\T = (\bmu + \C^{\frac12} \P_{\U} \z_i + \C^{\frac12} \z_i^\perp ) (\bmu + \C^{\frac12} \P_{\U} \z_i + \C^{\frac12} \z_i^\perp )^\T \\ 
  &= \balpha_i \balpha_i^\T + \C^{\frac12} \z_i^\perp (\z_i^\perp)^\T \C^{\frac12} + \balpha_i (\z_i^\perp)^\T \C^{\frac12} + \C^{\frac12} \z_i^\perp \balpha_i^\T,
\end{align*}
where we denote $\balpha_i \equiv \bmu + \C^{\frac12} \P_{\U} \z_i$. As such
\begin{align*}
	\M_\theta &\equiv \EE \left[ \frac{g_1 \cdot \x_1 \x_1}{1 + g_1 \cdot \theta} \right] = \EE \left[ \frac{g_1 \cdot ( \balpha_1 \balpha_1^\T + \C^{\frac12} \z_i^\perp (\z_i^\perp)^\T \C^{\frac12} )}{1 + g_1 \cdot \theta} \right] \\ 
	&= \EE \left[ \frac{g_1}{1 + g_1 \cdot \theta} \left( \balpha_1 \balpha_1^\T + \C^{\frac12} (\I_p - \P_\U) \C^{\frac12} \right) \right]
\end{align*}
where we used the crucial fact that $\z_i^\perp$ is independent of $g_i, \P_\U \z_i $ and $ \balpha_i$ due to the Gaussianity of $\z_i$, with $\EE[\z_i^\perp] = \zo$ and $\EE[\z_i^\perp (\z_i^\perp)^\T] = \I_p - \P_\U$ from \eqref{eq:cov-z-perp}. As a result, we have
\begin{equation}\label{eq:SM-diff-M}
	\M - \M_\theta = \EE \left[ \frac{g^2 (\delta - \theta)}{(1+g \cdot \delta) (1 + g \cdot \theta)} \cdot ( \C^{\frac12} (\I_p - \P_\U) \C^{\frac12} +  \balpha \balpha^\T) \right]
\end{equation}
for $\delta = \delta(z) = \frac1p \tr \bar \Q_b(z)$ as defined in \eqref{eq:def-m(z)}.

Further note that
\[
	\bar \Q_\theta - \bar \Q = \bar \Q_\theta \left( \M - \M_\theta \right) \bar \Q 
\]
so that with $\bar \Q_b(z) = \left( \EE \left[ \frac{g \cdot \C}{1 + g \delta(z)} \right] - z \I-p \right)^{-1}$ as defined in \eqref{eq:def-m(z)},
\begin{align*}
	|\delta - \theta| = \left| \frac1n \tr \C \left( \bar \Q_b - \EE[\Q_{-1}] \right) \right| &\leq \left| \frac1n \tr \C ( \bar \Q_b - \bar \Q) \right| + \left| \frac1n \tr \C( \bar \Q - \bar \Q_\theta) \right| \\ 
	&+ \left| \frac1n \tr \C( \bar \Q_\theta - \EE[\Q] ) \right| + \left| \frac1n \tr \C \EE[\Q - \Q_{-1}]  \right|
\end{align*}
where we used the fact that $\| \EE[\Q] - \bar \Q_\theta \| = O(n^{-\frac14})$ and $\| \EE[\Q - \Q_{-1}] \| = O(n^{-1})$ to bound the third and fourth right-hand side term, respectively, and the following lemma to bound the first right-hand side term.
\begin{Lemma}[{\cite[Lemma~2.6]{silverstein1995empirical}}]\label{lem:rank-one-update-trace}
For $\A, \M \in \RR^{n \times n}$ symmetric, $\uu \in \RR^n$, $\tau \in \RR$ and $z \in \CC \setminus \RR$, 
\[
  \tr \A (\M + \tau \uu \uu^\T - z \I_n)^{-1} - \tr \A (\M - z \I_n)^{-1} \le \frac{\| \A \|}{ |\Im[z]| }.
\]
The inequality further extends to the case of symmetric and positive semidefinite $\A$, with $\tau > 0$ and $z<0$ so long that the inverses are well-defined, by replacing the right-hand-side $|\Im[z]|$ by $|z|$.
\end{Lemma}

We thus obtain from \eqref{eq:SM-diff-M} that
\begin{align*}
	|\delta - \theta| &\leq \left| \frac1n \tr \C( \bar \Q - \bar \Q_\theta) \right| + O(n^{-\frac14}) \\ 
	&= |\delta - \theta| \cdot \frac1n \tr \left( \C \bar \Q_\theta \EE \left[  \frac{g^2}{ (1 + g \cdot \delta) (1 + g \cdot \delta) } ( \C^{\frac12} (\I_p - \P_\U) \C^{\frac12} +  \balpha \balpha^\T) \bar \Q \right] \right) + O(n^{-\frac14})
\end{align*}
so that
\begin{equation}
	|\delta - \theta| = O(n^{-\frac14}), \quad\text{and}\quad \| \bar \Q_\theta - \bar \Q \| = O(n^{-\frac14}).
\end{equation}
This concludes the proof of Theorem~\ref{theo:DE}.
\end{proof}
Note in passing that we can alternatively write
\begin{equation}
	\bar \Q =  \left( \EE \left[\frac{g \cdot \C}{1+g \cdot \delta(z)} \right] - z \I_p + \V \bLambda \V^\T \right)^{-1}
\end{equation}
with $\V \equiv [\bmu,~\C^{\frac12} \U] \in \RR^{p \times 3}$, $\U = \C^{\frac12}[\w_*,~\w]$ and 
\begin{equation}
  \bLambda(z) \equiv \EE \left[\frac{g}{1+g \cdot \delta(z)} \begin{bmatrix} 1 & (\U^+ \z)^\T \\ \U^+ \z & \U^+ \z (\U^+ \z)^\T - (\U^\T \U)^+ \end{bmatrix} \right]
\end{equation}
for $\U^+$ the Moore–Penrose pseudoinverse of $\U$ as defined in Theorem~\ref{theo:spike}.

\subsection{Proof of Theorem~\ref{theo:lsd}}\label{subsec:proof-theo-lsd}

The proof of Theorem~\ref{theo:lsd} follows straightforwardly from Theorem~\ref{theo:DE} by writing the corresponding Stieltjes transform
\begin{align*}
	\frac1p \tr \Q(z) &= \frac1p \tr \left( \Q(z) - \EE[\Q(z)] + \EE[\Q(z)] - \bar \Q(z) + \bar \Q(z) - \bar \Q_b(z) + \bar \Q_b(z) \right) \\ 
	&= O(n^{-\frac14}) + O(n^{-1}) + \frac1p \tr \bar \Q_b(z)
\end{align*}
where we used the fact that $\| \EE[\Q] - \bar \Q \| = O (n^{-\frac14})$ from the proof of Theorem~\ref{theo:DE} above and Lemma~\ref{lem:rank-one-update-trace} for the approximations, respectively.
\qed

\subsection{Proof of Theorem~\ref{theo:spike}}\label{subsec:proof-theo-spike}

For $\bar \Q (z)$ defined in Theorem~\ref{theo:DE}, we have, with the Woodbury identity that
\[
  \bar \Q(z) = \left( \EE \left[\frac{g \cdot \C}{1+g \cdot \delta(z)} \right] - z \I_p + \V \bLambda \V^\T \right)^{-1} = \left( \bar \Q^{-1}_b(z) + \V \bLambda \V^\T \right)^{-1} 
\]
with $\bar \Q_b (z) \equiv \left( \EE \left[\frac{g \cdot \C}{ 1 + g \cdot \delta(z) } \right] - z \I_p \right)^{-1}$ that characterizes the main bulk, $\V \equiv [\bmu,~\C^{\frac12} \U] \in \RR^{p \times 3}$ and  
\begin{equation}
  \bLambda(z) \equiv \EE \frac{g}{1+g \cdot \delta(z)} \begin{bmatrix} 1 & (\U^+ \z)^\T \\ \U^+ \z & \U^+ \z (\U^+ \z)^\T - (\U^\T \U)^+ \end{bmatrix}.
\end{equation}

Recall from Theorem~\ref{theo:lsd} that $z \in \RR$ is an eigenvalue within the main bulk if the Stieltjes transform is not defined
\[
  m(z) = \frac1p \tr \left( \EE_g \frac{g \cdot \C}{ 1 + g \cdot \delta(z) } - z \I_p \right)^{-1} = \frac1p \tr \bar \Q_b(z).
\]

To characterize the possible isolated eigenvalues (i.e., the ``spikes''), we aim to find $z \in \RR$ such that $\bar \Q(z)$ is not defined but $\bar \Q_b(z)$ is (and thus, outside the main bulk), or equivalently
\[
  \det \bar \Q^{-1} (z) = 0, \quad \text{such that $\det \bar \Q_b^{-1} (z) \neq 0$ }
\]
by solving for $z \in \RR$ such that
\begin{align*}
  0& = \det \bar \Q^{-1} (z) = \det \left( \bar \Q^{-1}_b(z) + \V \bLambda \V^\T \right) = \det \bar \Q^{-1}_b(z) \cdot \det ( \I_p + \V \bLambda \V^\T \bar \Q_b(z) ) \\
  &= \det ( \I_{2K+1} + \bLambda \V^\T \bar \Q_b(z) \V) 
\end{align*}
with Sylvester's identity. This concludes the proof of Theorem~\ref{theo:spike}. 
\qed


\subsection{Proof of Theorem~\ref{theo:eigenvector}}
\label{subsec:proof-theo-eigenvector}


Similar to the behavior of the isolated eigenvalues studied in Theorem~\ref{theo:spike}, the corresponding eigenvector can be characterized via the Cauchy's integration technique. 
More precisely, from Theorem~\ref{theo:DE}, recall that $\V \equiv [\bmu,~\C \w_*,~\C \w] \in \RR^{p \times 3}$ captures the low rank structure of the Hessian matrix $\H$.
Thus, we can characterize the projection of a given isolated eigenvector $\hat \uu$ onto the subspace spanned by the columns of $\V$ (that are assumed to be linearly independent for simplicity of presentation). To this end, we need to evaluate, for $\lambda$ the asymptotic approximation given in Theorem~\ref{theo:spike} of an isolated eigenvalue $\hat \lambda$ of $\H$ with corresponding eigenvector $\hat \uu$, the following matrix
\begin{equation}\label{eq:proof-proj-1}
	\V^\T \hat \uu \hat \uu^\T \V = -\frac1{2\pi \imath} \oint_{\Gamma_\lambda} \V^\T \Q(z) \V \,dz = -\frac1{2\pi \imath} \oint_{\Gamma_\lambda} \V^\T \bar \Q(z) \V \,dz + o(1)
\end{equation}
where $\Gamma_\lambda$ is a positively oriented contour surrounding the isolated eigenvalue only.

Since the columns of $\V$, denoted $\vv_a$ for $a \in \{1,2,3 \}$, are assumed to be linearly independent, one can write
\[
  \hat \uu = \sum_{a=1}^3 \theta_a \vv_a + \bomega_a
\]
where $\bomega_a \in \RR^p$ is a random vector of zero mean and orthogonal to $\vv_a$. As a consequence, the projection in \eqref{eq:proof-proj-1} allows to retrieve the coefficient vector $\btheta \in \RR^3$ via 
\begin{equation}
	\V^\T \hat \uu \hat \uu^\T \V = \V^\T \V \btheta \btheta^\T \V^\T \V + o(1)
\end{equation}
with a central limit theorem argument. 

We now present the proof of Theorem~\ref{theo:eigenvector} as follows.
\begin{proof}[Proof of Theorem~\ref{theo:eigenvector}]
Let $(\hat \lambda, \hat \uu)$ be an isolated eigenvalue-eigenvector pair of $\H$, and let $\lambda$ denote the asymptotic approximation given in Theorem~\ref{theo:spike}. We have
\begin{align*}
  &\V^\T \hat \uu \hat \uu^\T \V = -\frac1{2\pi \imath} \oint_{\Gamma_{\lambda}} \V^\T \Q(z) \V \,dz = -\frac1{2\pi \imath} \oint_{\Gamma_{\lambda}} \V^\T \bar \Q(z) \V \,dz + o(1) \\ 
  &= -\frac1{2\pi \imath} \oint_{\Gamma_{\lambda}} \V^\T \left( \bar \Q^{-1}_b(z) + \V \bLambda(z) \V^\T \right)^{-1} \V \,dz + o(1) \\ 
  &= -\frac1{2\pi \imath} \oint_{\Gamma_{\lambda}} \V^\T \bar \Q_b(z) \V \left( \I_3 + \bLambda(z) \V^\T \bar \Q_b(z) \V \right)^{-1} \,dz + o(1)
\end{align*}
with Woodbury identity, where we recall $\bar \Q_b (z) \equiv \left( \EE \left[ \frac{g \cdot \C}{ 1 + g \cdot \delta(z) } \right] - z \I_p \right)^{-1}$ defined in \eqref{eq:def-m(z)} that characterizes the main bulk. With Cauchy's residue theorem, it remains to evaluate the associated residue and thus the values of $z$ lying within $\Gamma_{\lambda}$ such that $\I_3 + \bLambda \V^\T \bar \Q_b(z) \V$ is singular. As a consequence,
\begin{align*}
    &\V^\T \hat \uu \hat \uu^\T \V = -{\rm Res} \left( \V^\T \bar \Q_b(z) \V \cdot \left( \I_3 + \bLambda(z) \V^\T \bar \Q_b(z) \V \right)^{-1} \right) + o(1) \\ 
    &= -\V^\T \bar \Q_b(\lambda) \V \cdot \lim_{z \to \lambda} \left( \I_3 + \bLambda(z) \V^\T \bar \Q_b(z) \V \right)^{-1} + o(1) \\ 
    &= -\V^\T \bar \Q_b(\lambda) \V \cdot \frac{\vv_{r,\G} \vv_{l,\G}^\T }{ \vv_{l,\G}^\T \G'(\lambda) \vv_{r,\G} } + o(1) \equiv - \V^\T \bar \Q_b(\lambda) \V \cdot \Xi(\lambda) + o(1)
\end{align*}
with $\G(z)$ defined in \eqref{eq:def-G(z)}, $\vv_{l,\G}, \vv_{r,\G} \in \RR^3$ the left and right eigenvectors of $\G(\lambda)$ associated with eigenvalue zero, and $\G'(z)$ the derivative of $\G(z)$ with respect to $z$ evaluated at $z = \lambda$.

This concludes the proof of Theorem~\ref{theo:eigenvector}.
\end{proof}


Note in particular that
\[
  \G'(z) = \bLambda'(z)\V^\T \bar \Q_b(z) \V + \bLambda(z) \V^\T \bar \Q_b'(z) \V
\]
with
\begin{align*}
  \bLambda'(z) &= -\EE \frac{g^2 \cdot \delta'(z)}{( 1+g \cdot \delta(z))^2} \begin{bmatrix} 1 & (\U^+ \z)^\T \\ \U^+ \z & (\U^+ \z) (\U^+ \z)^\T - (\U^\T \U)^+ \end{bmatrix}, \\ 
  \bar \Q_b'(z) &= \bar \Q_b(z) \left( \EE \frac{g^2 \cdot \delta'(z)}{( 1+g \cdot \delta(z))^2} \C + \I_p \right) \bar \Q_b(z),
\end{align*}
and $\delta'(z) = \frac{\frac1n \tr \bar \Q_b \C \bar \Q_b}{ 1 - \EE \left[ \frac{g^2}{ (1+g \cdot \delta)^2 } \right] \frac1n \tr \C \bar \Q_b \C \bar \Q_b }$. These \emph{explicit} derivatives will be used in the numerical evaluations of Theorem~\ref{theo:eigenvector} in Section~\ref{sec:discuss}.

\section{Detailed derivations for specific models}
\label{sec:SM-details-model}

\subsection{Logistic model}\label{subsec:proof-logistic}

Consider the logistic model as in \eqref{eq:def-logistic}
\begin{equation}
  \mathbb P(y = 1 \mid \w_*^\T \x) = \sigma(\w_*^\T \x) = 1 - \mathbb P(y = -1 \mid \w_*^\T \x), \quad \sigma(t) = (1 + e^{-t})^{-1}
\end{equation}
with the logistic loss $\ell(y, h) = \ln(1+e^{- y h})$ and $h = \w^\T \x \sim \mathcal N(\w^\T \bmu, \w^\T \C \w)$, we have for $g \equiv \ell''(y, h)$ defined in \eqref{eq:def-g} that
\begin{equation}
  g = \frac{\partial^2 \ell(y,h)}{\partial h^2} = \frac{e^{y h}}{ (1 + e^{y h})^2 }
\end{equation}
so that for any fixed $\delta(z)$,
\begin{equation}
  \EE \left[\frac{g}{1 + g \delta(z)} \right] = \EE \left[ \frac1{ \delta(z) + 2 + e^{-y h} + e^{y h} } \right] = \EE[ 1/(\delta(z)+ 2 + e^{-h} + e^h) ].
\end{equation}

\subsection{Phase retrieval model}\label{subsec:proof-phase}

Consider the noiseless phase retrieval model for which $y = (\w_*^\T \x)^2$ with square loss $\ell(y, h) = (y - h^2 )^2/4$ and $h = \w^\T \x \sim \mathcal N(\w^\T \bmu, \w^\T \C \w)$, we have for $g \equiv \ell''(y, h)$ defined in \eqref{eq:def-g} that
\begin{equation}
  g = \frac{\partial^2 \ell(y,h)}{\partial h^2} = 3 h^2 - y
\end{equation}
so that for any fixed $\delta(z)$,
\begin{equation}
  \EE \left[\frac{g}{1 + g \cdot \delta(z)} \right] = \EE \left[ \frac{3 h^2 - y}{1 + (3 h^2 - y) \delta(z)} \right] 
\end{equation}
where we note that $3 h^2 - y = \x^\T \Delta \x$ with $\Delta \equiv 3 \w \w^\T - \w_* \w_*^\T \in \RR^{p \times p}$ so that $3 h^2 - y$ is a weighted sum of (at most) two independent non-central chi-square random variable. (Recall that, for $\x \sim \mathcal N(\bmu, \I_p)$, $\| \x \|^2$ follows a non-central $\chi^2$ with $p$ degree of freedom and denotes $\| \x \|^2 \sim \chi^2(p, \| \bmu \|^2)$.) In particular, for $\w = \alpha \w_*$ with $\alpha \in \RR$, we obtain $\Delta = (3 \alpha^2 -1) \w_* \w_*^\T$ and $3 h^2 - y$ is thus a (non-centered) chi-square random variable.

\subsection{Derivation of Corollary~\ref{coro:special-logistic-w=0}}\label{subsec:proof-special-logistic-w=0}

For the logistic model defined in \eqref{eq:def-logistic} with $\w = \w_* = \mathbf{0}$ and $\C = \I_p$, it follows from Appendix~\ref{subsec:proof-logistic} that $g = 1/4$, so that, by Theorem~\ref{theo:lsd}, the Stieltjes transform $m(z)$ satisfies
\begin{equation}
	z c m^2(z) - (1 - c - 4 z) m(z) + 4 = 0
\end{equation}
which corresponds to the Mar\u{c}enko-Pastur law in \eqref{eq:MP-equation} rescaled by $1/4$. It then follows from Theorem~\ref{theo:spike} that
\begin{equation}
	\det \G(\lambda) = 1 + \frac{ \| \bmu \|^2 m(\lambda)}{4 + c m(\lambda)} = 0 \Leftrightarrow m(\lambda) = - \frac4{c + \| \bmu \|^2}
\end{equation}
which, by plugging back into the equation above gives the expression of $\lambda_{\bmu}$, and the condition $\| \bmu \|^2 > \sqrt c$ follows from the property $z m(z) < 0$ for $\Im[z] = 0$ of the Stieltjes transform. With Theorem~\ref{theo:eigenvector}, we conclude the proof of Corollary~\ref{coro:special-logistic-w=0}.

\subsection{Derivation of Corollary~\ref{coro:special-logistic-mu=0}}\label{subsec:proof-special-logistic-mu=0}

For the logistic model defined in \eqref{eq:def-logistic} with $\bmu = \mathbf{0}$ and $\C = \I_p$, it follows from Appendix~\ref{subsec:proof-logistic} and Theorem~\ref{theo:lsd} that the Stieltjes transform $m(z)$ satisfies
\begin{equation}\label{eq:m(z)-logistic-mu=0}
  m(z ) = \frac{\delta(z)}{c} = \frac1{\EE[ (2 + \exp(-r) + \exp(r) + c m(z))^{-1} ] - z}
\end{equation}
for $r \sim \mathcal N(0, \| \w \|^2)$ that \emph{depends} on $\w$ but not on $\w_*$. Moreover, Theorem~\ref{theo:spike} gives $\bar \Q_b(z) = m(z) \I_p$ and
\begin{align*}
	0 &= \det \G(z) = \det \left(\I_2 +  m(z) \EE [f(r,z) (r^2 - \| \w\|^2)] (\U^\T \U)^+ \U^\T \frac{\w \w^\T}{\| \w \|^4} \U \right) \\ 
	&= 1 + m(z) \frac{ \EE [f(r,z) (r^2 - \| \w\|^2)] }{\| \w \|^2} 
\end{align*}
with $\U = [\w_*,~\w]$ and $f(t,z) = \frac1{ c m(z) + 2 + \exp(-t) + \exp(t) }$, where we use the fact that
\begin{align*}
	\EE \left[ \frac{g \cdot \z \z^\T}{1 + g \cdot \delta(z)} \right] &=  \EE [f(r,z)] \cdot \I_p + \EE [ f(r,z) (r^2 - \| \w\|^2) ] \cdot \frac{\w \w^\T}{\| \w \|^4}  ,
\end{align*}
and $\EE[ \frac{g}{1 + g \cdot \delta(z)} ] = \EE[f(r,z)]$ from Appendix~\ref{subsec:proof-logistic}. As such, there exists \emph{no} spike due to $\w_*$.

Moreover, taking $\w_* = \zo$, we further obtain
\begin{equation}
	\G(z) = 1 + m(z) \EE [ f(r,z) (r^2/\| \w\|^2 - 1) ]
\end{equation}
and
\begin{equation}
	 \G'(z) = m'(z) \left( \EE [ f(r,z) (r^2/\| \w\|^2 - 1) ]  - c m(z) \EE [ f^2(r,z) (r^2/\| \w\|^2 - 1) ] \right)
\end{equation}
where we have $m'(z) = \frac{m^2(z)}{1 - c m^2(z) \EE[f^2(r,z)]}$ by differentiating \eqref{eq:m(z)-logistic-mu=0}. Therefore, we obtain from Theorem~\ref{theo:eigenvector} that
\begin{equation}
	\frac{|\w^\T \hat \uu|^2}{\| \w\|^2} = -\frac{m(\lambda_\w)}{\G'(\lambda_\w)} + o(1).
\end{equation}

\end{document}